\documentclass{article} % For LaTeX2e
\usepackage{iclr2027_conference,times}

\usepackage{hyperref}
\usepackage{url}

\usepackage{graphicx}
\usepackage{float}
\usepackage{booktabs}

\usepackage{amsthm}

\usepackage{blindtext}
\usepackage{titletoc}
\usepackage{listings}
\usepackage{amsmath}
\usepackage{amsfonts}
\usepackage{amssymb}
\usepackage{thmtools,thm-restate}
\usepackage{csquotes}
\usepackage{subcaption}
\usepackage{booktabs}
\usepackage{esint}
\usepackage{algorithm}
\usepackage{algpseudocode}
\usepackage{tikz}
\usepackage{xcolor}
\usepackage{amsmath, amsthm}  % Provides \lesssim and \gtrsim
\usepackage{amssymb}  % Provides additional symbols like \lessapprox
\definecolor{DarkGreen}{rgb}{0, 0.6, 0}

\usepackage{todonotes}
\usepackage{hyperref}

\usepackage{amsmath,amssymb,amsthm,mathtools} \usepackage{enumitem} \usepackage{hyperref}

\newcommand{\cE}{\mathcal{E}}
\newcommand{\E}{\mathbb{E}}
\newcommand{\cN}{\mathcal{N}}
\renewcommand{\P}{\mathbb{P}}
\newcommand{\cX}{\mathcal{X}}
\newcommand{\Gq}{\Gamma_{B}}

\newcommand{\R}{\mathbb{R}}
\newcommand{\sq}{\mathcal{Q}_{B}}
\newcommand{\roundB}[1]{\left\lfloor #1 \right\rceil_B}

\newcommand{\MB}{M_{B}}

\newcommand{\reff}{r_{\text{eff}}}
\renewcommand{\theta}{\vartheta}
\renewcommand{\P}{\mathbb{P}}
\newcommand{\cU}{\mathcal{U}}
\newcommand{\cJ}{\mathcal{E}^{\infty}}
\newcommand{\Id}{\operatorname{Id}}
\newcommand{\cbJ}{\bar{\mathcal{E}}^{\infty}}
\newcommand{\cV}{\mathcal{V}}
\newcommand{\cF}{\mathcal{F}}

\newcommand{\op}{\mathrm{op}}
\newcommand{\opn}[1]{\left\|#1\right\|_{\op}}

\newcommand{\Sph}[1]{\mathbb{S}^{#1}}

\DeclareMathOperator{\Tr}{tr}

\definecolor{grey}{rgb}{.7,.7,.7}
\definecolor{evidJ}{rgb}{0,0,1}
\definecolor{evidM}{rgb}{0,0.5,0}

\renewcommand{\P}{\mathbb{P}}
\newcounter{experiment}

\newtheorem{defn}{Definition}
\newtheorem{lemma}{Lemma}

\newtheorem{theorem}{Theorem}
\newtheorem{corollary}{Corollary}
\theoremstyle{remark}
\newtheorem{remark}{Remark}
\theoremstyle{plain}
\usepackage{xcolor}
\usepackage{listings}
\title{ Scale Sensitivity in Low-Bit Post-Training Quantization: Curvature of the Quantization Error Landscape
}

\author{Jonas von Berg\thanks{Equal contribution.} \& Massimiliano Datres\footnotemark[1] \& Carlo Kneißl \&  Gitta Kutyniok\thanks{University of Tromso, DLR-German Aerospace Center} \\
Ludwig-Maximilians-Universität München \\
Munich Center for Machine Learning (MCML) \\ 
Konrad Zuse School of Excellence in Reliable AI (DAAD)\\ 
\texttt{\{berg,datres,kneissl,kutyniok\}@math.lmu.de}\\
}

\iclrfinalcopy 
\begin{document}

\maketitle
\fancyhead{}         
\renewcommand{\headrulewidth}{0pt}

\begin{abstract}
Post-training quantization (PTQ) methods in the GPTQ family minimize a layer-wise reconstruction error on a uniform grid whose scale must be chosen; the common max-based choice degrades sharply at low bit-widths. We study how sensitive this objective is to the scale. For a layer with i.i.d.\ Gaussian weights and calibration activations of sufficiently large effective rank, we prove that, as the width grows, the normalized round-to-nearest loss converges with high probability, uniformly over all scales, to the mean-squared error of a uniform quantizer applied to a standard Gaussian; we verify the effective-rank condition for wide, randomly initialized MLPs with odd Lipschitz activations and isotropic Gaussian calibration data. The limiting objective has a unique nondegenerate minimizer, whose scale decreases strictly with the number of levels and whose curvature with respect to relative scale errors decays approximately exponentially with the bit-width. GPTQ experiments on five LLMs show the same trend: the scale rule changes perplexity substantially at 2--3 bits and negligibly from 6 bits on, and a local measure of GPTQ scale sensitivity decreases with bit-width in line with the Gaussian curvature. The Gaussian-optimal scale fails on raw weights; after Hadamard incoherence processing it matches the best searched rule at 3 bits and above without any search, but remains clearly worse at 2 bits.

\end{abstract}

\section{Introduction}
Quantization has emerged as a key approach to reduce the computational cost and memory footprint of modern machine learning models. As these models grow larger, methods that compress them and accelerate inference without retraining have attracted considerable interest in the machine learning community. The success of small-scale language models (LLMs) that can be deployed locally, such as Llama 3.2 \citep{DBLP:journals/corr/abs-2407-21783}, DeepSeek \citep{liu2024deepseek}, and BitNet \citep{ma2025bitnet}, has further intensified this interest.

Among various compression strategies, Post-training quantization (PTQ) is especially attractive because it can be applied to a trained model using only a small calibration set, without the need of retraining. It therefore offers a practical route to efficient deployment with modest additional overhead. Despite its widespread adoption, the theoretical understanding of PTQ remains limited. A wide variety of PTQ methods have been proposed, such as GPTQ \cite{frantar2022gptq} and AWQ \cite{lin2024awq}, but a mathematical analysis of these methods has only recently begun to emerge \cite{tseng2024quip, chen2026geometry, zhang2026QRoNoS}. All of these, however, consider the quantization problem on a fixed grid. A key parameter shared by these algorithms is the quantization scale, which determines the spacing and range of the quantization grid. Although many rules have been proposed for selecting it \cite{banner2019post, nagel2021white}, see also Table \ref{tab:scale-selection}, it remains unclear which choices best preserve model performance and how sensitive quantization algorithms are to the quantization scale choice. Independently of how the scale is chosen, these methods are surprisingly accurate at 4 or 8 bits. At lower precision, however, performance degradation seems unavoidable \cite{dettmers2023case}.

Preprocessing techniques can further reduce the degradation caused by quantization. A prominent example is Hadamard incoherence processing (HIP) \citep{tseng2024quip}, which applies a random orthogonal transformation to the weights so as to spread their magnitudes evenly across coordinates. As a consequence, the entries of the transformed weights become approximately Gaussian.

In this work, we take a step towards a mathematical understanding of how scale choice affects these algorithms. We view PTQ as the problem of approximating the linear map of a network layer, in the layerwise reconstruction sense, at a fixed quantization granularity. Our main observation is that, for networks at initialization and under suitable conditions on the calibration data, the reconstruction loss under round-to-nearest (RTN) quantization, viewed as a function of the scale, is well approximated by a deterministic Gaussian quantization objective.
We then analyze the Gaussian quantization problem, characterizing its loss landscape around the optimal scale and its dependence on the bit width. We use this analysis to shed light on the scale sensitivity of existing PTQ methods, especially when combined with HIP preprocessing.

\paragraph{Paper Roadmap.}
Section~\ref{gen_inst} reviews related work, and
Section~\ref{sec:notandpre} introduces the notation and quantization
setup. Section~\ref{others} establishes uniform convergence to a Gaussian limiting loss, characterizes its landscape, and verifies the convergence assumptions for randomly initialized MLPs with
Gaussian data. Section~\ref{sec:experiments} evaluates scale
sensitivity in LLM quantization. Proofs and additional experiments appear in the appendix.
\paragraph{Contributions.} We summarize our contributions below. 
\begin{enumerate}
    \item \textbf{Uniform Gaussian limit.} Under Gaussian weight and effective-rank assumptions, we prove uniform convergence over scales of the normalized RTN
reconstruction loss to a scalar Gaussian objective, with explicit
finite-dimensional, high-probability bounds
(Theorem~\ref{thm:uniform}). We verify the rank condition for wide
MLPs at initialization (Theorem~\ref{thm:mlpworks}). 
\item \textbf{Optimal scales and sensitivity.} For an analytic extension of the limiting objective, we prove a unique, nondegenerate minimizer and a strictly decreasing optimal
scale as quantizer resolution increases (Proposition~\ref{prop:scalarlandscape}), extending the discrete
monotonicity of \citet{na2018monotonicity}. We derive its scale-invariant curvature and numerically observe approximately exponential decay with bit-width. 
\item \textbf{Implications for LLM quantization.} Across several LLMs, GPTQ exhibits greater sensitivity to scale selection at 2--3 bits than at 6--8 bits, with local sensitivity trends consistent with the Gaussian model. HIP substantially
improves the performance of the resulting analytic scale rule,
which requires no scale search. We also introduce Proxy-static, an efficient GPTQ-aware
scale-selection objective that performs strongly in
our experiments, particularly without preprocessing.

\end{enumerate}

\section{Related Work}
\label{gen_inst}
\paragraph{Layerwise post-training quantization.}
For scalability, PTQ often minimizes the layerwise reconstruction error $\|(W-\hat W)X\|_F^2$ on a fixed quantization grid. GPTQ \cite{frantar2022gptq}, building on Optimal Brain Quantization \cite{frantar2022optimal} and OBS \cite{hassibi1992second}, quantizes columns sequentially while compensating errors through the remaining floating-point weights. GPTAQ \cite{li2025gptaq}, ResComp \cite{li2026rethinking}, and QRoNoS \cite{zhang2026QRoNoS} modify reconstruction and error compensation. Orthogonal preprocessing \citep{chee2023quip,tseng2024quip,ashkboos2024quarot,liu2025spinquant}, activation smoothing \citep{xiao2023smoothquant}, activation-aware weight scaling \cite{lin2024awq}, and magnitude reduction \citep{zhang2024magr} make weights and activations easier to quantize. In particular, QuIP \cite{tseng2024quip} uses random orthogonal preprocessing to obtain incoherence-based error guarantees for adaptive rounding on a fixed grid. Scale selection remains often heuristic; we study its influence, especially at low precision, without changing the projection procedure. BRECQ \cite{li2021brecq} combines blockwise reconstruction with learned rounding based on AdaRound \cite{nagel2020up}, improving low-bit accuracy at an optimization cost that can be prohibitively expensive for large models.

\paragraph{Scale and clipping selection.}
OmniQuant \cite{shao2024omniquant} optimizes clipping parameters and equivalent transformations through blockwise reconstruction, avoiding per-weight rounding optimization. PiSO \citep{amboage2026optimal} exploits the piecewise-quadratic objective to compute exact channelwise RTN scales, but reports collapse at 3 bits in the data-free setting. ACIQ \citep{banner2018aciq} derives Gaussian and Laplacian clipping thresholds, applying them only to activations after finding no benefit for weights. LLAPQ \cite{nahshan2021loss} studies the full-model task-loss landscape over quantization parameters. We instead study layerwise reconstruction, establish a scalar asymptotic limit under coordinatewise RTN, and empirically examine scale sensitivity for RTN and GPTQ.

\paragraph{Optimal uniform scalar quantization.}
Classical scalar quantization theory studies optimal quantizers for specific distributions. Max \citep{max1960quantizing} derived optimality conditions and numerical solutions, including for Gaussian sources. For symmetric uniform quantization, Na and Neuhoff \citep{na2017convexity,na2018monotonicity} study distortion convexity in the grid scale and prove that the Gaussian-optimal scale decreases when two levels are added. We connect this theory to layerwise reconstruction through a rigorous asymptotic limit, derive explicit curvature formulas, and compare them with empirical estimates from real quantization landscapes.

\section{Notation and Preliminaries}\label{sec:notandpre}
\paragraph{Notation.} 

We use standard notation for sets, Gaussian distributions, and differentiability classes, with $\mathbb N={1,2,\dots}$ and $[n]={1,\dots,n}$. Write $\operatorname{Id}_d$ for the identity matrix and $\mathbb S^{p-1}$ for the Euclidean unit sphere in $\mathbb R^p$. Vectors, including matrix rows $M_{i:}$, are viewed as columns. The norms $\|\cdot\|_{op}$, $\|\cdot\|_F$, and $\|\cdot\|_{\psi_2}$ denote the operator, Frobenius, and sub-Gaussian norms, respectively. For a nonzero positive semidefinite matrix $M$, the effective rank is given by $$r_{\mathrm{eff}}(M):=\operatorname{Tr}(M)/\|M\|_{op}\, .$$ Symmetric rounding is $\lfloor t\rceil:=\operatorname{sign}(t)\lfloor |t|+1/2\rfloor$, with $\operatorname{sign}(0)=0$. We denote the standard Gaussian density and cumulative density function by $\varphi$ and $\Phi$. Subscripts on $\mathbb E,\mathbb P$ specify the variable integrated over; primes and subscripts denote ordinary and partial derivatives respectively, e.g., $f_{ssm}=\partial_m\partial_s^2f$. Constants $\tilde C,C,c>0$ may change from line to line.

\paragraph{Post-training quantization.}
Post-training quantization (PTQ) maps the weights of a trained
neural network onto a low-bit grid without retraining, aiming
to preserve its input--output behavior using a small calibration
set. Let $\vartheta\in\R^d$ be a vector of full-precision weights,
let $B\geq2$ denote the target bit precision, and set
$M_B:=2^{B-1}-1$. We define the
\textit{symmetric quantization cone} $\Gq\subset\R^d$ as the
union of all symmetric uniform grids with $2^B-1$ levels
per coordinate:
\[
\Gq := \left\{s v : s>0,\;
v\in\{-M_B,\dots,M_B\}^d\right\} \, .
\]
For a fixed scale $s>0$, coordinatewise round-to-nearest
(RTN) quantization is the Euclidean projection onto the
corresponding grid:
\begin{equation}\label{def:roundB}
\Pi_{s,M_B}(\vartheta)
:=s\roundB{\frac{\vartheta}{s}}\, , \qquad \roundB{t}
:=\begin{cases}
    -M_B & t<-M_B,\\
    \lfloor t\rceil & |t|\leq M_B,\\
    M_B & t>M_B.
\end{cases}.
\end{equation}
where $\roundB{\cdot}$ is applied componentwise. We write $\sq:=\{\Pi_{s,M_B}:s>0\}$ for the resulting family of quantization maps.

We study layerwise PTQ through the empirical squared $L_2$
distance between a full-precision linear map and its quantized
approximation, evaluated on the same input activations.
Let $\cX\subseteq\R^{d_0}$ denote the input space, equipped
with a probability measure $\mu$, and let
$\mathbb D_{\mathrm{cal}}=\{X'_j\}_{j=1}^{N_c}$ consist of
$N_c$ i.i.d.\ samples from $\mu$.
For a layer with weights
$\vartheta^{(\ell)}\in\R^{d_\ell\times d_{\ell-1}}$,
we consider the RTN scale-selection problem of minimizing,
over $s>0$, the reconstruction error
\begin{equation}\label{eq:PTQprob}
\cE^{(\ell)}(s,\vartheta^{(\ell)},X^{\ell-1})
:=
\left\|
\bigl(\vartheta^{(\ell)}
-\Pi_{s,M_B}(\vartheta^{(\ell)})\bigr)X^{\ell-1}
\right\|_F^2,
\end{equation}
where the quantization map acts entrywise.
Here, $X^{\ell-1}\in\R^{d_{\ell-1}\times N_c}$ denotes
the input activation matrix used to calibrate layer $\ell$.
Its columns are obtained by propagating the calibration
samples through the preceding layers, which may be
full-precision or already quantized, depending on the PTQ
procedure. We hold $X^{\ell-1}$ fixed when optimizing
the scale of the current layer.
We write $\cE^{(\ell)}(s)$ when the dependence on
$\vartheta^{(\ell)}$ and $X^{\ell-1}$ is clear.

This formulation also covers query, key, value, and output projections in attention, using their respective input activations. While \eqref{eq:PTQprob} uses one scale per layer, other granularities are commonly used. For instance, \textit{per-channel} quantization assigns one scale per channel of $\vartheta^{(\ell)}$, yielding $d_\ell$ independent row-wise problems. Finer granularities, such as \textit{group quantization}, are also used.

A simple scale-selection rule is the symmetric min-max
choice $s^e:=\tfrac{\max_{i,j}|\vartheta^{(\ell)}_{ij}|}{M_B}$,
with the maximum taken over the relevant channel or group
for finer granularities. This choice uses only the weight
range and does not optimize \eqref{eq:PTQprob}.
Our experiments show that it can perform well at higher
precision, while becoming substantially less reliable
at lower bit-widths.

\section{Theoretical Results}
\label{others}
In this section, we study how sensitive the reconstruction
error \eqref{eq:PTQprob} is to the quantization scale $s$
across different bit-widths $B$.
We focus on one output channel and write
$\vartheta\in\R^d$ for its full-precision weight vector,
where $d=d_{\ell-1}$ is the input dimension.
Let $X\in\R^{d\times N_c}$ denote the corresponding
calibration activations.
For $\Tr(XX^\top)>0$, we consider the normalized reconstruction loss
\begin{align*}
\cE(s,\vartheta,X)=\cE(s)
&:=\frac{\|\vartheta^\top X
-\Pi_{s,M_B}(\vartheta)^\top X\|_2^2}{\Tr(XX^\top)}=\frac{
\bigl(\vartheta-\Pi_{s,M_B}(\vartheta)\bigr)^\top
XX^\top
\bigl(\vartheta-\Pi_{s,M_B}(\vartheta)\bigr)
}{\Tr(XX^\top)}.
\end{align*}
The normalization by $\Tr(XX^\top)$ removes the overall
scale of the calibration activations and, under our assumptions,
yields a loss of order one as $d$ grows.
For instance, when $XX^\top=\Id_d$, the normalized loss
reduces to
$\cE(s)=\tfrac{1}{d}\|\vartheta-\Pi_{s,M_B}(\vartheta)\|_2^2$,
the mean squared approximation error per coordinate.
For fixed $s>0$ and bit-width $B$, this admits a nontrivial
limit as $d\to\infty$ when
$\vartheta\sim\cN(0,\Id_d)$.
In the rest of this section, we assume that 
\begin{itemize}
    \item[(H1)] the parameters $\theta \overset{iid}{\sim}\cN(0,1)$;
    \item[(H2)] the effective rank of $XX^\top$ satisfies 
    $\reff(XX^\top) \ge C\log^8 d$ for some constant $C>0$ not dependent on $d$.
\end{itemize}
Assumptions (H1) is standard in the literature on wide random neural networks. (H2) asks that the inputs to the group of parameters being quantized are not concentrated on a small number of directions. Section~\ref{sec:jd} establishes (H2) with high probability for an MLP at initialization under the stated
assumptions on its architecture and isotropic input distribution. 

We study the asymptotic limit of $s\mapsto\cE(s)$ to gain
insight into its loss landscape. In
Section~\ref{sec:unifconv}, we establish a high-probability
bound yielding uniform convergence in $s$ to the deterministic
limiting objective
\[
\cE^\infty(s)
:=\E_{Z\sim\cN(0,1)}
\bigl[(Z-\Pi_{s,M_B}(Z))^2\bigr].
\]
We then analyze the landscape of $s\mapsto\cE^\infty(s)$
in Section~\ref{sec:limfuncland}.

\subsection{Convergence to the limit}\label{sec:unifconv}
Since our goal is to study the loss as a function of $s$, we seek an approximation that holds simultaneously over all $s>0$. Such a uniform guarantee also applies when $s$ is chosen based on the realization of $\vartheta$ and the calibration data $X$. The following theorem provides this guarantee with an explicit error bound.
\begin{restatable}{theorem}{uniformtheorem}\label{thm:uniform}
 Let $B\ge 2$, $X\in\R^{d\times N_c}$ be nonzero matrix satisfying (H2), and $\vartheta$ satisfying (H1). Then, for all $d\ge2$, it holds
\begin{align*}
    \P\left( \sup_{s>0} \ \left|\cE(s) - \cE^{\infty}(s)\right| > \frac{2+12M_B}{\log d}\right)
    \le  C\,M_B\,d\,e^{-c \log^2 d}\, ,
\end{align*}
where $c,C>0$ are absolute constants independent on $d$.  
\end{restatable}

A detailed proof of Theorem~\ref{thm:uniform} appears in
Appendix~\ref{app:proofthmuniform}. For fixed $B$, the theorem
bounds the approximation error by a constant multiple of
$(\log d)^{-1}$ with probability tending to one as
$d\to\infty$. In particular,
\[
\sup_{s>0}\left|\cE(s)-\cE^\infty(s)\right|
\xrightarrow[d\to\infty]{\P}0.
\]
Let
$m:=\inf_{s>0}\cE(s)$ and
$m^\infty:=\inf_{s>0}\cE^\infty(s)$
denote the optimal empirical and limiting reconstruction
losses. We note that both $m$ and $m^\infty$ are finite, since $\cE,\cE^\infty\ge0$. Uniform convergence immediately implies
$m\xrightarrow{\P}m^\infty$. Convergence of minimizers
is established in Corollary~\ref{cor:minimizers} in
Appendix~\ref{app:minimizers}.

\begin{remark}
The uniform approximation of Theorem~\ref{thm:uniform} relies on (H2), namely a lower bound on the effective rank of $XX^\top$. In Appendix \ref{sec:jd} we show that this condition holds, with high probability, in the setting of a random multi-layer perceptron with odd activation functions and isotropic inputs.
\end{remark}

%-------------------------------------------------

\subsection{Limiting function landscape}\label{sec:limfuncland}
The uniform convergence established in Theorem~\ref{thm:uniform} motivates the study of the limiting functional
\begin{align}\label{eq:assylandsimple}
\cE^{\infty}(M_B,s):=\int_{\R}\big(s\,\roundB{t/s}-t\big)^2\varphi(t)\,dt \, .
\end{align}
Since our goal is to understand the loss landscape of $\cE^{\infty}$ as a function of the bit precision $B$, which enters only through $M_B$, we make this dependence explicit in the notation, that is from now on we write $\cE^{\infty}(M_B,s)$ in place of $\cE^{\infty}(s)$ used above. We start by noting that $\cJ(M_B,s)$ splits into the rounding error of an infinite grid, which depends on $s$ only, and a nonnegative clipping term, which is the only part through which the bit precision $B$ enters. More precisely, $\cE^{\infty}(M_B,s) = 1 - 4s\sum_{j=0}^{\infty}g\bigl((j+\tfrac12)s\bigr) + 4s\sum_{j=0}^{\infty}g\bigl((M_B+j+\tfrac12)s\bigr)$, where $g(x):=\varphi(x)-x\,(1-\Phi(x))$. The derivation of this expression, together with the proof of its well-posedness, is deferred to Corollary \ref{cor:series} in Appendix \ref{app:lim_func}.
For the analysis that follows it is convenient to let the second argument vary continuously. We therefore consider the extension $\cbJ:[1,\infty)\times (0,\infty)\to(0,\infty)$ of $\cJ$, defined by $\cbJ(m,s) :=1 - 4s\sum_{j=0}^{\infty}g\bigl((j+\tfrac12)s\bigr)+4s\sum_{j=0}^{\infty}g\bigl((m+j+\tfrac12)s\bigr)$,  so that $\cbJ(M_B,s)=\cJ(M_B,s)$ for every $B\ge2$.

The following proposition shows that $\cbJ$ is smooth, that for every $m$ the function $s\mapsto\cbJ(m,s)$ has a unique critical point, which is a nondegenerate global minimizer $s^*(m)$, and that the optimal scale $s^*(m)$ is a smooth, strictly decreasing function of $m$, bounded below by $\sqrt2/(m+\tfrac12)$.
\begin{restatable}{prop}{inftyprop}\label{prop:scalarlandscape}
    The following properties hold: 

\begin{enumerate}[label=(\roman*)]
        \item For every $B\ge 2$, the function $\cbJ(m,s)$ is smooth, e.g. $\cbJ(m,s)\in\mathcal{C}^\infty([1,\infty)\times(0,\infty))$.
        \item For every real \(m\ge1\), the function \(s\mapsto\cbJ(m,s)\) has a unique critical point on \((0,\infty)\). This point is its unique global minimizer \(s^*(m)\) and is nondegenerate:
\[
\partial_s^2\cbJ\bigl(m,s^*(m)\bigr)>0.
\]
\item The map $m\mapsto s^*(m)$ is continuously differentiable and strictly decreasing on $[1,\infty)$. Moreover, for every $m\ge1$,
\[
\frac{d s^*}{dm}(m)=-\frac{\cbJ_{ms}\bigl(m,s^*(m)\bigr)}{\cbJ_{ss}\bigl(m,s^*(m)\bigr)}<0 , \qquad  s^*(m)>\frac{\sqrt2}{m+\frac12}.
\]
\end{enumerate}

\end{restatable}
Proposition~\ref{prop:scalarlandscape} describes the landscape
of the limiting energy $\cbJ(m,\cdot)$ as a function of the
scale $s$ for a fixed level parameter $m$.
Item~(ii) shows that it has a unique critical point, which
is a nondegenerate global minimum.
Item~(iii) shows that increasing the number of levels
strictly decreases the optimal scale $s^*(m)$.
However, at fixed $m$, decreasing the scale trades finer
resolution for a narrower representable range and hence
greater clipping error. The optimal scale balances these
effects and satisfies
$s^*(m)>\sqrt{2}/(m+\tfrac12)$.

Beyond existence and uniqueness of the optimal scale
$s^*(m)$, we are interested in how sensitive the energy
is to scale choice. Since the optimal scale changes with
the number of levels, we consider relative rather than
absolute scale perturbations. The ordinary second derivative
measures sensitivity to additive changes in $s$; multiplying
it by $s^2$ instead measures sensitivity to proportional
changes.

\begin{defn}
Let $m\in[1,\infty)$ be such that
$\bar{\cE}^{\infty}(m,\cdot)$ is twice differentiable at $s^*(m)$.
We define its \emph{scale-invariant curvature} by $\kappa(m,s^*(m))
:=s^2\bar{\cE}^{\infty}_{ss}(m,s^*(m))$.
\end{defn}

At the minimizer $s^*=s^*(m)$, a relative perturbation
$s=s^*(1+\varepsilon)$ gives $\bar{\cE}^{\infty}(m,s^*(1+\varepsilon))
-\bar{\cE}^{\infty}(m,s^*)
=
\frac12\kappa(m,s^*)\varepsilon^2
+o(\varepsilon^2)$.
We invite the reader to Remark \ref{rem:scale_sense} for a more detailed discussion on the role of $s^(*)(m)$ in the definition of $\kappa$. Thus, $\kappa(m,s^*)$ quantifies the local increase in reconstruction loss caused by a relative error in the scale. An explicit computation of $\kappa(m, s^*(m))$, for an integer $m>0$, leads to
\begin{align*}
    \kappa(m,s)=4s\sum_{j=0}^{m-1}\Big[2\big(j+\tfrac12\big)s\Big(1-\Phi\big((j+\tfrac12)s\big)\Big)-\big(j+\tfrac12\big)^2s^2\,\varphi\big((j+\tfrac12)s\big)\Big]\, .
\end{align*}
The derivation of this formula can be found in Appendix \ref{sec:compofkappa}. A numerical evaluation for $B=2,\dots,8$ shows that $\kappa(M_B,s^*(M_B))$ is exponentially decreasing with $B$ as shown in Figure \ref{fig:kappaisdecr}. We refer to Algorithm \ref{alg:kappa} in Appendix \ref{sec:compofkappa} for the algorithm we have used for the numerical evaluation.

This shows that the loss landscape is sharper near its minimum at low bit-widths and flatter at high bit-widths, when measured with respect to relative changes in the quantization scale.
%----------------------------------------------------------

\section{Experiments}\label{sec:experiments}
The previous section showed that, under suitable assumptions,
the layerwise PTQ reconstruction loss converges to the simpler
asymptotic loss \eqref{eq:assylandsimple}. We now test whether
predictions from this limiting landscape carry over to trained
models. In particular, we want to answer the following questions: \textbf{1)} Does the greater local scale sensitivity of the
limiting objective at lower precision translate into greater sensitivity to scale choice under GPTQ? \textbf{2)}  Does the analytic scale $s=s^*(M_B)\hat{\sigma}$ remain near-optimal for the reconstruction loss of trained, finite-width layers? Here, $\hat{\sigma}$ is the RMS of the weights in the
channel being quantized, and $s^*(M_B)$ is computed
by solving $\cE_s^\infty(M_B,s^*(M_B))=0$,
using the explicit expression in \eqref{eq:determiningsm}.

\subsection{Scale sensitivity of LLM post-training quantization}
We quantize several LLMs using GPTQ, with a calibration set
of 128 randomly sampled windows of 2,048 tokens from the C4
training split. We use per-channel quantization and vary
only the scale-selection rule. Our theory concerns coordinatewise RTN, whereas GPTQ updates
the remaining weights to compensate for quantization errors.
The two procedures coincide on the same fixed grid when
$XX^\top$ is diagonal, but GPTQ is not generally covered
by our theory. These experiments therefore test whether
the predicted scale sensitivity extends to GPTQ. We consider five different scale-selection rules, ranging from no search (Min-max) through search without activation data (WMSE, Shrink-2.4) to Hessian-weighted search (Proxy-static, RTNH), as specified in
Table~\ref{tab:scale-selection}.
\begin{remark}
Existing scale rules target the RTN loss, which need not align with GPTQ's compensated loss without preprocessing. Our \emph{Proxy-static} rule, see Appendix~\ref{app:proxy-static}, reuses GPTQ's inverse-Hessian factorization to account for the compensation: it is the best rule without HIP (Table~\ref{tab:scale-choice-performance-raw}) but gives no clear gain with HIP (Table~\ref{tab:scale-choice-performance_HIP}).
\end{remark}

For each model and precision,
Figure~\ref{fig:spread} reports the performance
range: the difference between the maximum and minimum NLL
across the five scale-selection rules.
This measures how much the quantized model performance depends on the choice
of scale-selection rule. 
\begin{figure}[h!]
    \centering
    \includegraphics[width=0.7\textwidth]{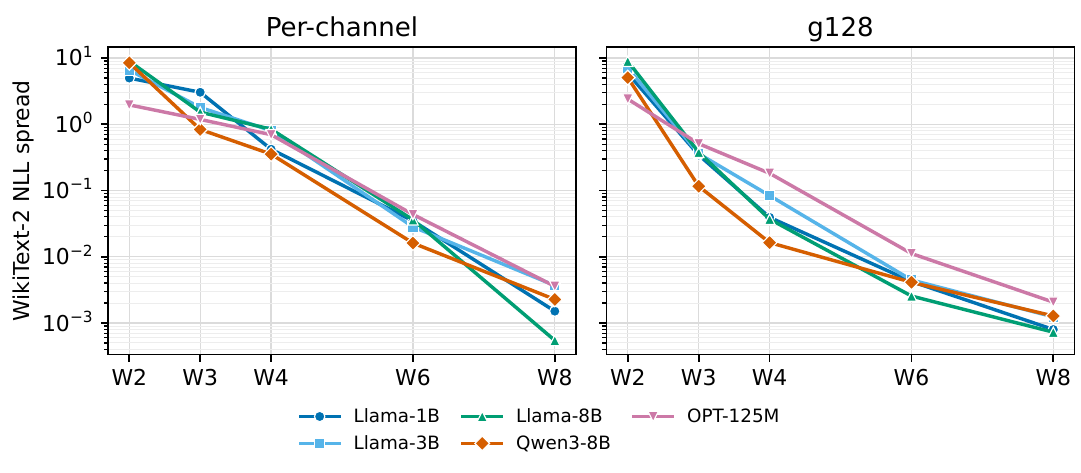}
   \caption{Log-scale GPTQ sensitivity to the scale rule across bit per channel (left) and group size 128 (right) for different models. NLL spread is the max minus min WikiText-2 NLL across Min-max, Shrink-2.4, WMSE, Proxy-static, and RTNH. }
      \label{fig:spread}
\end{figure}
We observe, that from W6 onward, all scale rules give essentially the same NLL. 
The scale matters most at W2 and W3, consistent with $\kappa$ decreasing exponentially in the number of bits, i.e., the theoretical loss landscape being more sensitive to relative scale perturbations at low precision.  \cite{amboage2026optimal} reports a related trend, i.e. data-aware scale optimization pays off increasingly as the bit-width decreases. We note that, for group size 128, the performance differences between scale-selection rules diminish even more rapidly with bit-width, consistent with the greater flexibility of fitting scales to smaller groups of weights.

\begin{remark}
\label{rem:ruleisnotcurvature}
The NLL spread across scale-selection rules reflects both the model's sensitivity to scale changes and the differences between the scales selected by the rules. It shows the practical importance of scale selection in low-precision settings, but does not establish higher curvature of the local GPTQ reconstruction loss. Appendix~\ref{app:controlledsens} complements this comparison by measuring sensitivity under fixed relative scale perturbations.
\end{remark}

\begin{table}[h!]
      \centering
      \caption{Llama-3.1-8B, W3 per-channel GPTQ without HIP. Lower perplexity and higher accuracy indicate better performance. Proxy-static, the GPTQ-aware scale-selection rule introduced in Appendix \ref{app:scale-rules}, consistently outperforms the other tested rules in this setting. }
      \label{tab:scale-choice-performance-raw}
      \setlength{\tabcolsep}{3pt}
      \resizebox{\linewidth}{!}{%
      \begin{tabular}{@{}lrrccccccc@{}}
          \toprule
          Rule
          & \shortstack{WikiText-2\\PPL}
          & \shortstack{C4\\PPL}
          & PIQA & ARC-E & ARC-C & HellaSwag
          & WinoGrande & BoolQ & Mean \\
          \midrule
          Shrink-2.4
          & 14.050 & 17.177
          & 0.705 & 0.553 & 0.369 & 0.708
          & 0.666 & 0.795 & 0.632 \\

          Proxy-static
          & \textbf{11.756} & \textbf{16.328}
          & \textbf{0.729} & \textbf{0.635} & \textbf{0.412}
          & \textbf{0.721} & \textbf{0.685} & \textbf{0.798}
          & \textbf{0.663} \\

          Min-max
          & 41.473 & 45.506
          & 0.624 & 0.408 & 0.282 & 0.472
          & 0.515 & 0.524 & 0.471 \\

          Analytic
          & 930{,}372.777 & 643{,}344.609
          & 0.514 & 0.254 & 0.268 & 0.267
          & 0.507 & 0.622 & 0.405 \\
          \bottomrule
      \end{tabular}%
      }
  \end{table}
Table~\ref{tab:scale-choice-performance-raw} contrasts the best and
worst of the five heuristic scale-selection rules for Llama-3.1-8B
at W3 without HIP, together with zero-shot accuracy on six
benchmarks to assess whether perplexity differences carry over
to downstream tasks. To address the second question, we also
include the analytic scale $s^*(M_B)\hat\sigma$.  
\subsection{Preprocessing and scale sensitivity}
Scale selection strongly affects performance, and the asymptotically optimal scale fails on the unprocessed model. Since trained weights need not be iid Gaussian, we ask whether preprocessing can preserve the model's function while bringing its weights closer to the assumptions of Theorem \ref{thm:uniform} and making $s^*(M_B)\hat\sigma$ near-optimal for the finite-width reconstruction objective.

Figure \ref{fig:analyticregret_HIP} shows that, after HIP  (see Appendix \ref{app:HIP} for a short description), the analytic rule
$s^*(M_B)\hat{\sigma}$ consistently matches the best tested heuristic in NLL across the tested models
at W3 and above, without activation-based scoring or scale search.

\begin{figure}[h!]
    \centering
    \includegraphics[width=0.7\linewidth]
        {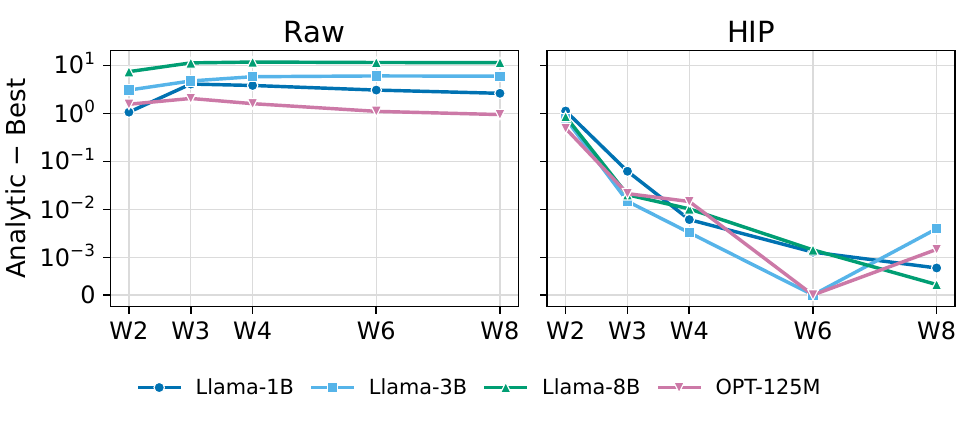}
    \caption{
    HIP substantially reduces the NLL gap between the analytic rule $s^*(M_B)\hat\sigma$ and the best of five heuristics under per-channel GPTQ across tested models and precisions. We report the positive part of this gap, setting it to zero whenever the analytic rule matches or outperforms all five
heuristics. }
    \label{fig:analyticregret_HIP}
\end{figure}

\begin{table}[H]
      \centering
      \caption{Llama-3.1-8B, W3 per-channel GPTQ with HIP.
      Lower is better for perplexity; higher is better for accuracy.
      Best results are shown in bold.}
      \label{tab:scale-choice-performance_HIP}
      \setlength{\tabcolsep}{3pt}
      \resizebox{\linewidth}{!}{%
      \begin{tabular}{@{}lrrccccccc@{}}
          \toprule
          Rule
          & \shortstack{WikiText-2\\PPL}
          & \shortstack{C4\\PPL}
          & PIQA & ARC-E & ARC-C & HellaSwag
          & WinoGrande & BoolQ & Mean \\
          \midrule
          Shrink-2.4
          & \textbf{9.095} & \textbf{14.452}
          & 0.785 & 0.744 & \textbf{0.508} & \textbf{0.740}
          & 0.717 & 0.817 & 0.718 \\

          Proxy-static
          & 9.229 & 14.618
          & 0.781 & 0.757 & 0.488 & 0.739
          & \textbf{0.733} & 0.817 & 0.719 \\

          Min-max
          & 15.396 & 22.775
          & 0.705 & 0.544 & 0.332 & 0.620
          & 0.595 & 0.621 & 0.569 \\

          Analytic
          & 9.214 & 14.602
          & \textbf{0.786} & \textbf{0.761} & 0.506 & 0.737
          & 0.711 & \textbf{0.824} & \textbf{0.721} \\
          \bottomrule
      \end{tabular}%
      }
  \end{table}
  QuIP\#~\cite{tseng2024quip} likewise selects the scales of its vector codebooks from the Gaussian quantization error. Our analysis makes this connection rigorous for uniform scalar quantization of the layerwise reconstruction objective, and our experiments show that the resulting analytic scale transfers, after HIP, to GPTQ (and to QRoNoS and ResComp, see the appendix) with competitive performance and no scale search.

\subsection{Local reconstruction sensitivity}
\label{sec:localreconstr}
We now examine scale sensitivity directly at the level of
the reconstruction objective. For each weight matrix, we
evaluate 32 predetermined rows using the calibration
activations from a sequential GPTQ run. We repeat the
experiment across five models, five precisions, and three
calibration draws.

For each row, let $\hat{s}$ denote the best searched GPTQ
scale and $L(s)$ its RMS-normalized, Hessian-weighted
reconstruction loss. We measure sensitivity to relative
scale perturbations through
\begin{equation}\label{eq:emph_count_of_kappa}
    C_{\mathrm{sym}}(\delta)
:=
\frac{
L(\hat{s}(1+\delta))
+L(\hat{s}(1-\delta))
-2L(\hat{s})
}{\delta^2}.
\end{equation}
For a twice-differentiable loss, this converges to
$\hat{s}^2L''(\hat{s})$ as $\delta\to0$.
For the empirical GPTQ landscapes, we interpret it as a
finite-window sensitivity measure.

Figure~\ref{fig:local-gptq-sensitivity} shows that the median empirical GPTQ sensitivity\footnote{For each model and bit-width, we pool observations from 32 predetermined rows per matrix across all studied matrices and three calibration seeds.} decreases approximately exponentially with bit-width over the tested range, consistently across all three perturbation windows. The theoretical Gaussian curvature exhibits the same trend, and at $\delta=0.05$, the empirical and theoretical quantities also agree closely in magnitude across models. This behavior mirrors the global sensitivity trend in Figure~\ref{fig:spread}. Together, these observations suggest that the Gaussian landscape provides a useful model for the bit-width dependence of GPTQ sensitivity even without HIP preprocessing, despite substantial differences between the empirically optimal scales and their Gaussian predictions.

\begin{figure}[t]
    \centering \includegraphics[width=0.9\linewidth]{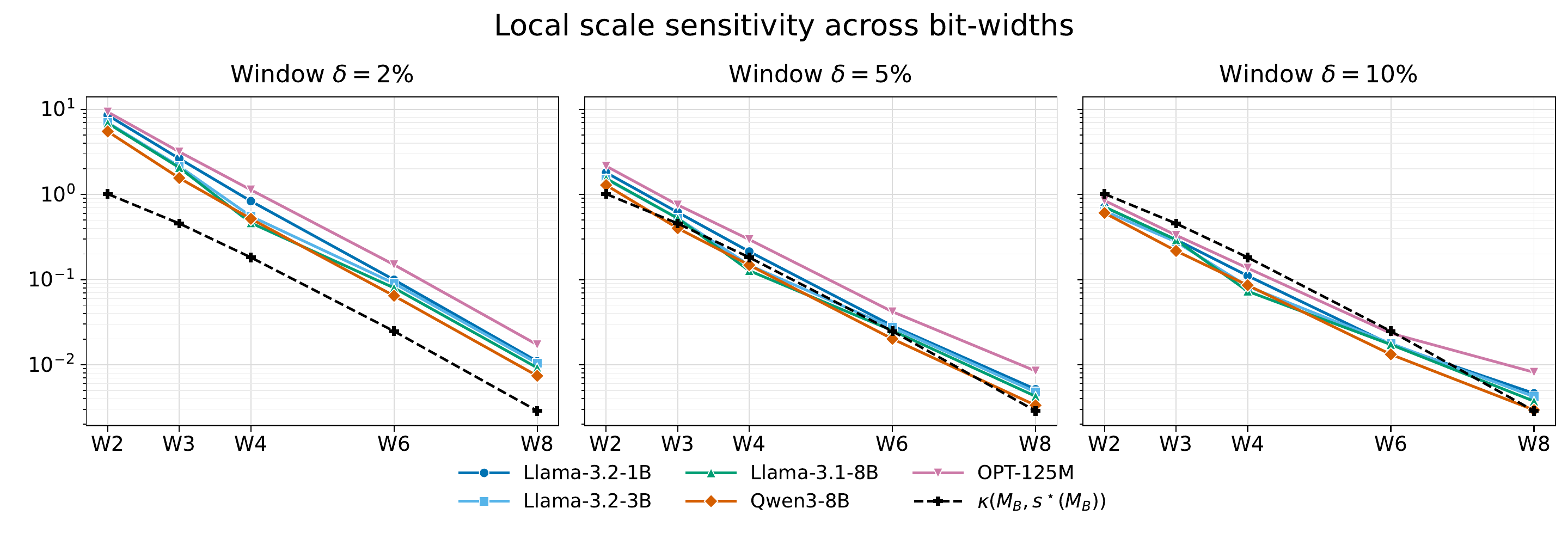}
    \caption{Colored curves: median finite-window GPTQ sensitivity $C_{\mathrm{sym}}(\delta)$ in logscale at rowwise scales optimized by grid search for different $\delta$. Dashed black: Gaussian curvature $\kappa(M_B,s^\star(M_B))$ computed using Algorithm~\ref{alg:kappa}. Same exponential decay with number of bits as the empirical GPTQ sensitivity. The dependence of $C_{\mathrm{sym}}(\delta)$ on the perturbation window indicates that these finite-window measurements cannot be interpreted as a single local quadratic curvature.
    }
    \label{fig:local-gptq-sensitivity}
\end{figure}

Figure~\ref{fig:landscape} illustrates the reconstruction
landscape for one weight row across precisions, with and
without HIP preprocessing. Figure~\ref{fig:localreconstrHIP} further shows that median empirical GPTQ sensitivity after HIP follows a similar bit-width dependence to that observed without preprocessing. This suggests that the reduced performance spread across scale-selection rules after HIP could reflect better agreement among the scales selected by these rules, and not a flatter local landscape. 

Note that, HIP's rotation alone preserves the effective rank of the input Gram matrix $H=XX^\top$. Motivated by the tighter convergence bound at higher effective rank (Theorem~\ref{thm:uniform}), we test whether transformations that change the spectrum improve analytic scale prediction.

We study W2 per-channel GPTQ on Llama-3.1-8B-Instruct,
using a single unrotated sequential quantization stream.
We partition the model's 32 blocks into ten fixed depth
bins. Within each bin, we select the matrix with the largest
normalized reconstruction error after per-row scale search,
then its highest-error row.
We compare four settings: no preprocessing, HIP, diagonal
scaling followed by HIP, and whitening followed by HIP (~\ref{app:HIP}).
All settings use the same pre-target weights and incoming
activations; preprocessing is applied only to the target.

Relative to HIP alone, diagonal scaling and whitening
reduce the median scale gap $\left|\log\left(\frac{s_{\mathrm{search}}}{s_{\mathrm{an}}}\right)\right|$, where $s_{\mathrm{an}}=s^*(M_2)\hat{\sigma}$, from $0.187$ to $0.0298$ and $0.00406$, improving agreement in $8/10$ and $10/10$ selected rows, respectively
(Figure~\ref{fig:w2-preprocessing-depth}). Here, $s_{\mathrm{search}}$ is the best
searched GPTQ scale, not a certified global optimum.
However, using a common normalization across settings,
the best searched reconstruction error increases in every
row, by median factors of $3.63$ and $20.8$ relative to
HIP alone (Table \ref{tab:w2-preprocessing-depth}).

These results reveal a trade-off in the selected cases:
the transformations bring the best searched scale closer
to the analytic prediction, but worsen the reconstruction
error attained by GPTQ even after scale optimization.

\section{Limitation and Future Work}
We establish the first rigorous connection between Gaussian scalar quantization and layerwise post-training quantization at random initialization. Under Gaussian weights and an effective-rank condition on calibration activations, the normalized round-to-nearest reconstruction loss converges uniformly in scale to the Gaussian objective (Theorem \ref{thm:uniform}); this condition holds with high probability for wide MLPs (Theorem \ref{thm:mlpworks}). The objective has a unique, nondegenerate minimizer whose scale-invariant curvature decreases with bit-width (Proposition \ref{prop:scalarlandscape}; Section \ref{sec:limfuncland}), explaining strong scale sensitivity at 2--3 bits and weaker sensitivity from 6 bits onward, consistent with LLM experiments. The analysis also shed light on why HIP improves quantization: most scale-selection rules target near-optimal scales for Gaussian weights, and HIP brings weights closer to this regime. We stress the fact that convergence is not guarantee but Gaussianity of the weights alone for iterative PTQ. Finally, our experiments reveal that naively increasing the effective rank improves scale predictability but removes low-eigenvalue directions that GPTQ exploits for error compensation.

\paragraph{Limitations \& Future works.} Our theory concerns randomly initialized networks and does not model training or its effects on weight distributions and activation geometry. We therefore do not establish whether the Gaussian and effective-rank assumptions hold after training; the experiments on pretrained LLMs provide empirical evidence beyond our theoretical guarantees. Moreover, the analysis is specific to round-to-nearest (RTN) quantization and does not cover other strategies, including error-compensation methods such as GPTQ. Extending the theory to trained networks and more general quantization methods remains an open direction.
In our experiments HIP makes weight distributions more Gaussian, improving alignment with (H1) and GPTQ performance, whereas the transformations we tested to increase the effective rank, motivated by (H2), degrade performance when combined with HIP. This calls for preprocessing that improves Gaussianity and effective rank jointly without sacrificing GPTQ's error compensation. On the theoretical side, extending our analysis to trained weights and to compensated methods such as GPTQ would clarify how scale selection, spectral structure and error compensation interact. Finally, the approximately exponential decay of relative-scale sensitivity in the Gaussian model provides a quantitative reference for future PTQ methods: a natural goal is lower sensitivity at low precision, under matched bit-widths and loss normalization, without loss of reconstruction quality or end-to-end performance.
\newpage
\subsection*{AI use statement}

In this work, we used generative AI tools for polishing code, improving sentence clarity, and refining grammar. Moreover, discussion with AI has been used to provide insight on how to prove positiveness and decreasing behaviour of (ii) and (iii) in Proposition \ref{prop:scalarlandscape}. The conversation with AI tools has not produced the proof but just the idea, that has been carefully written by the authors. Therefore, we have not used generative AI tools for creating results and proof in a copy and paste manner. We have carefully reviewed all AI-assisted work: in particular, the LLM-generated part of the code used in the experiments was verified and tested for correctness. We take responsibility for the final content of this work, including text and claims or artifacts produced with the aid of generative AI.

\subsection*{Ethics statement}
This work focuses on the theoretical analysis of scale sensitivity and quantization objective in machine learning model quantization and does not involve experiments on human subjects, sensitive personal data, or applications with direct societal risks. The datasets referenced are publicly available, and no private or restricted data was used. Potential ethical concerns related to misuse are minimal, as the contributions are mainly theoretical.

\subsection*{Reproducibility statement}
We have taken multiple steps to ensure reproducibility of our results. All theoretical claims are accompanied by rigorous proofs, presented in detail in the appendix. Assumptions underlying the theorems are explicitly stated, and definitions are given in full to allow independent verification. The code used for the experiment is provided as anonymous supplementary material at
\url{https://anonymous.4open.science/r/SOQ-345F}.

\subsection*{Acknowledgments}
Jonas von Berg, Massimiliano Datres and Gitta Kutyniok acknowledge support by the project ”Next Generation AI Computing (gAIn),” funded by the Bavarian Ministry of Science and the Arts and the Saxon Ministry for Science, Culture, and Tourism as well as by the Hightech Agenda Bavaria. 

Jonas von Berg, Carlo Kneißl and Gitta Kutyniok are also grateful for partial support from the Konrad Zuse School of Excellence in Reliable AI (DAAD). 

Additionally, Jonas von Berg, Massimiliano Datres, Carlo Kneißl and Gitta Kutyniok acknowledge support by the Munich Center for Machine Learning (MCML). 

Carlo Kneissl and Gitta Kutyniok acknowledge support by the project ”Genius Robot” (01IS24083), funded by
the Federal Ministry of Education and Research (BMBF), as well as the ONE Munich Strategy Forum (LMU Munich,
TU Munich, and the Bavarian Ministery for Science and Art)

Gitta Kutyniok furthermore acknowledges support by the German Research Foundation under Grants DFG-SPP-2298, KU 1446/31-1 and KU 1446/32-1, and by the Bavarian Ministry for Digital Affairs.

\newpage
\bibliography{iclr2027_conference}

@article{DBLP:journals/corr/abs-2407-21783,
  publtype={informal},
  author={Abhimanyu Dubey and Abhinav Jauhri and et al.},
  title={The Llama 3 Herd of Models},
  year={2024},
  cdate={1704067200000},
  journal={CoRR},
  volume={abs/2407.21783},
  url={https://doi.org/10.48550/arXiv.2407.21783}
}

@article{liu2024deepseek,
  title={Deepseek-v3 technical report},
  author={Liu, Aixin and Feng, Bei and Xue, Bing and Wang, Bingxuan and Wu, Bochao and Lu, Chengda and Zhao, Chenggang and Deng, Chengqi and Zhang, Chenyu and Ruan, Chong and others},
  journal={arXiv preprint arXiv:2412.19437},
  year={2024}
}

@article{ma2025bitnet,
  title={BitNet b1. 58 2B4T Technical Report},
  author={Ma, Shuming and Wang, Hongyu and Huang, Shaohan and Zhang, Xingxing and Hu, Ying and Song, Ting and Xia, Yan and Wei, Furu},
  journal={arXiv preprint arXiv:2504.12285},
  year={2025}
}

@book{vershynin2018,
  author    = {Vershynin, Roman},
  title     = {High-Dimensional Probability: An Introduction with Applications in Data Science},
  series    = {Cambridge Series in Statistical and Probabilistic Mathematics},
  publisher = {Cambridge University Press},
  year      = {2018},
}

@article{laurent2000,
  author  = {Laurent, B{\'e}atrice and Massart, Pascal},
  title   = {Adaptive estimation of a quadratic functional by model selection},
  journal = {Annals of Statistics},
  volume  = {28},
  number  = {5},
  pages   = {1302--1338},
  year    = {2000}
}

@book{wainwright2019high,
  title={High-dimensional statistics: A non-asymptotic viewpoint},
  author={Wainwright, Martin J},
  volume={48},
  year={2019},
  publisher={Cambridge university press}
}

@article{frantar2022gptq,
  title={Gptq: Accurate post-training quantization for generative pre-trained transformers},
  author={Frantar, Elias and Ashkboos, Saleh and Hoefler, Torsten and Alistarh, Dan},
  journal={arXiv preprint arXiv:2210.17323},
  year={2022}
}

@article{lin2024awq,
  title={Awq: Activation-aware weight quantization for on-device llm compression and acceleration},
  author={Lin, Ji and Tang, Jiaming and Tang, Haotian and Yang, Shang and Chen, Wei-Ming and Wang, Wei-Chen and Xiao, Guangxuan and Dang, Xingyu and Gan, Chuang and Han, Song},
  journal={Proceedings of machine learning and systems},
  volume={6},
  pages={87--100},
  year={2024}
}

@inproceedings{dettmers2023case,
  title={The case for 4-bit precision: k-bit inference scaling laws},
  author={Dettmers, Tim and Zettlemoyer, Luke},
  booktitle={International Conference on Machine Learning},
  pages={7750--7774},
  year={2023},
  organization={PMLR}
}

@inproceedings{chen2026geometry,
  title={The geometry of llm quantization: Gptq as babai's nearest plane algorithm},
  author={Chen, Jiale and Shabanzadeh, Yalda and Crn{\v{c}}evi{\'c}, Elvir and Hoefler, Torsten and Alistarh, Dan},
  booktitle={International Conference on Learning Representations},
  volume={2026},
  pages={122653--122695},
  year={2026}
}

@inproceedings{zhang2026qronos,
  title={Qronos: Correcting the past by shaping the future... in post-training quantization},
  author={Zhang, Shihao and Zhang, Haoyu and Colbert, Ian and Saab, Rayan},
  booktitle={International Conference on Learning Representations},
  volume={2026},
  pages={87775--87798},
  year={2026}
}

@article{tseng2024quip,
  title={Quip\#: Even better llm quantization with hadamard incoherence and lattice codebooks},
  author={Tseng, Albert and Chee, Jerry and Sun, Qingyao and Kuleshov, Volodymyr and De Sa, Christopher},
  journal={Proceedings of machine learning research},
  volume={235},
  pages={48630},
  year={2024}
}

@article{banner2019post,
  title={Post training 4-bit quantization of convolutional networks for rapid-deployment},
  author={Banner, Ron and Nahshan, Yury and Soudry, Daniel},
  journal={Advances in neural information processing systems},
  volume={32},
  year={2019}
}

@article{nagel2021white,
  title={A white paper on neural network quantization},
  author={Nagel, Markus and Fournarakis, Marios and Amjad, Rana Ali and Bondarenko, Yelysei and Van Baalen, Mart and Blankevoort, Tijmen},
  journal={arXiv preprint arXiv:2106.08295},
  year={2021}
}

@article{frantar2022optimal,
  title={Optimal brain compression: A framework for accurate post-training quantization and pruning},
  author={Frantar, Elias and Alistarh, Dan},
  journal={Advances in Neural Information Processing Systems},
  volume={35},
  pages={4475--4488},
  year={2022}
}

@article{hassibi1992second,
  title={Second order derivatives for network pruning: Optimal brain surgeon},
  author={Hassibi, Babak and Stork, David},
  journal={Advances in neural information processing systems},
  volume={5},
  year={1992}
}

@article{li2025gptaq,
  title={Gptaq: Efficient finetuning-free quantization for asymmetric calibration},
  author={Li, Yuhang and Yin, Ruokai and Lee, Donghyun and Xiao, Shiting and Panda, Priyadarshini},
  journal={arXiv preprint arXiv:2504.02692},
  year={2025}
}

@article{li2026rethinking,
  title={Rethinking Residual Errors in Compensation-based LLM Quantization},
  author={Li, Shuaiting and Deng, Juncan and Xu, Kedong and Deng, Rongtao and Gu, Hong and Jiang, Minghan and Shen, Haibin and Huang, Kejie},
  journal={arXiv preprint arXiv:2604.07955},
  year={2026}
}

@article{chee2023quip,
  title={Quip: 2-bit quantization of large language models with guarantees},
  author={Chee, Jerry and Cai, Yaohui and Kuleshov, Volodymyr and De Sa, Christopher M},
  journal={Advances in neural information processing systems},
  volume={36},
  pages={4396--4429},
  year={2023}
}

@article{ashkboos2024quarot,
  title={Quarot: Outlier-free 4-bit inference in rotated llms},
  author={Ashkboos, Saleh and Mohtashami, Amirkeivan and Croci, Maximilian L and Li, Bo and Cameron, Pashmina and Jaggi, Martin and Alistarh, Dan and Hoefler, Torsten and Hensman, James},
  journal={Advances in Neural Information Processing Systems},
  volume={37},
  pages={100213--100240},
  year={2024}
}

@article{zhang2024magr,
  title={Magr: Weight magnitude reduction for enhancing post-training quantization},
  author={Zhang, Aozhong and Wang, Naigang and Deng, Yanxia and Li, Xin and Yang, Zi and Yin, Penghang},
  journal={Advances in neural information processing systems},
  volume={37},
  pages={85109--85130},
  year={2024}
}

@inproceedings{xiao2023smoothquant,
  title={Smoothquant: Accurate and efficient post-training quantization for large language models},
  author={Xiao, Guangxuan and Lin, Ji and Seznec, Mickael and Wu, Hao and Demouth, Julien and Han, Song},
  booktitle={International conference on machine learning},
  pages={38087--38099},
  year={2023},
  organization={PMLR}
}

@inproceedings{liu2025spinquant,
  title={Spinquant: Llm quantization with learned rotations},
  author={Liu, Zechun and Zhao, Changsheng and Fedorov, Igor and Soran, Bilge and Choudhary, Dhruv and Krishnamoorthi, Raghuraman and Chandra, Vikas and Tian, Yuandong and Blankevoort, Tijmen},
  booktitle={International Conference on Learning Representations},
  volume={2025},
  pages={92009--92032},
  year={2025}
}

@article{li2021brecq,
  title={Brecq: Pushing the limit of post-training quantization by block reconstruction},
  author={Li, Yuhang and Gong, Ruihao and Tan, Xu and Yang, Yang and Hu, Peng and Zhang, Qi and Yu, Fengwei and Wang, Wei and Gu, Shi},
  journal={arXiv preprint arXiv:2102.05426},
  year={2021}
}

@inproceedings{nagel2020up,
  title={Up or down? adaptive rounding for post-training quantization},
  author={Nagel, Markus and Amjad, Rana Ali and Van Baalen, Mart and Louizos, Christos and Blankevoort, Tijmen},
  booktitle={International conference on machine learning},
  pages={7197--7206},
  year={2020},
  organization={PMLR}
}

@inproceedings{shao2024omniquant,
  title={Omniquant: Omnidirectionally calibrated quantization for large language models},
  author={Shao, Wenqi and Chen, Mengzhao and Zhang, Zhaoyang and Xu, Peng and Zhao, Lirui and Li, Zhiqian and Zhang, Kaipeng and Peng, Gao and Qiao, Yu and Luo, Ping},
  booktitle={International Conference on Learning Representations},
  volume={2024},
  pages={45472--45496},
  year={2024}
}

@article{amboage2026optimal,
  title={Optimal Post-Training Quantization Scales and Where to Find Them},
  author={Amboage, Juan and Monteagudo-Lago, Pablo and Colbert, Ian and Franco, Giuseppe and Fraser, Nicholas},
  journal={arXiv preprint arXiv:2606.10890},
  year={2026}
}

@article{banner2018aciq,
  title={Aciq: Analytical clipping for integer quantization of neural networks},
  author={Banner, Ron and Nahshan, Yury and Hoffer, Elad and Soudry, Daniel},
  year={2018}
}

@article{nahshan2021loss,
  title={Loss aware post-training quantization},
  author={Nahshan, Yury and Chmiel, Brian and Baskin, Chaim and Zheltonozhskii, Evgenii and Banner, Ron and Bronstein, Alex M and Mendelson, Avi},
  journal={Machine Learning},
  volume={110},
  number={11},
  pages={3245--3262},
  year={2021},
  publisher={Springer}
}

@article{max1960quantizing,
  title={Quantizing for minimum distortion},
  author={Max, Joel},
  journal={IRE Transactions on Information Theory},
  volume={6},
  number={1},
  pages={7--12},
  year={1960},
  publisher={IEEE}
}

@article{na2017convexity,
  title={On the convexity of the MSE distortion of symmetric uniform scalar quantization},
  author={Na, Sangsin and Neuhoff, David L},
  journal={IEEE Transactions on Information Theory},
  volume={64},
  number={4},
  pages={2626--2638},
  year={2017},
  publisher={IEEE}
}

@article{na2018monotonicity,
  title={Monotonicity of step sizes of MSE-optimal symmetric uniform scalar quantizers},
  author={Na, Sangsin and Neuhoff, David L},
  journal={IEEE Transactions on Information Theory},
  volume={65},
  number={3},
  pages={1782--1792},
  year={2018},
  publisher={IEEE}
}
\bibliographystyle{iclr2027_conference}

\newpage
\appendix
\startcontents[appendix]
\printcontents[appendix]{}{1}{\section*{Appendix}}

\section{Theory}
\subsection{Known results}

In this section we collect the probabilistic tools used throughout the paper.
Throughout the paper the sub-gaussian and sub-exponential norms are understood in the following Orlicz sense, which is the one used in \cite{vershynin2018}. Most of the results appearing in this section are taken from \cite{vershynin2018}.

\begin{defn}\label{def:psi2}
For a real random variable $Y$ we set
\[
\|Y\|_{\psi_2} := \inf\left\{u>0: \ \E\left[e^{Y^2/u^2}\right] \le 2 \right\} ,
\]
with the convention $\inf\emptyset = +\infty$, and we call $Y$ sub-gaussian if $\|Y\|_{\psi_2}<\infty$.
\end{defn}

\begin{defn}\label{def:psi1}
For a real random variable $Y$ we set
\[
\|Y\|_{\psi_1} := \inf\left\{u>0: \ \E\left[e^{|Y|/u}\right] \le 2 \right\} ,
\]
with the convention $\inf\emptyset = +\infty$, and we call $Y$ sub-exponential if $\|Y\|_{\psi_1}<\infty$.
\end{defn}
Gaussian random variables are sub-gaussian \cite[Ex.~2.5.8]{vershynin2018}; since the value of their $\psi_2$-norm enters the constants of our bounds, we compute it exactly.
\begin{lemma}\label{lem:gausspsi2}
If $Y\sim\cN(0,\tau^2)$ with $\tau>0$, then $\|Y\|_{\psi_2} = \sqrt{8/3}\,\tau$. Moreover, $\|aY\|_{\psi_2} = |a|\,\|Y\|_{\psi_2}$ for every $a\in\R$.
\end{lemma}
\begin{proof}
For $u>\sqrt{2}\,\tau$ one has $\E[e^{Y^2/u^2}] = (1-2\tau^2/u^2)^{-1/2}$, and $\E[e^{Y^2/u^2}]=+\infty$ otherwise. Hence $\E[e^{Y^2/u^2}]\le 2$ if and only if $1-2\tau^2/u^2 \ge 1/4$, i.e.\ $u^2 \ge \tfrac{8}{3}\tau^2$. The homogeneity is immediate from the definition.
\end{proof}
The square of a sub-gaussian random variable is sub-exponential, with an exact relation between the two norms.
\begin{lemma}
\label{lem:subexp-subgauss-squared}
A random variable $X$ is sub-gaussian if and only if $X^2$ is sub-exponential. Moreover,
\[
\|X^2\|_{\psi_1} = \|X\|_{\psi_2}^2 .
\]
\end{lemma}
The following lemma shows that centering does not increase the sub-exponential norm by more than a constant factor, which we make explicit.
\begin{lemma}
\label{lem:centering}
Let $X$ be a sub-exponential random variable. Then $X-\E X$ is sub-exponential as well, and
\[
\|X-\E X\|_{\psi_1}\ \le\ \Big(1+\frac{1}{\log 2}\Big)\,\|X\|_{\psi_1}\ \le\ 3\,\|X\|_{\psi_1}\, .
\]
\end{lemma}
\begin{proof}
Since $\|\cdot\|_{\psi_1}$ is a norm, the triangle inequality gives $\|X-\E X\|_{\psi_1}\le\|X\|_{\psi_1}+\|\E X\|_{\psi_1}$. For a deterministic $a\in\R$ we have $e^{|a|/t}\le 2$ if and only if $t\ge |a|/\log 2$, hence $\|a\|_{\psi_1}=|a|/\log 2$. It remains to bound $|\E X|$ by $\|X\|_{\psi_1}$. Let $t>\|X\|_{\psi_1}$, so that $\E e^{|X|/t}\le 2$. Using $e^{x}\ge 1+x$ we get $2\ge\E e^{|X|/t}\ge 1+\E|X|/t$, i.e.\ $\E|X|\le t$. Letting $t\downarrow\|X\|_{\psi_1}$ yields $|\E X|\le\E|X|\le\|X\|_{\psi_1}$. Combining the three estimates,
\[
\|X-\E X\|_{\psi_1}\ \le\ \|X\|_{\psi_1}+\frac{|\E X|}{\log 2}\ \le\ \Big(1+\frac{1}{\log 2}\Big)\|X\|_{\psi_1}\, ,
\]
and $1+1/\log 2$.
\end{proof}

We shall use the following concentration inequalities for sums and quadratic forms of independent sub-gaussian or sub-exponential random variables.

\begin{theorem}[Bernstein's inequality]
\label{cor:bernstein}
Let $X_1,\dots,X_N$ be independent, mean-zero, sub-exponential random variables, and let $K:=\max_i\|X_i\|_{\psi_1}$. Then, for every $t\ge0$,
\[
\P\left(\Big|\sum_{i=1}^N X_i\Big|\ge t\right)
\ \le\ 2\exp\left[-c\min\left(\frac{t^2}{K^2N},\ \frac{t}{K}\right)\right].
\]
\end{theorem}

\begin{theorem}[Hanson--Wright inequality]\label{thm:HanWrig}
Let $X=(X_1,\ldots,X_n)\in\R^n$ be a random vector with independent components satisfying $\E[X_i]=0$ and $\|X_i\|_{\psi_2}\le K$ for every $i\in[n]$, and let $A\in\R^{n\times n}$. Then, for every $t\ge 0$,
\begin{align*}
\P\left(\left|X^\top A X-\E\left[X^\top A X\right]\right|>t\right) \le2\exp\left(-c\min\left\{\frac{t^2}{K^4\|A\|_{F}^2},\ \frac{t}{K^2\|A\|_{op}} \right\}\right),
\end{align*}
where $\|A\|_{F}$ and $\|A\|_{op}$ denote the Frobenius and the operator norm of $A$.
\end{theorem}

We recall some tail bounds for the sup-norm and the Euclidean norm of a standard Gaussian vector. The first is a union bound over the standard Gaussian tail $\P(|g|>t)\le 2e^{-t^2/2}$, $t\ge0$ \cite[Prop.~2.1.2]{vershynin2018}; the second is the sharp $\chi^2$ tail bound of Laurent and Massart \cite[Lemma~1]{laurent2000}.

\begin{lemma}\label{lem:norm_of_gaussian_vectors}
Let $X\sim\cN(0,\Id_d)$. Then, for every $d\ge2$,
\begin{align*}
\P\left(\|X\|_\infty > \log d\right)\ \le\ 2d\,e^{-\log^2 d/2} ,
\end{align*}
and, for every $x>0$,
\begin{align*}
\P\left(\|X\|_2^2 \ge d+2\sqrt{dx}+2x\right)\le e^{-x},
\qquad
\P\left(\|X\|_2^2 \le d-2\sqrt{dx}\right)\le e^{-x}.
\end{align*}
In particular, choosing $x=d/16$, each of the two inequalities $\tfrac12 d\le\|X\|_2^2\le 2d$ fails with probability at most $e^{-d/16}$.
\end{lemma}
\begin{proof}
For the first claim, by a union bound and the Gaussian tail bound,
\[
\P\left(\|X\|_\infty > \log d\right) \le \sum_{i=1}^d \P\left( |X_i| > \log d\right)\le 2d\,e^{-\log^2d/2}.
\]
The two-sided bound on $\|X\|_2^2\sim\chi^2_d$ is in \cite[Lemma~1]{laurent2000}. For the last claim, with $x=d/16$ one has $d-2\sqrt{dx}=d/2$ and $d+2\sqrt{dx}+2x=\tfrac{13}{8}d\le 2d$.
\end{proof}

Next, we recall that the operator norm of an $m\times n$ random matrix with independent sub-gaussian entries is of order $\sqrt m+\sqrt n$ with high probability.

\begin{theorem}
\label{thm:subgaussian-norm}
Let $A$ be an $m \times n$ random matrix with independent, mean-zero, sub-gaussian entries $A_{ij}$, and let $K := \max_{i,j} \|A_{ij}\|_{\psi_2}$. Then, for any $t > 0$,
\[
\|A\|_{op} \le CK \left( \sqrt{m} + \sqrt{n} + t \right)
\]
with probability at least $1 - 2\exp(-t^2)$.
\end{theorem}

Finally, we recall the Gaussian concentration inequality for Lipschitz functions. These references state the inequality for Euclidean-Lipschitz functions of a standard Gaussian vector, but, since the Frobenius norm is the Euclidean norm on $\R^{m\times n}\cong\R^{mn}$ and the entries of $G$ are i.i.d.\ $\cN(0,1)$, the matrix version below is the same statement with the same constants.

\begin{lemma}[Theorem 2.26 in \cite{wainwright2019high}]\label{prop:concLip}
Let $F \colon \R^{m \times n} \to \R$ be $K$-Lipschitz with respect to the Frobenius norm, i.e.\ $|F(A)-F(B)| \le K\|A-B\|_F$ for all $A,B \in \R^{m \times n}$, and let $G \in \R^{m \times n}$ be a random matrix with i.i.d.\ $\cN(0,1)$ entries. Then, for every $t \ge 0$,
\[
\P\left\{|F(G)-\E F(G)|>t\right\}\le 2\exp\left(-\frac{t^2}{2K^2}\right).
\]
\end{lemma}

Finally, we recall a known result on randomly initialized MLPs, which says that, with high probability over the weights, the Euclidean norm of each hidden layer grows at most like the square root of the width. 
\begin{lemma}\label{lem:boundonact}
Under (H1), (H3) and (H4), for all $x\in \cX$ 
and $\ell\in[L-1]$, there exists a constant $C>0$ not dependent on $d$ such that 
\begin{align*}
\P_{\vartheta}\left(\|\alpha^{(\ell)}(x, \vartheta^{(\ell)})\|_2 \ge  C\sqrt{d'}\right) \le 2\ell e^{-d'} \, .
\end{align*}
\end{lemma}
\begin{proof}
We proceed by induction. We start by the base case $\ell = 1$. We notice that, using Lemma \ref{lem:norm_of_gaussian_vectors} and Theorem \ref{thm:subgaussian-norm}, for $x\in \cX$, it holds
\begin{align*}
\left\|\alpha^{(1)}\left(x, \vartheta^{(1)}\right)\right\|^2_2 & = \sum_{i=1}^{d_1} \sigma\left(\frac{1}{\sqrt{d_0}}\vartheta^{(1)}_{i:}x\right)^2 \\
    & \le L^2_\sigma \sum_{i=1}^{d_1} \left(\frac{1}{\sqrt{d_0}}\vartheta^{(1)}_{i:}x\right)^2\\
    & \le L^2_\sigma \left\|\frac{1}{\sqrt{d_0}}\vartheta^{(1)}\right\|^2_{op} \|x\|^2_2
    \le L^2_\sigma\, \left\|\vartheta^{(1)}\right\|^2_{op} \|x\|_2^2 \, .
\end{align*}
Since the entries of $\vartheta^{(1)}$ are independent, centered and, by Lemma \ref{lem:gausspsi2} applied with $\tau^2 = 1$, $\|\vartheta^{(1)}_{ij}\|_{\psi_2} = \sqrt{8/3}$ for all $i\in[d_1]$, $j\in[d_0]$. Combining Theorem \ref{thm:subgaussian-norm} with $K = \sqrt{8/3}$ and $t = \sqrt{d_1}$ and Lemma \ref{lem:norm_of_gaussian_vectors}, we obtain that there exists a constant $C>0$ not dependent on $d$ such that
\begin{align*}
\P_\vartheta\left(\left\|\alpha^{(1)}\left(x, \vartheta^{(1)}\right)\right\|_2 \ge C \sqrt{d}\right) \le 2 e^{-d}\, .
\end{align*}
Inductively, assume that, for all $x\in \cX$, there exists constants $C,c>0$ not dependent on $d$ such that $\left\|\alpha^{(\ell-1)}\left(x, \vartheta^{(\ell-1)}\right)\right\|_2 \le  C\sqrt{d}$ with probability at least  $1 - 2(L-1) e^{-cd}$. Then, using Theorem \ref{thm:subgaussian-norm} with $K = \sqrt{8/3}$ and $t = \sqrt{d_{\ell}}$ as in the base case, together with the independence of $\vartheta^{(\ell)}$ from $\alpha^{(\ell-1)}$, we conclude that
\begin{align*}
&\P_\vartheta\left(\left\|\alpha^{(\ell)}\left(x, \vartheta^{(\ell)}\right)\right\|_2 \ge C \sqrt{d}\right) \\
& \qquad \le \P_\vartheta\left(L_\sigma \left\|\frac{1}{\sqrt{d_{\ell-1}}}\vartheta^{(\ell)}\right\|_{op} \|\alpha^{(\ell-1)}\|_2 \ge C \sqrt{d}\right) \\
& \qquad \le \P_\vartheta\left(\left\|\frac{1}{\sqrt{d_{\ell-1}}}\vartheta^{(\ell)}\right\|_{op} \ge C\right) + 2(L-1)e^{-d'} \\
& \qquad \le 2L e^{-d}\, ,
\end{align*}
for some constant $C>0$ not dependent on $d$.
\end{proof}

\subsection{Proofs of Theorem~\ref{thm:uniform}}\label{app:proofthmuniform}
To improve readability, we sometimes omit the constants which do not depend on $d$ and $N_c$ form the formulas, since they don't affect the asymptotic behavior. 
We define, for each $\ell\in [L]$, the quantization residual $r^{(\ell)}: (0, \infty)\times \R\to \R$ as
\begin{align}\label{eq:defres}
   r\left(s, t\right) := t - s\roundB{t/s}\, .
\end{align}
Whenever $r$ is applied to a multidimensional object, it acts componentwise.
In the next lemma, we show some basic properties of the quantized residual evaluated at initialization.

\begin{lemma}\label{lem:everyissub}
    Let $\vartheta\in\R^d$ with $\vartheta_i\overset{iid}{\sim}\cN(0, 1)$. Then, for all $s>0$, the residual satisfies the following properties
    \begin{enumerate}
        \item $r(s, \vartheta_i)$ is odd, i.e.\ $r(s, \vartheta_i) = -r(s, -\vartheta_i)$, for all $i\in[d]$.
        \item $|r(s, \vartheta_i)|\le |\vartheta_i|$ for all $i\in[d]$.
        \item For any fixed $x\in \R$, the function $s\mapsto r^2(s, x)$ is $2M_B |x|$-Lipschitz continuous.
        \item $r\left(s, \vartheta_i\right)$, $i\in[d]$, are independent, centered and sub-Gaussian; more precisely, for all $i\in[d]$, it holds
        \begin{align*}
\E_{\vartheta}\left[r\left(s, \vartheta_i\right)\right] = 0  \, , \qquad \left\|r\left(s, \vartheta_i\right)\right\|_{\psi_2} \le \sqrt{\dfrac{8}{3}} \, .
        \end{align*}
    \end{enumerate}
\end{lemma}
\begin{proof}
We prove each property separately.
\begin{enumerate}
    \item Property 1. follows immediately by oddness and symmetry of $\roundB{\cdot}$. Indeed, for all $i\in[d]$, we have
    \begin{align*}
      r(s, -\vartheta_i) =   -\vartheta_i - \Pi_{s, M_B}(-\vartheta_i) = -\vartheta_i + \Pi_{s, M_B}(\vartheta_i)  = -r(s, \vartheta_i)\, .
    \end{align*}
    \item For all $i\in[d]$, it holds 
    \begin{align}\label{eq:normofquantres}
        \left| r(s, \vartheta_i) \right|
        = \left| \vartheta_i - \Pi_{s, M_B}(\vartheta_i) \right|
        \overset{(*)}{=} \operatorname{dist}\!\left(\vartheta_i,\, s \left\{-\MB,\dots,\MB\right\}\right) \le \left| \vartheta_i \right| \, ,
    \end{align}
    where $(*)$ follows from the definition \eqref{def:roundB} of $\roundB{\cdot}$, while the last inequality holds because $0$ is a quantization level, i.e.\ $0\in s\{-\MB,\dots,\MB\}$.
    \item Let us assume without loss of generality that $x>0$. The case with $x<0$ follows by oddness of $\Pi_{s,M_B}$. Let us fix $n\in[\MB]$, and denote with $\sigma_n = \tfrac{x}{n-\frac{1}{2}}$. On $(\sigma_{n+1},\sigma_n]$ (with $\sigma_{\MB+1}:=0$) we
have $\lfloor x/s\rceil_B=n$, hence $r^2(s,x)=(x-ns)^2$ there;
on $(\sigma_n,\sigma_{n-1}]$ (with $\sigma_0:=\infty$) we have
$\lfloor x/s\rceil_B=n-1$, hence $r^2(s,x)=(x-(n-1)s)^2$ there.
Both expressions are polynomials in $s$, so $r^2(\cdot,x)$ is continuous
on each of the two half-open intervals, and in particular left-continuous at
$\sigma_n$. For right-continuity at $\sigma_n$ we use
$x=(n-\tfrac12)\sigma_n$ to compute
\[
\lim_{s\to \sigma_n^+} r^2(s,x)=\big(x-(n-1)\sigma_n\big)^2
=\Big(\frac{\sigma_n}{2}\Big)^2
=\big(x-n\sigma_n\big)^2 = r^2(\sigma_n,x).
\]
Hence $r^2(\cdot,x)$ is continuous at every $\sigma_n$, and therefore on
all of $(0,\infty)$. On each interval $ (\sigma_{n+1},\sigma_n)$, $r^2(s, x)$ is differentiable, with derivative
\begin{align*}
    \left|\frac{d}{ds}\, r^2(s, x)\right|
=2n \bigl|x - sn\bigr|\le 2M_B|x| \, ,
\end{align*}
where we used $n\le M_B$ and $|x-sn|=|r(s,x)|\le|x|$ by Property 2.
A continuous function on an interval that is $\mathcal C^1$ off finitely many points with derivative bounded by $L$ is $L$-Lipschitz (apply the mean value theorem between
consecutive breakpoints and sum), which gives the claim.
\item Independence follows from the independence of the components of $\vartheta$ and from the point-wise nature of $r(s, \vartheta)$. Since the Gaussian distribution is symmetric and $\Pi_{s, M_B}$ is odd, we have that
\begin{align*}
        \E_{\vartheta}\left[r(s, \vartheta_i)\right]
        = \E_{\vartheta}\left[\vartheta_i\right] - \E_{\vartheta}\left[\Pi_{s,M_B}(\vartheta_i)\right] = 0 \, .
    \end{align*}
    From the point-wise bound in Property 2, we get $\E_\vartheta\big[e^{r^2(s,\vartheta_i)/u^2}\big] \le \E_\vartheta\big[e^{\vartheta_i^2/u^2}\big]$ for every $u>0$, hence
    \begin{align*}
        \left\| r(s, \vartheta_i) \right\|_{\psi_2}
        &= \inf\left\{u>0: \ \E_{\vartheta_i}\left[e^{r^2(s, \vartheta_i)/u^2}\right] \le 2\right\} \\
        &\le \inf\left\{u>0: \ \E_{\vartheta_i}\left[e^{\vartheta_i^2/u^2}\right] \le 2\right\}
         = \left\|\vartheta_i\right\|_{\psi_2}
         = \sqrt{\frac{8}{3}}\, ,
    \end{align*}
    where the last equality is Lemma \ref{lem:gausspsi2} applied with $\tau^2 = 1$. As $i$ was arbitrary, this concludes the proof.
\end{enumerate}

\end{proof}
For clarity, we recall here Theorem \ref{thm:uniform}.

\uniformtheorem*

\begin{proof}
Let us set $H:=XX^\top\in\R^{d\times d}$ and abbreviate $\reff:=\reff(H)$.  We note that
\begin{align}\label{eq:rec_splitted}
    \cE(s)
    &= \frac{1}{\operatorname{Tr}(H)}\sum_{i=1}^d H_{ii} \bigl(\vartheta_i - \Pi_{s,M_B}(\vartheta_i)\bigr)^2 \\
    & \notag\quad+\frac{1}{\operatorname{Tr}(H)}\sum_{i=1}^d\sum_{j\ne i}^d H_{ij} \bigl(\vartheta_i - \Pi_{s,M_B}(\vartheta_i)\bigr)\bigl(\vartheta_j - \Pi_{s,M_B}(\vartheta_j)\bigr) \\
    & \notag= D_B(s, \vartheta, X) + O_B(s, \vartheta, X) \, ,
\end{align}
where
\begin{align*}
   &  D_B(s, \vartheta, X) := \frac{1}{\operatorname{Tr}(H)}\sum_{i=1}^d H_{ii} \bigl(\vartheta_i - \Pi_{s,M_B}(\vartheta_i)\bigr)^2 \\
   & O_B(s, \vartheta, X) := \frac{1}{\operatorname{Tr}(H)}\sum_{i=1}^d\sum_{j\ne i}^d H_{ij} \bigl(\vartheta_i - \Pi_{s,M_B}(\vartheta_i)\bigr)\bigl(\vartheta_j - \Pi_{s,M_B}(\vartheta_j)\bigr) \, .
\end{align*}
Since the $\vartheta_i$ are independent and, by Lemma~\ref{lem:everyissub}, the residuals $r(s,\vartheta_i)$ are centered, we have $\E_\vartheta[O_B(s,\vartheta,X)]=0$. Hence, it holds
\begin{align*}
&\cE^{\infty}(s)=\E_\vartheta\big[D_B(s,\vartheta,X)\big]\\
&\cE(s)-\cE^\infty(s)
=D_B(s,\vartheta,X)-\E_\vartheta[D_B(s,\vartheta,X)]+O_B(s,\vartheta,X).    
\end{align*}

Let us start by controlling $D_B(s, \vartheta, X)$. Let us fix $s>0$. We note that
\begin{align*}
    \E_\vartheta[D_B(s, \vartheta, X)]
    = \frac{1}{\operatorname{Tr}(H)}\sum_{i=1}^d H_{ii}\, \E_\vartheta\big[r(s,\vartheta_i)^2\big]
    = \,\E_{\vartheta\sim\cN(0,1)}\big[r(s,\vartheta)^2\big]\, ,
\end{align*}
since the $\vartheta_i$ are identically distributed.
We note that $H$ is positive semidefinite, so $0\le H_{ii}\le\|H\|_{op}$ for every $i\in[d]$, and therefore
\begin{align}\label{eq:bound_on_frob}
 \|\operatorname{diag}(H)\|_F^2=\sum_{i=1}^d H_{ii}^2 \le \|H\|_{op} \operatorname{Tr}(H),
 \qquad
 \|\operatorname{diag}(H)\|_{op}\le\|H\|_{op}.
\end{align}
By Lemma~\ref{lem:everyissub} the coordinates of $r(s,\vartheta)$ are independent, centered and sub-Gaussian with $\|r(s,\vartheta_i)\|_{\psi_2}\le K:=\sqrt{\tfrac{8}{3}}$. Set $\varepsilon_d:=\sqrt{\log^6 d/\reff}$.
Therefore, by Theorem \ref{thm:HanWrig} applied to the matrix $\operatorname{diag}(H)$ with threshold $\operatorname{Tr}(H)\varepsilon_d$, and \eqref{eq:bound_on_frob}, for all $s>0$, we have that

\begin{align}\label{eq:conv}
& \notag\P_\vartheta\left(\big|D_B(s, \vartheta, X)- \E_\vartheta[D_B(s, \vartheta, X)]\big| \ge \varepsilon_d\right) \\
& \notag\quad = \P_\vartheta\left(\Big|r(s,\vartheta)^\top \operatorname{diag}(H)\, r(s,\vartheta) - \E_\vartheta\left[r(s,\vartheta)^\top \operatorname{diag}(H)\, r(s,\vartheta)\right]\Big| \ge \operatorname{Tr}(H)\varepsilon_d\right) \\
& \quad\le2\exp\left(
-c\min\left\{\frac{\operatorname{Tr}(H)^2\varepsilon_d^2}{ K^4\,\|H\|_{op}\operatorname{Tr}(H)},\
\frac{\operatorname{Tr}(H)\,\varepsilon_d}{ K^2\,\|H\|_{op}} \right\}
\right)\\
& \notag\quad=2\exp\left(
-c\min\left\{\tfrac{9}{64}\,\reff\,\varepsilon_d^2,\ \tfrac{3}{8}\,\reff\,\varepsilon_d \right\}
\right)\\
&\notag\quad\le2\exp\left(
-c'\min\left\{\log^6 d,\
\sqrt{\reff}\, \log^3 d \right\}
\right)\, ,
\end{align}
where $c'=3c/64$.
Moreover, by Lemma \ref{lem:everyissub}, the function $s\mapsto r(s,\vartheta)^2$ is Lipschitz continuous, that is
\begin{align*}
|r(s,\vartheta)^2 -r(t,\vartheta)^2| \le 2M_B|\vartheta| |t - s| \, .
\end{align*}
Therefore, using $H_{ii}\ge0$, $\sum_iH_{ii}=\operatorname{Tr}(H)$ and $\E_{\theta\sim\cN(0,1)}[|\vartheta|]=\sqrt{2/\pi}$, we obtain that
\begin{equation}\label{eq:Lipstuff}
    \begin{split}
        & \left|D_B(s, \vartheta, X) - D_B(t, \vartheta, X)\right| \le |s-t| \ \frac{2 M_B}{\operatorname{Tr}(H)}\sum_{i=1}^d H_{ii} |\vartheta_i| \le 2M_B\,\,\|\vartheta\|_\infty\,|s-t| \, ,\\
    & \left|\E_\vartheta[D_B(s, \vartheta, X)] - \E_\vartheta[D_B(t, \vartheta, X)] \right| \le |s-t| \ 2M_B\,\E|\vartheta_1| \le 2M_B\,|s-t| \, .
    \end{split}
\end{equation}
Let us denote, for simplicity, $\beta := \log d$. Let us consider the partition of $(0, 2\beta]$ constructed via the grid $s_k := 2k \frac{\beta}{d}$ for $k=0,\dots,d$, whose mesh is $\eta := s_{k+1}-s_k = \frac{2\beta}{d}$. Let us now fix $k\in \{0, \dots, d-1\}$ and $s\in [s_k, s_{k+1}]$. Then, using \eqref{eq:Lipstuff}, it holds
\begin{align*}
& \left|D_B(s, \vartheta, X)- \E_\vartheta[D_B(s, \vartheta, X)]\right| \\
& \quad \le \left|D_B(s, \vartheta, X) - D_B(s_{k+1}, \vartheta, X)\right|+ \left|D_B(s_{k+1}, \vartheta, X) - \E_\vartheta[D_B(s_{k+1}, \vartheta, X)]\right| \\
& \qquad+ \left|\E_\vartheta[D_B(s_{k+1}, \vartheta, X)] - \E_\vartheta[D_B(s, \vartheta, X)]\right| \\
& \quad\le   2\eta\, M_B \|\vartheta\|_\infty + 2\eta\, M_B  + \left|D_B(s_{k+1}, \vartheta, X) - \E_\vartheta[D_B(s_{k+1}, \vartheta, X)]\right|\\
& \quad=  4 M_B\,\frac{\beta}{d}\,\|\vartheta\|_\infty + 4M_B\,\frac{\beta}{d} + \left|D_B(s_{k+1}, \vartheta, X) - \E_\vartheta[D_B(s_{k+1}, \vartheta, X)]\right|\, .
\end{align*}
Since, by Lemma \ref{lem:norm_of_gaussian_vectors}, $\|\vartheta\|_{\infty} \le \beta$ with probability at least $1-2d\,e^{-\log^2d/2}$, and $\beta/d=\log d/d\le\beta^2$ for $d\ge3$, on the event $\{\|\vartheta\|_\infty\le\beta\}$ it holds
\begin{equation}\label{eq:discr_terms}
    \begin{split}
        &\sup_{s\in(0,2\beta]}\left|D_B(s, \vartheta, X)- \E_\vartheta[D_B(s, \vartheta, X)]\right|\\
    &\quad\le 8 M_B\,\frac{\log^2 d}{d} + \max_{k\in[d]}\left|D_B(s_{k}, \vartheta, X) - \E_\vartheta[D_B(s_{k}, \vartheta, X)]\right| .
    \end{split}
\end{equation}

It remains to cover $s>2\beta$. Since $\Pi_{s,M_B}(t)=0$ whenever $s>2|t|$, on $\{\|\vartheta\|_\infty<\beta\}$ we have $r(s,\vartheta_k)=\vartheta_k=r(2\beta,\vartheta_k)$ for every $k\in[d]$ and every $s>2\beta$, hence $D_B(s,\vartheta,X)=D_B(s_d,\vartheta,X)$ for all $s>2\beta$, where $s_d=2\beta$. Moreover, by Lemma~\ref{lem:everyissub}, $0\le r(s,\vartheta_k)^2\le\vartheta_k^2$ for all $s>0$, and $r(s,\vartheta_k)=\vartheta_k=r(2\beta,\vartheta_k)$ when $|\vartheta_k|<\beta$ and $s>2\beta$, so that $|r(s,\vartheta_k)^2-r(s_d,\vartheta_k)^2|\le\vartheta_k^2\mathbf 1_{\{|\vartheta_k|\ge\beta\}}$. Therefore
\begin{equation}\label{eq:outsiiide}
    \begin{split}
\sup_{s>2\beta}\big|\E_\vartheta[D_B(s,\vartheta,X)]-\E_\vartheta[D_B(s_d,\vartheta,X)]\big|
& \le \sum_{k=1}^d\E\big[\vartheta_k^2\mathbf 1_{\{|\vartheta_k|\ge\beta\}}\big]
 = d\,\E_{G\sim\cN(0,1)}\big[G^2\mathbf 1_{\{|G|\ge\log d\}}\big]\\
& \le 4d\log d\, e^{-\log^2 d/2}
 \le 4\,\frac{\log^2 d}{d}
    \end{split}
\end{equation}
for $d\ge 27$, where we used $\E[G^2\mathbf 1_{\{|G|\ge a\}}]\le 4a\,e^{-a^2/2}$ for $a\ge1$ with $a=\log d$, and in the last step that $e^{2u-u^2/2}\le u$ for $u=\log d\ge 4$. Hence, on $\{\|\vartheta\|_\infty<\beta\}$,
\begin{align*}
 & \sup_{s>0}\left|D_B(s, \vartheta, X)- \E_\vartheta[D_B(s, \vartheta, X)]\right| \\
 &\quad = \max\Big\{ \sup_{s\in(0,2\beta]}\left|D_B(s, \vartheta, X) - \E_\vartheta[D_B(s, \vartheta, X)]\right| ,\ \sup_{s >2\beta}\left|D_B(s, \vartheta, X) - \E_\vartheta[D_B(s, \vartheta, X)]\right| \Big\}\\
 &\quad \le  \max\Big\{ \sup_{s\in(0,2\beta]}\left|D_B(s, \vartheta, X) - \E_\vartheta[D_B(s, \vartheta, X)]\right| ,\ 
 \left|D_B(s_d, \vartheta, X) -\E_\vartheta[D_B(s_d, \vartheta, X)] \right| \\
 & \qquad+\sup_{s >2\beta} \left| \E_\vartheta[D_B(s_d, \vartheta, X)]-\E_\vartheta[D_B(s, \vartheta, X)]\right| \Big\}\\
 & \quad\overset{(*)}{\le} 12 M_B\,\frac{\log^2 d}{d} + \max_{k\in[d]}\left|D_B(s_{k}, \vartheta, X) - \E_\vartheta[D_B(s_{k}, \vartheta, X)]\right| ,
\end{align*}
where in the second step we used that $D_B(s,\vartheta,X)=D_B(s_d,\vartheta,X)$ for $s>2\beta$ on $\{\|\vartheta\|_\infty<\beta\}$, and in $(*)$ we used \eqref{eq:discr_terms} for the first term, \eqref{eq:outsiiide} for the last term, $M_B\ge2$, and the fact that $s_d$ is one of the grid points, so that $|D_B(s_d,\vartheta,X)-\E_\vartheta[D_B(s_d,\vartheta,X)]|\le\max_{k\in[d]}|D_B(s_k,\vartheta,X)-\E_\vartheta[D_B(s_k,\vartheta,X)]|$.
Combining \eqref{eq:conv} and \eqref{eq:discr_terms}, we can conclude that
\begin{equation}\label{eq:bound_1}
    \begin{split}
    &\P_\vartheta\left( \sup_{s>0} \ \left|D_B(s, \vartheta, X)- \E_\vartheta[D_B(s, \vartheta, X)]\right| > \varepsilon_d + 12 M_B\,\frac{\log^2 d}{d}\right)
    \\
    &\quad\le \P_\vartheta\big(\|\vartheta\|_\infty > \beta\big) + \sum_{k=1}^{d}\P_\vartheta\left(  \left|D_B(s_{k}, \vartheta, X) - \E_\vartheta[D_B(s_{k}, \vartheta, X)]\right| > \varepsilon_d\right) \\
    &\quad\le  2d\,e^{-\log^2 d/2} + 2d\exp\left(-c'\min\left\{\log^6 d,\ \sqrt{\reff}\,\log^3 d \right\}\right)\\
    &\quad \le 4d\exp\left(-c''\min\left\{\log^2 d,\ \sqrt{\reff}\,\log^3 d \right\}\right)\, ,
    \end{split}
\end{equation}
which concludes the bound on $D_B(s, \vartheta, X)$.

It remains now to bound $O_B(s, \vartheta, X)$. Let $\xi\in\{-1,1\}^d$ have i.i.d.\ $\operatorname{Rad}(1/2)$ coordinates, independent of $\vartheta$. Since the coordinates of $\vartheta$ are independent and symmetric, $\xi\odot\vartheta\overset{d}{=}\vartheta$ as random vectors; moreover $r(s,\xi_i\vartheta_i)=\xi_i\,r(s,\vartheta_i)$ by item~1 of Lemma~\ref{lem:everyissub}. Writing $\widetilde H := H - \operatorname{diag}(H)$, it holds, for every $s>0$,

\begin{align*}
    \sum_{i=1}^d\sum_{j\ne i}^d H_{ij}\, r(s,\xi_i\vartheta_i)\,r(s,\xi_j \vartheta_j)
   & = \sum_{i,j=1}^d \xi_i\, r(s,\vartheta_i )\,\widetilde H_{ij}\, r(s,\vartheta_j)\, \xi_j \\
    &= \xi^\top \operatorname{diag}(r(s,\vartheta))\, \widetilde H\, \operatorname{diag}(r(s,\vartheta))\, \xi\, ,
\end{align*}
and consequently, as processes indexed by $s>0$,
\[
\big(O_B(s,\vartheta,X)\big)_{s>0}\ \overset{d}{=}\ \Big(\frac{d}{\operatorname{Tr}(H)}\,\xi^\top \operatorname{diag}(r(s,\vartheta))\, \widetilde H\, \operatorname{diag}(r(s,\vartheta))\, \xi\Big)_{s>0},
\]
so that it suffices to bound the supremum of the right-hand side.
Since $\widetilde H$ is equal to $H$ on the off-diagonal terms and equal to zero on the diagonal terms, and $H$ is positive semidefinite with diagonal entries in $[0,\|H\|_{op}]$, it follows that $\|\widetilde H\|_F\le\|H\|_F$ and, by Weyl's inequality, $\|\widetilde H\|_{op}\le\|H\|_{op}$.
For fixed $\vartheta\in \R^d$, there are at most $M_Bd$ positive breakpoints,
\[
\sigma_{i,k}=\frac{|\vartheta_i|}{k+1/2},\qquad 1\le i\le d,\quad 0\le k\le M_B-1 \, ,
\]
which we sort increasingly as $0<\tau_1<\dots<\tau_K$, $K\le M_Bd$, and we set $\tau_0:=0$.
Let us define
\[
p(s) := \sum_{i,j=1}^d \xi_i\, r(s,\vartheta_i )\,\widetilde H_{ij}\, r(s,\vartheta_j)\, \xi_j \qquad s>0 .
\]
For each $s\in (\tau_k, \tau_{k+1})$ the residual $r(s,\vartheta_i)$ is affine, and therefore
\[
p(s)=p_k(s):=  \sum_{i,j=1}^d \widetilde H_{ij}\,\xi_i \xi_j\, (\vartheta_i - s\, n_i)(\vartheta_j - s\, n_j) \, ,
\]
where $n_i := \roundB{\vartheta_i/s}$, which is constant for $s\in(\tau_k, \tau_{k+1})$. Thus $p_k$ is a polynomial of degree at most $2$, defined on all of $\R$, and it coincides with $p$ on the open interval $(\tau_k,\tau_{k+1})$. Therefore, it holds
\begin{align*}
    \sup_{s>0} \, |p(s)| = \max\Big\{\max_{k=0,\dots,K-1}\ \sup_{s\in (\tau_k, \tau_{k+1})} |p_k(s)|,\ \ \max_{k=1,\dots,K}|p(\tau_k)|,\ \ \sup_{s>\tau_K}|p(s)|\Big\}.
\end{align*}
For $s>\tau_K=2\|\vartheta\|_\infty$ every rounding vanishes, so $r(s,\vartheta)=\vartheta$ and $p(s)=\sum_{i,j}\widetilde H_{ij}\xi_i\xi_j\vartheta_i\vartheta_j$ is constant.
By Lemma \ref{lem:threepointLagr}, applied to $p_k$ on the closed interval $[\tau_k,\tau_{k+1}]$, we have

\begin{align*}
    \sup_{s\in (\tau_k, \tau_{k+1})} \, |p_k(s)| \le 3\max\big\{|p_k(\tau_k)|,\,|p_k(\tau_{k+1})|,\,|p_k((\tau_k + \tau_{k+1})/2)|\big\}.
\end{align*}
Therefore, we have that
\begin{align*}
    &\sup_{s>0} \, |p(s)| \\
    & \quad\le 3\max\Big\{\max_{k<K}\max\big\{|p_k(\tau_k)|,|p_k(\tau_{k+1})|,|p_k(\tfrac{\tau_k+\tau_{k+1}}2)|\big\},\ \max_{k\le K}|p(\tau_k)|,\ \Big|\sum_{i,j=1}^d \widetilde H_{ij}\xi_i \xi_j \vartheta_i \vartheta_j\Big| \Big\}\, .
\end{align*}
Each quantity inside the maximum is of the form $\xi^\top C\xi$ with $C=\operatorname{diag}(b)\,\widetilde H\operatorname{diag}(b)$, where $b\in\R^d$ collects either the residuals $r(s_0,\vartheta_i)$ at a node $s_0$, or their one-sided limits at a breakpoint, or $\vartheta_i$ itself. Since $0$ is a quantization level, $|r(s,\vartheta_i)|\le|\vartheta_i|$, and at a breakpoint the one-sided limits equal $\pm\sigma_{i,k}/2$ with $\sigma_{i,k}\le2|\vartheta_i|$; hence $|b_i|\le|\vartheta_i|$ in all cases.
Thus there are at most $4M_Bd+1$
matrices $C_j \in \R^{d\times d}$ of the form $C_j=\operatorname{diag}(b^{(j)})\,\widetilde H\operatorname{diag}(b^{(j)})$ with $|b^{(j)}_i|\le|\vartheta_i|$, such that
\begin{equation}
\sup_{s>0}\left|\xi^\top \ \operatorname{diag}(r(s,\vartheta))\, \widetilde H\, \operatorname{diag}(r(s,\vartheta)) \ \xi \right|\le 3\max_{j\in[4M_Bd+1]}|\xi^\top C_j\xi| \, .
\label{eq:er-offdiag-finite}
\end{equation}
This includes rounding ties. Since $\widetilde H$ has zero diagonal, so does every
$C_j$, and hence
\[
\E_\xi[\xi^\top C_j\xi]=\operatorname{Tr} C_j=0.
\]
On the event $\{\|\vartheta\|_\infty\le\beta\}$, which has probability at least $1-2d\,e^{-\log^2 d/2}$, using $|b^{(j)}_i|\le\|\vartheta\|_\infty\le\beta$, the submultiplicativity of the operator norm and $\|H\|_F^2\le\|H\|_{op}\operatorname{Tr}(H)$, it holds that
\begin{align*}
    & \|C_j\|_{op} \le \|\vartheta\|^2_\infty\, \|\widetilde H\|_{op} \le \beta^2\, \|H\|_{op} = \log^2 d \|H\|_{op} \, ,\\
    & \|C_j\|^2_{F}  \le \|\vartheta\|^4_{\infty} \|\widetilde H\|^2_F \le \beta^4 \|H\|_{op} \operatorname{Tr}(H) = \log^4 d \|H\|_{op} \operatorname{Tr}(H)\, .
\end{align*}
We note that the $\xi_i$ are independent, centered, with $\|\xi_i\|_{\psi_2}$ an absolute constant. Hence, on $\{\|\vartheta\|_\infty\le\beta\}$, by Theorem~\ref{thm:HanWrig} applied
conditionally on $\vartheta$, \eqref{eq:er-offdiag-finite} and a union bound, it holds
\begin{equation}\label{eq:bound2}
    \begin{split}
        &\mathbb P_\xi\Big(\sup_{s>0} \ \frac{1}{\operatorname{Tr}(H)}
 \big|\xi^\top\operatorname{diag}(r(s,\vartheta))\,\widetilde H\,
 \operatorname{diag}(r(s,\vartheta))\,\xi\big|>\varepsilon_d\Big)\\
 & \quad\le \sum_{j=1}^{4M_Bd+1}\mathbb P_\xi\Big(|\xi^\top C_j\xi|>\frac{\operatorname{Tr}(H)}{3}\varepsilon_d\Big)\\
 &\quad\le 2(4M_Bd+1)\exp\Big(-c\min\Big\{
 \frac{\operatorname{Tr}(H)^2\varepsilon_d^2/9}{\log^4 d\,\|H\|_{op}\operatorname{Tr}(H)},\
 \frac{\operatorname{Tr}(H)\varepsilon_d/3}{\log^2 d\,\|H\|_{op}}\Big\}\Big) \\
 &\quad= 2(4M_Bd+1)\exp\Big(-c\min\Big\{
 \frac{\reff\,\varepsilon_d^2}{9\log^4 d},\ \frac{\reff\,\varepsilon_d}{3\log^2 d}\Big\}\Big)\\
 & \quad= 2(4M_Bd+1)\exp\Big(-c\min\Big\{
 \frac{\log^2 d}{9},\ \frac{\log d \sqrt{\reff}}{3}\Big\}\Big) \, .
    \end{split}
\end{equation}
Integrating over $\vartheta$ and using the equality in law established above,
\begin{align*}
&\P\Big(\sup_{s>0}|O_B(s,\vartheta,X)|>\varepsilon_d\Big)\\
& \quad\le \P_\vartheta\big(\|\vartheta\|_\infty>\beta\big)+2(4M_Bd+1)\exp\Big(-\frac{c}{9}\min\big\{\log^2 d,\ \log d \sqrt{\reff}\big\}\Big).    
\end{align*}

Combining \eqref{eq:rec_splitted}, \eqref{eq:bound_1} and \eqref{eq:bound2}, we conclude that
\begin{align*}
    &\P\left( \sup_{s>0} \ \left|\cE(s) - \cE^{\infty}(s)\right| > 2\sqrt{\frac{\log^6 d}{\reff(H)}} +12 M_B\,\frac{\log^2 d}{d}\right)\\
    &\quad \le  C\,M_B\,d\, \exp\left\{-c\min\left\{\log^2 d,\ \log d \sqrt{\reff(H)}\right\}\right\}
\end{align*}
for absolute constants $c,C>0$. In particular, if $\reff(H) \ge \log^8 d$, then $\min\{\log^2 d,\log d\sqrt{\reff}\}=\log^2 d$ and $\sqrt{\log^6 d/\reff}\le 1/\log d$, so that

\begin{align*}
    \P\left( \sup_{s>0} \ \left|\cE(s) - \cE^{\infty}(s)\right| > \frac{2+12M_B}{\log d}\right)
    \le  C\,M_B\,d\,e^{-c \log^2 d}\, ,
\end{align*}
where $c,C>0$ are universal constants independent of $d$.
\end{proof}

\subsection{Proof of Corollary \ref{cor:minimizers}}\label{app:minimizers}
 Corollary \ref{cor:minimizers} follows directly from the uniform approximation in Theorem~\ref{thm:uniform}, which controls the difference between the infima of $\cE$ and $\cE^\infty$ and ensures that any approximate
minimizer of $\cE$ also approximately minimizes $\cE^\infty$.
\begin{corollary}\label{cor:minimizers}
Under the assumptions of Theorem~\ref{thm:uniform}, set
$\varepsilon_d:=\frac{2+12M_B}{\log d}$. Then, with probability at least
$1-C\,M_B\,d\,e^{-c\log^2 d}$, the following statements hold simultaneously:
\begin{enumerate}
\item The optimal values satisfy
\[
\,\bigl|m-m^{\infty}\bigr|\ \le\ \varepsilon_d \, .
\]
\item For every $\eta\ge0$ and every $\widehat s>0$ satisfying
$\cE(\widehat s)\le m+\eta$ one has
\[
0\ \le\ \bigl(\cE^{\infty}(\widehat s)-m^{\infty}\bigr)
\ \le\ 2\varepsilon_d+ \eta .
\]
In particular, any exact minimizer $\widehat s$ of $\cE$ satisfies this bound
with $\eta=0$.
\end{enumerate}
\end{corollary}
\begin{proof}
    By Theorem~\ref{thm:uniform}, with probability at least
$1-CM_Bd\,e^{-c\log^2 d}$ it holds
\[
\sup_{s>0}\ \,\bigl|\cE(s)-\cE^{\infty}(s)\bigr|
\ \le\ \varepsilon_d \, .
\]
We work on this event. For every $s>0$, we have
\[
\cE^{\infty}(s)-\varepsilon_d\ \le\ \cE(s)\ \le\ \cE^{\infty}(s)+\varepsilon_d \, .
\]
Taking infima over $s>0$ yields
\[
m^{\infty}-\varepsilon_d\ \le\ m\ \le\ m^{\infty}+\varepsilon_d ,
\]
which proves the first point of the corollary.

Now let $\eta\ge0$ and let $\widehat s>0$ satisfy $\cE(\widehat s)\le m+\eta$.
The uniform error bound and the first assertion give
\[
\begin{aligned}
\cE^{\infty}(\widehat s)
&\le \cE(\widehat s)+\varepsilon_d\\
&\le m+\eta+\varepsilon_d\\
&\le m^{\infty}+\eta+2\varepsilon_d .
\end{aligned}
\]
The observation that $\cE^{\infty}(\widehat s)\ge m^{\infty}$ concludes the proof.
\end{proof}

\subsection{Verifying (H2) for Multi-Layer Perceptrons at initialization}\label{sec:jd}
The uniform approximation of Theorem~\ref{thm:uniform} relies on (H2), namely a lower bound on
the effective rank of $XX^\top$. In this section we show that this condition holds, with high probability, in the setting of a random multi-layer perceptron with isotropic data for channelwise quantization, i.e., for layer $\ell$, we quantize independently the rows $\theta^{(\ell)}_{i:}$ and $X$ is given by the output of the previous layer (not quantized). We begin by defining the considered model.

 \begin{defn}\label{def:smlp}
    Let \(L > 1\) be fixed. A \emph{multi-layer perceptron} (MLP) of depth \(L\) is defined as the parametric function \(f_\vartheta: \R^{d_0} \to \R\) obtained via the recursion
    \begin{equation}\label{eq:model_mlp}
    \begin{split}
     & \alpha^{(0)}(x) := x, \quad
    \alpha^{(\ell)}(x) := \sigma\left(\frac{1}{\sqrt{d_{\ell-1}}}\vartheta^{(\ell)} \alpha^{(\ell-1)}(x) \right) \quad \text{for } \ell = 1, \dots, L-1, \\
        & f_\vartheta(x) := \frac{1}{\sqrt{d_{L-1}}}\vartheta^{(L)\top} \alpha^{(L-1)}(x),
    \end{split}  
    \end{equation}
    Here, $\vartheta^{(\ell)}\in\R^{d_\ell\times d_{\ell-1}}$
for $\ell=1,\dots,L-1$, and
$\vartheta^{(L)}\in\R^{d_{L-1}}$. The nonlinearity $\sigma:\R\to\R$ acts componentwise,
is $L_\sigma$-Lipschitz, and satisfies $\sigma(0)=0$.
    \end{defn}

Rather than assuming (H2) directly, we give explicit
conditions on the network architecture and calibration
distribution under which it holds with high probability.

\paragraph{Assumptions} We assume that:
\begin{itemize}
\item[(H3)] $d_k = \alpha_k d$ for some constant $\alpha_\ell>0$ and $\ell\in[L]$, and we denote with $d := \min_{k=1,\dots,L-1} \, d_k $.
\item[(H4)] the calibration dataset $\mathbb{D}_{cal} = \{X_i'\}_{i=1}^{N_c}$ is the realization of $N_c$ iid copies of $X' \sim \cN(0, \operatorname{Id}_{d_0})$. We denote with with $X' = [X_1' \dots, X_{N_c}']\in \R^{d_0\times N_c}$. We assume that $N_c \ge C \log^8 d$ for some constant $C>0$ not dependent on $d$.
\end{itemize}

Assumption (H3) places the hidden-layer widths in a
proportional regime. Assumption (H4) models the calibration
inputs as isotropic Gaussian and requires the sample size
to grow at least polylogarithmically with the minimum hidden
width. This permits calibration sets that are asymptotically
smaller than the layer widths. We additionally assume that
the activation is odd and Lipschitz, covering choices such
as $\tanh$ but excluding ReLU.

The results of Theorem \ref{thm:uniform} relies on the feature Gram matrix $\alpha^{(\ell)}(X')\,\alpha^{(\ell)}(X')^\top$ having large effective rank. The following theorem shows that this condition is fulfilled, with high probability, for the (iterative) PTQ problem of a MLP with Gaussian weights and Gaussian calibration inputs, as soon as the width is large enough. 

In the rest of this section, to improve readability, we sometimes omit the constants which do not depend on $d$ and $N_c$ form the formulas, since they don't affect the asymptotic behavior. 
We start by controlling, with high probability, the operator norm $\opn{\alpha^{(\ell)}(X')\,\alpha^{(\ell)}(X')^\top}$ from above and the trace $\Tr\big(\alpha^{(\ell)}(X')\,\alpha^{(\ell)}(X')^\top\big)$ from below, which are the two quantities entering the effective rank. 
\begin{lemma}\label{lem:opnorm}
Let us assume (H1), (H3) and (H4) and let $\ell\in[L-1]$. Then, there exists universal constants $C,c>0$ such that
\[
\P_{\theta, X'}\left(\opn{\alpha^{(\ell)}(X')\,\alpha^{(\ell)}(X')^\top}>C(d+N_c)\right)\ \le\ 2\cdot 9^{-c(d+N_c)}+2\ell e^{-d}\, .
\]
\end{lemma}
\begin{proof}
Let us fix $\ell\in[d]$, $u\in \Sph{d_\ell-1}$ and $v\in\Sph{N_c-1}$. Let us fix $\theta^{(k)}\in \R^{d_k\times d_{k -1}}$ for $k\in[\ell]$ and define $ F_{u,v}: \R^{d_0\times N_{c}} \to \R$ as follows
\begin{align*}
    F_{u,v}(X):= u^\top \sigma\left(\frac{1}{\sqrt{d_{\ell-1}}}\theta^{(\ell)} \alpha^{(\ell -1)}(X)\right)v \, ,
\end{align*}
where $\alpha^{(\ell -1)}(X) \in \R^{d_{\ell-1}\times N_c}$ is the output of the $(\ell-1)$-th layer evaluated on the calibration dataset (with the convention $\alpha^{(0)}(X)=X$).
We note that $F_{u,v}$ is Lipschitz. Indeed, since $\|u\|_2=\|v\|_2=1$ and $\sigma$ is $L_\sigma$-Lipschitz and applied entrywise, for all $X,Y\in\R^{d_0\times N_c}$ it holds
\begin{align*}
  \left|F_{u,v}(X) - F_{u,v}(Y)\right| &\le  \left\|\sigma\left(\frac{1}{\sqrt{d_{\ell-1}}}\theta^{(\ell)} \alpha^{(\ell -1)}(X)\right) - \sigma\left(\frac{1}{\sqrt{d_{\ell-1}}}\theta^{(\ell)} \alpha^{(\ell -1)}(Y)\right)\right\|_{F} \\
  & \le L_\sigma \left\|\frac{1}{\sqrt{d_{\ell-1}}}\theta^{(\ell)}\right\|_{op} \|\alpha^{(\ell -1)}(X) - \alpha^{(\ell -1)}(Y)\|_{F} \\
  & \le L_\sigma^\ell \prod_{k=1}^{\ell} \left\|\frac{1}{\sqrt{d_{k-1}}} \ \, \theta^{(k)}\right\|_{op} \|X - Y\|_F =: \Lambda(\theta)\,\|X-Y\|_F \, ,
\end{align*}
where we used $\|MN\|_F\le\|M\|_{op}\|N\|_F$ and, in the last step, we iterated the previous inequality over the layers $\ell-1,\dots,1$.
Moreover, since the activation $\sigma$ is odd, it follows that $F_{u,v}(-X) = - F_{u,v}(X)$. Now, since by (H4), $X'_{ij}\overset{iid}{\sim} \cN(0,1)$ which is symmetric, it holds that 
\begin{align*}
    \E_{X'}[F_{u,v}(X')] = 0 \, .
\end{align*}
Therefore, using Lemma~\ref{prop:concLip}, for all $t>0$, it holds
\begin{equation}\label{eq:fixedpair}
\P_{X'}\big(|F_{u,v}(X')|>t\big)\ \le\ 2\exp\left\{-\frac{t^2}{2 \Lambda(\theta)^2}\right\}.
\end{equation}
We now control $\Lambda(\theta)$. Since the entries of $\theta^{(k)}$ are i.i.d.\ $\cN(0,1)$, we have $K=\max_{i,j}\|\theta^{(k)}_{ij}\|_{\psi_2}=\sqrt{8/3}$, and Theorem~\ref{thm:subgaussian-norm} with $t=\sqrt{d}$ gives, for every $k\in[\ell]$ (recall $d_k = \alpha d$),
\[
\P_\theta\left(\|\theta^{(k)}\|_{op} > 3CK\sqrt{d}\right)\ \le\ \P_\theta\left(\|\theta^{(k)}\|_{op} > CK\big(\sqrt{d_k}+\sqrt{d_{k-1}}+\sqrt d\big)\right)\ \le\ 2e^{-d}\, .
\]
Hence, by a union bound, the event
\[
\cE:=\left\{\max_{k\in[\ell]}\ \|\theta^{(k)}\|_{op}\le 3CK\sqrt{d}\right\}
\]
satisfies $\P_\theta(\cE^c)\le 2\ell e^{-d}$, and on $\cE$, under hypothesis (H3), it holds $\Lambda(\theta)\le\Lambda:=\left(3CKL_\sigma\right)^{\ell}\prod_{k=1}^{\ell}\sqrt{\frac{1}{\alpha_{k-1}}}$.
In particular, by \eqref{eq:fixedpair}, for every $\theta\in\cE$ and every $t>0$,
\begin{equation}\label{eq:fixedpair-E}
\P_{X'}\big(|F_{u,v}(X')|>t\big)\ \le\ 2\exp\left\{-\frac{t^2}{2 \Lambda^2}\right\}.
\end{equation}
Let $\cU\subset\Sph{d_{\ell}-1}$, $\cV\subset\Sph{N_c-1}$ be the $\tfrac{1}{4}$-nets of Lemma~\ref{lem:net}, so that $|\cU|\le 9^{d_\ell}$ and $|\cV|\le 9^{N_c}$. By Lemma~\ref{lem:net}, $$\opn{\sigma\left(\frac{1}{\sqrt{d_{\ell-1}}}\theta^{(\ell)} \alpha^{(\ell -1)}(X')\right)}\le 2\max_{(u,v)\in\cU\times\cV}|F_{u,v}(X')|\, ,$$ so, for every $t>0$, we have
\[
\left\{\opn{\sigma\left(\frac{1}{\sqrt{d_{\ell-1}}}\theta^{(\ell)} \alpha^{(\ell -1)}(X')\right)}>2t\right\}\ \subseteq\ \bigcup_{(u,v)\in\cU\times\cV}\{|F_{u,v}(X')|>t\},
\]
and by \eqref{eq:fixedpair-E} and a union bound, for every $\theta\in\cE$,
\begin{align*}
    \P_{X}\left(\opn{\sigma\left(\frac{1}{\sqrt{d_{\ell-1}}}\theta^{(\ell)} \alpha^{(\ell -1)}(X')\right)}>2t\right)\ &\le\ |\cU|\,|\cV|\cdot 2\exp\left\{-\frac{t^2}{2 \Lambda^2}\right\}\\
& \le\ 2\cdot 9^{\,d_\ell+N_c}\exp\left\{-\frac{t^2}{2 \Lambda^2}\right\}\, .
\end{align*}
Let us choose $t^2:=4\log 9\;(d_\ell+N_c)\,\Lambda^2$. Then $\dfrac{t^2}{2\Lambda^2}=2(d_\ell+N_c)\log 9$, so the right-hand side equals $2\cdot 9^{d_\ell+N_c}\cdot 9^{-2(d_\ell+N_c)}=2\cdot 9^{-(d_\ell+N_c)}$. Finally, since $X'$ and $\theta$ are independent,
\begin{align*}
& \P_{\theta,X'}\left(\opn{\sigma\left(\frac{1}{\sqrt{d_{\ell-1}}}\theta^{(\ell)} \alpha^{(\ell -1)}(X')\right)}>2t\right)\\
&\quad\le \E_\theta\left[\mathbf{1}_{\cE}\,\P_{X'}\left(\opn{\sigma\left(\frac{1}{\sqrt{d_{\ell-1}}}\theta^{(\ell)} \alpha^{(\ell -1)}(X')\right)}>2t\right)\right]+\P_\theta(\cE^c)\\
&\quad\le 2\cdot 9^{-(d_\ell+N_c)}+2\ell e^{-d}\, .
\end{align*}
Finally, since $X$ and $\theta$ are independent,
\begin{align*}
& \P_{\theta,X'}\left(\opn{\sigma\left(\frac{1}{\sqrt{d_{\ell-1}}}\theta^{(\ell)} \alpha^{(\ell -1)}(X')\right)}>2t\right)\\
&\quad\le \E_\theta\left[\mathbf{1}_{\cE}\,\P_{X'}\left(\opn{\sigma\left(\frac{1}{\sqrt{d_{\ell-1}}}\theta^{(\ell)} \alpha^{(\ell -1)}(X')\right)}>2t\right)\right]+\P_\theta(\cE^c)\\
&\quad\le 2\cdot 9^{-(d_\ell+N_c)}+2\ell e^{-d}\, {\color{blue}.}
\end{align*}
We note that $\|\alpha^{(\ell)}(X')\,\alpha^{(\ell)}(X')^\top\|_{op} = \|\alpha^{(\ell)}(X)\|^2_{op}$. Therefore, under (H3), there exists universal constants $C,c>0$ such that  
\begin{align*}
\P_{\theta,X'}\left(\opn{\alpha^{(\ell)}(X')\,\alpha^{(\ell)}(X')^\top}>C\,(d+N_c)\right)
&\le \P_{\theta,X'}\left(\opn{\alpha^{(\ell)}(X')}>2t\right)\\
&\le 2\cdot 9^{-(d_\ell+N_c)}+2\ell e^{-d}\\
& \le\ 2\cdot 9^{-c(d+N_c)}+2\ell e^{-d}\, ,
\end{align*}
which concludes the proof.
\end{proof}

The next lemma shows that, with high probability, a lower bound on the trace $\Tr\left(\alpha^{(\ell)}(X')\,\alpha^{(\ell)}(X')^\top\right)$ for all $\ell\in[L-1]$.
\begin{lemma}\label{lem:trace-lower}
Let us assume (H1), (H3) and (H4). Then there exist universal constants $a,C,c>0$ not depending on $d$, such that for every $\ell\in\{0,\dots,L-1\}$, it holds
\[
\P\left(\Tr\left(\alpha^{(\ell)}(X')\,\alpha^{(\ell)}(X')^\top\right)\ <\ a\, d N_c\right)\ \le\ C N_c\, e^{-c d}\, .
\]
\end{lemma}

\begin{proof}
Let $\ell\in\{0,\dots,L-1\}$. We note that
\begin{align*}
  \Tr\left(\alpha^{(\ell)}(X')\alpha^{(\ell)}(X')^\top\right) = \|\alpha^{(\ell)}(X')\|_F^2=\sum_{j=1}^{N_c}\|\alpha^{(\ell)}_{:j}(X')\|_2^2=\sum_{j=1}^{N_c}\|\alpha^{(\ell)}(X'_{:j})\|_2^2 \, ,
\end{align*}
so that a lower bound on the trace follows from a lower bound on $\|\alpha^{(\ell)}(X'_{:j})\|_2^2$ holding simultaneously for all $j\in[N_c]$. For $\ell\ge1$, $i\in[d_\ell]$ and $j\in[N_c]$, let us define the random variables
\begin{align*}
   & Z_i^j := \frac{1}{\sqrt{d_{\ell-1}}}\, \theta^{(\ell)}_{i:}\alpha^{(\ell-1)}(X'_{:j}) \, ,\\
   \, 
   &Y_i^j := \sigma\big(Z_i^j\big)^2 \, . 
\end{align*}
We note that $\|\alpha^{(\ell)}(X'_{:j})\|_2^2=\sum_{i=1}^{d_\ell}Y_i^j$.
Let $\cF_{\ell-1}$ denote the $\sigma$-algebra generated by $X',\theta^{(1)},\dots,\theta^{(\ell-1)}$. Since $\theta^{(\ell)}$ is independent of $\cF_{\ell-1}$, the random variables $Y^j_1, \dots, Y^j_{d_\ell}$ are i.i.d., and, using (H3), we have
\begin{align*}
    0\le Y_i^j \le L_\sigma^2 \big(Z_i^j\big)^2 \, .
\end{align*}
Moreover, conditioned on $\cF_{\ell-1}$, we have that
$Z_i^j \sim \cN\left(0, \frac{\|\alpha^{(\ell-1)}(X'_{:j})\|_2^2}{d_{\ell-1}}\right)$. Let us define following iterative constants
\begin{align*}
    & a_0:=\frac12, \quad b_0:=2 \\
    & m_\ell = \inf_{h \in [a_{\ell-1}, b_{\ell-1}]} \, \E_{Z\sim \mathcal{N}(0, h)} \big[\sigma^2(Z)\big]\ , \\
    & b_\ell:=2L_\sigma^2\, b_{\ell-1}\, ,\\
    & a_\ell:=\tfrac12\, m_\ell\, .
\end{align*}
and the event
\begin{align*}
    E_\ell := \bigcap_{j=1}^{N_c} \left\{ a_\ell d_\ell \le \|\alpha^{(\ell)}(X'_{:j})\|_2^2 \le b_\ell d_\ell \right\} \, .
\end{align*}

We estimate now the probability of $E_{\ell}$. We proceed by induction over $\ell$. The base case $\ell=0$ follows from Lemma~\ref{lem:norm_of_gaussian_vectors}. Indeed, since $\alpha^{(0)}(X'_{:j})=X'_{:j}\sim\cN(0,\operatorname{Id}_{d_0})$ by (H4), for each $j\in [N_c]$ each of the two inequalities
\begin{align*}
    \frac{1}{2} d_0 \le \|\alpha^{(0)}(X'_{:j})\|_2^2 = \|X'_{:j}\|_2^2 \le 2d_0
\end{align*}
fails with probability at most $e^{-d_0/16}$. Therefore, by a union bound,
\begin{align*}
    \P \left( E_0^{c} \right) \le \sum_{j=1}^{N_c} \left[\P \left(\|X'_{:j}\|_2^2 < \frac{1}{2} d_0 \right) + \P \left(\|X'_{:j}\|_2^2 >  2d_0 \right)\right] \le 2N_c e^{-d_0/16} \, .
\end{align*}
Let now $\ell\ge1$. As inductive hypothesis, we assume that there exists $c_1,\dots,c_{\ell-1}>0$, depending only on the activation $\sigma$ and on the layer index, such that the event
\begin{align*}
    E_{\ell-1} = \bigcap_{j=1}^{N_c} \left\{ a_{\ell-1} d_{\ell-1} \le \|\alpha^{(\ell-1)}(X'_{:j})\|_2^2 \le b_{\ell-1} d_{\ell-1} \right\}
\end{align*}
satisfies
\begin{align*}
    \P\left(E_{\ell-1}^C\right)\ \le\ 2N_c\, e^{-d_0/16}+N_c\sum_{k=1}^{\ell-1}\left(2e^{-c_k d_k}+e^{-d_k/16}\right) \, .
\end{align*}
Then, we have that
\begin{equation}\label{eq:altogheter}
\begin{split}
    \P \left( E_\ell^{C} \right) &\le \P \left( E_\ell^{C} \ \Big| \ E_{\ell-1} \right) + \P\left( E^{C}_{\ell-1}\right)\\
 & \le \sum_{j=1}^{N_c} \left[\P \left(\|\alpha^{(\ell)}(X'_{:j})\|_2^2 < a_\ell d_\ell \ \Big| \ E_{\ell-1} \right) + \P \left(\|\alpha^{(\ell)}(X'_{:j})\|_2^2 > b_\ell d_\ell \ \Big| \ E_{\ell-1} \right)\right] \\
 & \qquad+ \P\left( E^{C}_{\ell-1}\right) \, .
\end{split}    
\end{equation}

Since $E_{\ell-1}\in\cF_{\ell-1}$, for every event $A$ we have $\P(A\,|\,E_{\ell-1})=\E\big[\mathbf 1_{E_{\ell-1}}\P(A\,|\,\cF_{\ell-1})\big]/\P(E_{\ell-1})$, so it suffices to bound $\P(A\,|\,\cF_{\ell-1})$ uniformly on the event $E_{\ell-1}$.

By Lemma~\ref{lem:norm_of_gaussian_vectors}, we notice that on $E_{\ell-1}$, it holds
\begin{equation}\label{eq:upppppp}
    \begin{split}
        \P \left(\|\alpha^{(\ell)}(X'_{:j})\|_2^2 > b_\ell d_\ell \ \Big| \ \cF_{\ell-1} \right)
  &\le \P \left(L_\sigma^2\sum_{i=1}^{d_\ell}\big(Z_i^j\big)^2 > 2L_\sigma^2 b_{\ell-1} d_\ell \ \Big| \ \cF_{\ell-1} \right) \\
  & \le \P \left( \sum_{i=1}^{d_\ell} \sqrt{d_{\ell-1}}\Big(\frac{Z_i^j}{\|\alpha^{(\ell-1)}(X'_{:j})\|_2}\Big)^2 > 2 d_\ell\ \Big| \ \cF_{\ell-1} \right)\\
  & \le e^{-d_\ell/16} \, ,
    \end{split}
\end{equation}
where we used $Y_i^j\le L_\sigma^2(Z_i^j)^2$ and $b_\ell=2L_\sigma^2 b_{\ell-1}$ in the first step, $s_j^2\le b_{\ell-1}$ in the second step, and the fact that, conditioned on $\cF_{\ell-1}$, $d_{\ell-1}\frac{Z_1^j}{\|\alpha^{(\ell-1)}(X'_{:j})\|_2}, \dots, d_{\ell-1}\frac{Z_{d_{\ell}}^j}{\|\alpha^{(\ell-1)}(X'_{:j})\|_2}$ are iid $\cN(0,1)$ in the last step.

Let us define
\begin{equation}\label{eq:m_ell}
m_\ell := \inf_{h \in [a_{\ell-1}, b_{\ell-1}]} \, \E_{Z\sim \mathcal{N}(0, h)} \big[\sigma^2(Z)\big] \, .
\end{equation}
We note that $m_\ell>0$. Indeed, the map $h\mapsto\E_{Z\sim\cN(0,h)}[\sigma^2(Z)]=\E_{g\sim\cN(0,1)}[\sigma^2(\sqrt h\, g)]$ is continuous on $(0,\infty)$, and it is strictly positive for every $h>0$: otherwise $\sigma(\sqrt h\, g)=0$ almost surely, i.e.\ $\sigma=0$ Lebesgue-almost everywhere, which is excluded by assumption. Hence the infimum over the compact interval $[a_{\ell-1},b_{\ell-1}]\subset(0,\infty)$ is attained and positive. 

Note also that $m_\ell\le L_\sigma^2 b_{\ell-1}$, since $\E_{Z\sim\cN(0,h)}[\sigma^2(Z)]\le L_\sigma^2 h$. We notice that on $E_{\ell-1}$,
\[
\E \big[Y_i^j \ \big| \ \cF_{\ell-1}\big] = \E_{Z\sim\cN(0,s_j^2)}\big[\sigma^2(Z)\big] \ge m_\ell \, ,
\]
by definition of $m_l$. Conditioned on $\cF_{\ell-1}$, we note that $Y^j_1, \dots, Y^j_{d_\ell}$ are sub-exponential i.i.d.\ random variables with
\begin{align*}
\big\|Y_i^j-\E [Y_i^j \ | \ \cF_{\ell-1}]\big\|_{\psi_1} & \le C\,\|Y_i^j\|_{\psi_1} \\
&\le C L_\sigma^2 \big\|(Z_i^j)^2\big\|_{\psi_1} \\
& = C L_\sigma^2 \|Z_i^j\|_{\psi_2}^2\\
&= \tfrac{8}{3}\, C L_\sigma^2 \frac{\|\alpha^{(\ell-1)}(X'_{:j})\|_2^2}{d_{\ell-1}} \\
&\le \tfrac{8}{3}\, C L_\sigma^2 b_{\ell-1}\, ,
\end{align*}
where $C\ge1$ is an absolute constant, the $\psi_1$- and $\psi_2$-norms are taken with respect to the conditional distribution given $\cF_{\ell-1}$, and we used Lemma \ref{lem:subexp-subgauss-squared}, Lemma \ref{lem:centering}, and Lemma \ref{lem:gausspsi2} with $\tau=\frac{\|\alpha^{(\ell-1)}(X'_{:j})\|_2^2}{d_{\ell-1}}$. By applying Theorem~\ref{cor:bernstein} conditioned on $\cF_{\ell-1}$ to the centered variables $Y_i^j-\E[Y_i^j\,|\,\cF_{\ell-1}]$, with $t:=\frac{m_\ell}{2}d_\ell$, it follows that on $E_{\ell-1}$,
\begin{align*}
    \P\left(\|\alpha^{(\ell)}(X'_{:j})\|_2^2 < a_\ell d_\ell \ \Big| \ \cF_{\ell-1} \right)
    & \le \P\left(\left|\sum_{i=1}^{d_\ell} Y_i^j - d_{\ell}\,\E [Y_1^j \ | \ \cF_{\ell-1}]\right| \ge \frac{m_\ell}{2} d_\ell \ \Big| \ \cF_{\ell-1} \right) \\
    & \le 2\exp\left\{-c\min\left(\frac{t^2}{d_\ell K_\ell^2},\frac{t}{K_\ell}\right)\right\} \, ,
\end{align*}
where $c>0$ is an absolute constant, and in the first step we used that, on $E_{\ell-1}$, $a_\ell d_\ell=\frac{m_\ell}{2}d_\ell\le d_\ell\,\E[Y_1^j\,|\,\cF_{\ell-1}]-\frac{m_\ell}{2}d_\ell$, since $\E[Y_1^j\,|\,\cF_{\ell-1}]\ge m_\ell$. We now compute the minimum in the exponent. Plugging in $t=\frac{m_\ell}{2}d_\ell$, we have
\begin{align*}
    \min\left(\frac{t^2}{d_\ell K_\ell^2},\frac{t}{K_\ell}\right)
    = d_\ell\min\left\{\left(\frac{m_\ell}{2K_\ell}\right)^2,\ \frac{m_\ell}{2K_\ell}\right\}
    = d_\ell\left(\frac{m_\ell}{2K_\ell}\right)^2 \, ,
\end{align*}
where the last equality holds because $m_\ell\le L_\sigma^2 b_{\ell-1}\le K_\ell$ gives $\frac{m_\ell}{2K_\ell}\le\frac12<1$, so that the square is the smaller of the two terms. Therefore, on $E_{\ell-1}$, we have
\begin{align}\label{eq:smalll}
    \P\left(\|\alpha^{(\ell)}(X'_{:j})\|_2^2 < a_\ell d_\ell \ \Big| \ \cF_{\ell-1} \right)
    \le 2\exp\left\{-c_\ell\, d_\ell\right\} \, ,
\end{align}
where $c_\ell=c\left(\frac{m_\ell}{2K_\ell}\right)^2>0$ depends only on $\sigma$ and $\ell$, and not on $d_\ell$ or $N_c$.

Using \eqref{eq:altogheter}, \eqref{eq:upppppp}, \eqref{eq:smalll} and taking the union over $j\in[N_c]$, we obtain
\begin{align*}
    \P\left(E_\ell^C\right)\ \le\ 2N_c\, e^{-d_0/16}+N_c\sum_{k=1}^{\ell}\left(2e^{-c_k d_k}+e^{-d_k/16}\right)\, .
\end{align*}
Under (H3), there exist universal constants $C, c>0$ which depends only on $\sigma$ and $\kappa$ such that
\begin{align*}
    \P\left(E_\ell^c\right)\ \le \ C N_c\, e^{-cd}\, .
\end{align*}

Finally, let $a:= \alpha_{\ell}\min_{0\le\ell\le L-1}a_\ell>0$. Since on the event $E_\ell$ we have $\|\alpha^{(\ell)}(X'_{:j})\|_2^2\ge a_\ell d_\ell\ge a\, d_\ell$ for every $j\in[N_c]$, the first display of the proof yields the inclusion
\begin{align*}
    E_\ell\ \subseteq\ \left\{\Tr\left(\alpha^{(\ell)}(X')\ \alpha^{(\ell)}(X')^\top\right)\ \ge\ a\, d N_c\right\} ,
\end{align*}
and therefore
\begin{align*}
    \P\left(\Tr\left(\alpha^{(\ell)}(X')\ \alpha^{(\ell)}(X')^\top\right)\ <\ a\, d_\ell N_c\right)
    \ \le\ \P\left(E_\ell^c\right)
    \ \le\ C N_c\, e^{-cd}\, .
\end{align*}
This concludes the proof.
\end{proof}
\begin{restatable}{theorem}{mlpworks}\label{thm:mlpworks}
Let us assume (H1), (H3) and (H4), and let $\ell\in[L-1]$. Then there exist constants $ C,c, \bar{C}>0$ not dependent on $d$, and an absolute constant $D>0$, such that for all $d>D$, with probability at least $1-\bar{C} N_c\, e^{- c\, d}$ over $\theta$ and $X$,
\[
\reff\big(\alpha^{(\ell)}(X')\,\alpha^{(\ell)}(X')^\top\big)\ \ge\, C\log^8 d\, .
\]
\end{restatable}

\begin{proof}
By definition of the effective rank,
\begin{align*}
    \reff\big(\alpha^{(\ell)}(X')\ \alpha^{(\ell)}(X')^\top\big) = \frac{\Tr\big(\alpha^{(\ell)}(X')\ \alpha^{(\ell)}(X')^\top\big)}{\opn{\alpha^{(\ell)}(X')\ \alpha^{(\ell)}(X')^\top}}\, ,
\end{align*}
By Lemma~\ref{lem:opnorm}, there exist constants $C,c>0$ such that
\[
\P_{\theta, X'} \left(\opn{\alpha^{(\ell)}(X')\,\alpha^{(\ell)}(X')^\top}> C(d+N_c)\right)\ \le\ 2\cdot 9^{-c(d+N_c)}+2\ell e^{-d}\, .
\]
On the other hand, by Lemma~\ref{lem:trace-lower}, there exist constants $a, c', C'>0$ such that
\begin{align*}
    \P_{\theta,X'}\left(\Tr\left(\alpha^{(\ell)}(X')\,\alpha^{(\ell)}(X')^\top\right)\ <\ a\, d N_c\right)\ \le\ C' N_c\, e^{-c' d} \, .
\end{align*}
Therefore, by a union bound, with probability at least
\[
1-2\cdot 9^{-c(d+N_c)}-2\ell e^{-d}-C'N_c e^{-c'd}\ \ge\ 1-\bar C N_c\, e^{-\bar c\, d}
\]
with $\bar c:=\min\{c\log 9,\,1,\,c'\},\bar C:=2+2L+C'$, it holds that
\begin{align*}
    \reff\big(\alpha^{(\ell)}(X')\ \alpha^{(\ell)}(X')^\top\big)
    = \frac{\Tr\big(\alpha^{(\ell)}(X')\ \alpha^{(\ell)}(X')^\top\big)}{\opn{\alpha^{(\ell)}(X')\ \alpha^{(\ell)}(X')^\top}}
    \ \ge\ \tilde{C}\,\frac{d N_c}{d+N_c}
    \ \ge\ \frac{\tilde{C}}{2}\,\min\{d,N_c\}\, ,
\end{align*}
with $\tilde C:=a\kappa/C$, where in the last step we have used the fact that $\frac{dN_c}{d+N_c}\ge\frac12\min\{d,N_c\}$. Therefore, since by (H4) $N_c\ge\log^8 d$, and $\log^8 d\le d$ for $d$ large enough, we conclude that, with probability at least $1-\bar C N_c e^{-\bar c d}$,
\begin{align*}
    \reff\big(\alpha^{(\ell)}(X')\ \alpha^{(\ell)}(X')^\top\big)\ \ge\ \frac{\tilde C}{2}\,\log^8 d\, .
\end{align*}
\end{proof}
\subsection{Proof of Proposition \ref{prop:scalarlandscape} }\label{app:lim_func}
We start by proving that $\cJ$ separates the two sources of quantization error. The first is the rounding error, which does not depend on the bit precision and would be present even with an infinite grid; the second is the clipping error, due to inputs falling outside the representable range $[-M_Bs,M_Bs]$, and it is the only term through which $M_B$ enters. We record this decomposition in the following corollary, which will be the starting point of our analysis of the loss landscape.
\begin{corollary}\label{cor:series}
For every $B\geq2$ and $s>0$,
\begin{align*}
\cJ(M_B,s) &=\mathcal{J}_\infty(s)+4s\sum_{j=0}^{\infty}g\!\left(\left(M_B+j+\tfrac12\right)s\right), 
\end{align*}
with
\begin{align*}
\mathcal{J}_\infty(s)&:= \E_{Z\sim\cN(0,1)}[(Z- s\lfloor Z/s \rceil)^2] = 1 - \sum_{j=0}^{\infty}4s\,g\bigl((j+\tfrac12)s\bigr)\, ,
\end{align*}
where $g(x):= \varphi(x) - x (1-\Phi(x))$, and both series converge absolutely.
\end{corollary}
\begin{proof}
    For a natural number $n\ge0$, let $\Pi_{s,n}(z):=s\max\{-n,\min\{n,\lfloor z/s\rceil\}\}$ denote the quantizer with $2n+1$ levels, so that $\cJ(n,s)=\E[(Z-\Pi_{s,n}(Z))^2]$. Let us fix $n\ge 0$, and notice that the quantizers $\Pi_{s,n}$ and $\Pi_{s,n+1}$ agree on $\{|z|\le (n+\tfrac12)s\}$. On $\{z>(n+\tfrac12)s\}$ one has $\Pi_{s,n}(z)=ns$, while $\lfloor z/s\rceil\ge n+1$ there, so
$\Pi_{s,n+1}(z)=(n+1)s$; the case $z<-(n+\tfrac12)s$ is symmetric. It holds
    \begin{align*}
       \cJ(n,s) - \cJ( n+1,s) & = 2\int_{(n+1/2)s}^{\infty}\Bigl[(z-ns)^2-\bigl(z-(n+1)s\bigr)^2\Bigr]\varphi(z)\,d z \\
       & \overset{(*)}{=} 4s\int_{(n+\tfrac12)s}^{\infty}(z-(n+\tfrac12)s)\varphi(z)\,d z\\
       & 
  =4s\bigl[\varphi((n+\tfrac12)s)-(n+\tfrac12)s\ \left(1-\Phi\left((n+\tfrac12)s\right)\right)\bigr]\\
  & =4s\,g\left((n+\tfrac12)s\right),
    \end{align*}
where in (*) we have used the fact that
\[
  (z-ns)^2-\bigl(z-(n+1)s\bigr)^2=2zs-2ns^2-s^2=2s\Bigl(z-\left(n+\tfrac12\right)s\Bigr),
\]
and $\int_x^\infty z\varphi(z) dz=\varphi(x)$ and $\int_x^\infty\varphi(z)\ dz=1-\Phi(x)$.
Let us fix $N>M_B$. Summing on $n=M_B, \dots, N-1$ gives
\[
\cJ(M_B,s)-\cJ(N,s)=\sum_{n=M_B}^{N-1}4s\,g\bigl((n+\tfrac12)s\bigr)\, .
\]
Since $0\le g(x)\le\varphi(x)$ for $x\ge0$, the series converges absolutely. Moreover
$\bigl(Z-\Pi_{s,N}(Z)\bigr)^2\le Z^2\in L^1$ and $\Pi_{s,N}(Z)\to s\lfloor Z/s \rceil$ point-wise as
$N\uparrow\infty$, so dominated convergence gives $\cJ(N, s)\to \mathcal{J}_\infty(s)$. Letting
$N\uparrow\infty$ and substituting $j=n-M_B$ concludes the proof. \newline 
Finally, since $\Pi_{s,0}\equiv0$ and hence $\cJ(0,s)=\E[Z^2]=1 $, we can use the above formula for $\cJ(M_B,s)-\cJ(N,s)$ with $M_B$ replaced by $0$ and $N \to \infty$ to prove that 
\[
\mathcal{J}_\infty(s)=1 - \sum_{n=0}^{\infty}4s\,g\bigl((n+\tfrac12)s\bigr)\, .
\]
 
\end{proof}
For clarity, we restate Proposition \ref{prop:scalarlandscape} below.
\inftyprop*
\begin{proof}
We prove each of the four points separately.
\begin{itemize}
    \item[(i)] The function $g$ is $\mathcal C^\infty$, with $g'=-(1-\Phi)$, $g''=\varphi$ and $g^{(k)}=p_k\varphi$ for a polynomial $p_k$ when $k\ge2$; in particular $|g^{(k)}(x)|\le C_k(1+x)^{k}\varphi(x)$ for $x\ge1$ and all $k\ge0$, using $g(x)\le\varphi(x)/x^2$ and $1-\Phi(x)\le\varphi(x)/x$ by Lemma~\ref{lem:millratio}. Hence every partial derivative in $(m,s)$ of the summand $s\,g\big((m+j+\tfrac12)s\big)$ is bounded by a summable function, on $\{s\ge s_0,\,m\ge1\}$. All series of partial derivatives thus converge uniformly on compact subsets of $[1,\infty)\times(0,\infty)$, and term-wise differentiation gives $\cbJ\in\mathcal C^\infty\big([1,\infty)\times(0,\infty)\big)$.
\item[(ii)] We start by noticing that
\begin{align*}
    g(as) & = \varphi(as) - as\big(1-\Phi(as)\big) = \int_{as}^\infty t\varphi(t)dt - as \int_{as}^\infty \varphi(t) dt \\
    & = \int_{as}^\infty (t -as) \varphi(t)dt \\
    & = s^2\int_{a}^\infty (u -a) \varphi(su) du \\
    & = s^2\int_{0}^\infty (u -a)_+ \varphi(su) du \, .
\end{align*}

Therefore, by definition of $\cbJ$ (cf.\ Corollary \ref{cor:series}), we can write
\begin{align*}
    \cbJ(m,s) &=1-4s\sum_{j= 0}^\infty
\left[
g\bigl((j+\tfrac12)s\bigr)
-g\bigl((j+m+\tfrac12)s\bigr)
\right] \\
& = 1-4s\sum_{j= 0}^\infty s^2\int_{0}^\infty \left[\left(u -(j+\tfrac12)\right)_+ -\left(u -(j+m+\tfrac12)\right)_+\right]\varphi(su) du \\
& = 1-4s^3\int_{0}^\infty \sum_{j= 0}^\infty\left[\left(u -(j+\tfrac12)\right)_+ -\left(u - m -(j+\tfrac12)\right)_+\right]\varphi(su) du \\
& = 1-4s^3\int_{0}^\infty D_m(u)\varphi(su) du \, ,
\end{align*}
where $D_m(u) := \sum_{j= 0}^\infty\left[\left(u -(j+\tfrac12)\right)_+ -\left(u - m -(j+\tfrac12)\right)_+\right]$.
Differentiating in $s$, it holds
\begin{align*}
    \partial_s \, \cbJ(m,s) &= -12s^2 \int_{0}^\infty D_m(u)\varphi(su) du + 4s^4 \int_{0}^\infty u^2D_m(u)\varphi(su) du \\
    & = 4s^2 \int_0^{\infty} (s^2u^2 - 3)D_m(u)\varphi(su) du \\
    & = -4s^2 \int_0^{\infty} (1-s^2u^2)D_m(u)\varphi(su) du - 8s^2 \int_0^{\infty} D_m(u)\varphi(su) du \\
    & \overset{(*)}{=} -4s^2 \int_0^{\infty} D_m(u) \frac{d}{d u}\left(u\varphi(su)\right) du - 8s^2 \int_0^{\infty} D_m(u)\varphi(su) du \\
    & = -4s^2 \left[D_m(u) u\varphi(su)\Big|_0^{\infty} - \int_0^{\infty}  u D'_m(u) \varphi(su)\, du\right]  - 8s^2 \int_0^{\infty} D_m(u)\varphi(su) du \\
    &=4s^2 \int_0^{\infty} [uD'_m(u) -2D_m(u)]\varphi(su) du
\end{align*}
where in (*) we used the fact that 
\[
\frac{d}{d u}\left(u\varphi(su)\right) = (1-s^2u^2)\varphi(us)\, .
\]
Integrating by parts we have 
\begin{equation}
    \begin{split}
         & \int [uD'_m(u) -2D_m(u)] du  \\
    & \quad= uD_m(u) - 3\int D_m(u)du \\
    & \quad= uD_m(u) - 3 \sum_{j=0}^\infty \int \left[\left(u -(j+\tfrac12)\right)_+ -\left(u - m -(j+\tfrac12)\right)_+\right] du \\
    &\quad= uD_m(u) - \frac{3}{2} \sum_{j=0}^\infty \Big[\left(u -(j+\tfrac12)\right)_+^2 -\left(u - m -(j+\tfrac12)\right)_+^2\Big] \, .        
    \end{split}
\end{equation}

Using again integration by parts, we get that
\begin{align*}
& \partial_s \, \cbJ(m,s) \\
& \quad= -4s^2 \int_0^{\infty}  \left(uD_m(u) -\frac{3}{2} \sum_{j=0}^\infty \Big[\left(u -(j+\tfrac12)\right)_+^2 -\left(u - m -(j+\tfrac12)\right)_+^2\Big]\right) \frac{d}{du} \varphi(su)  \, du\\
& \quad= 4s^4 \int_0^{\infty}  u P_m(u) \varphi(su)  \, du
\end{align*}
where $P_m(u):= uD_m(u) -\frac{3}{2} \sum_{j=0}^\infty \Big[\left(u -(j+\tfrac12)\right)_+^2 -\left(u - m -(j+\tfrac12)\right)_+^2\Big]$.
Now we notice that, for $0<u< m+\tfrac{1}{2}$, it holds 
\begin{align*}
    P_m(u) & =  \sum_{j=0}^\infty  u\left(u -(j+\tfrac12)\right)_+ -u\left(u - m -(j+\tfrac12)\right)_+ -\frac{3}{2}\left(u -(j+\tfrac12)\right)_+^2 \\
    & \quad+ \frac{3}{2}\left(u - m -(j+\tfrac12)\right)_+^2 \\
    & =  \sum_{j=0}^\infty \left(u-\frac{3}{2}\left(u -(j+\tfrac12)\right)_+\right)\left(u - j-\tfrac12\right)_+ \\
    & = \sum_{j=0}^{\lfloor u-1/2\rfloor_+} \left(u-\frac{3}{2}\left(u -j -\tfrac12\right)\right)\left(u - j-\tfrac12\right)\\
    & = \sum_{j=0}^{\lfloor u-1/2\rfloor_+} \left(-\frac{u}{2} +\frac 32 j +\frac34\right)\left(u - j-\tfrac12\right)   \, ,
\end{align*}
where $\sum_{j=0}^{\lfloor u-1/2\rfloor_+}$ denotes the sum over the indices $j\ge0$ with $u-j-\tfrac12>0$ (empty if $u<\tfrac12$). Writing $n:=\lfloor u-1/2\rfloor_++1$ for the number of such indices, the last sum equals $\frac{n}{2}\big(\frac14-(u-n)^2\big)$, which is nonnegative since $|u-n|\le\tfrac12$; hence $P_m\ge0$ on $(0,m+\tfrac12)$, and it vanishes exactly at the half-integers.

On the other hand, for all $m> \tfrac{1}{2}$, $u> m+1$, calling $n=\lfloor u-1/2\rfloor_++1$ and $r = \lfloor u-m-1/2\rfloor_+ +1$ we have that 
\begin{align*}
     P_m(u) & =  \sum_{j=0}^\infty  u\left(u -(j+\frac12)\right)_+ -u\left(u - m -(j+\frac12)\right)_+ -\frac{3}{2}\left(u -(j+\frac12)\right)_+^2 \\
     & \quad + \frac{3}{2}\left(u - m -(j+\frac12)\right)_+^2 \\ 
    & = \sum_{j=0}^{\lfloor u-1/2\rfloor_+} \left(-\frac{u}{2} +\frac 32 j +\frac34\right)\left(u - j-\frac12\right)  \\
    & \quad+ \sum_{j=0}^{\lfloor u-m-1/2\rfloor_+} \left( \frac{u}{2} - \frac32\left(j+\frac12 \right)  - \frac32 m\right) \left(u - m -j-\frac12\right)_+ \\
    & = \sum_{j=0}^{n-1} \left(-\frac{u^2}{2} + 2u\left(j+\frac12\right) - \frac{3}{2} \left(j+\frac12\right)^2\right) \\
    & \qquad+ \sum_{j=0}^{r-1}  \left( \frac{(u-m)^2}{2} -2(u-m)\left(j+\frac12 \right) + \frac32\left(j+\frac12 \right)^2  - m\left(u-m - \left(j+\frac12 \right)\right)\right) \\
    & = \frac{n}{2}\left(\frac14 -(u-n)^2\right) + \frac{r(u-m)^2}{2} - r^2(u-m) +\frac{r(4r^2-1)}{8} - m\left(r(u-m)- \frac{r^2}{2}\right)\\
    & = \frac{n}{2}\left(\frac14 -(u-n)^2\right)- \frac{r}{2}\left(\frac{1}{4} -(u-m -r)^2\right) - m\left(r(u-m)- \frac{r^2}{2}\right) \\
    & \overset{(*)}{\le} \frac{n}{8}- m\left(u-m-\frac{1}{2}\right) \\
  & \le \left(\frac{1}{8}-m\right)(u-m-1) + \frac{3}{16}(1-2m) \\
& < 0 \, ,
\end{align*}
where in (*) we have noticed that 
\begin{align*}
  & -\frac{1}{2} \le u-m-r = u-m-(\lfloor u-m-1/2\rfloor_+ +1) \le\frac{1}{2} \\
  & \qquad\implies \frac{r}{2}\left(\frac{1}{4} -(u-m -r)^2\right) \ge 0 \\
  & \left(r(u-m)- \frac{r^2}{2}\right) = \sum_{j=0}^{r-1} \left(u-m -j-\frac{1}{2}\right) \ge u-m -\frac{1}{2}\, ,
\end{align*}
and in the following step we used $n\le u+\tfrac12$.
We consider now what happens when $m+\tfrac{1}{2} < u < \min\{m+1, \lfloor m\rfloor +\tfrac{3}{2}\}$. On this interval, it holds
\begin{align*}
    D_m(u) & = \sum_{j= 0}^\infty\left[\left(u -(j+\tfrac12)\right)_+ -\left(u - m -(j+\tfrac12)\right)_+\right]\\
    & = \sum_{j=0}^{\lfloor m\rfloor} (u-j-\frac{1}{2}) - (u-m-\frac{1}{2})\\
    & = (\lfloor m \rfloor + 1)u -\frac{(\lfloor m\rfloor +1)^2}{2} - u + m + \frac{1}{2}\\
    & = \lfloor m \rfloor u - \frac{\lfloor m \rfloor(\lfloor m\rfloor + 2)}{2} + m
\end{align*}
and, in particular $D'_m(u) = \lfloor m \rfloor$. Note also that
\begin{align*}
 \frac{d}{d u } \sum_{j=0}^\infty \Big[\left(u -(j+\tfrac12)\right)_+^2 -\left(u - m -(j+\tfrac12)\right)_+^2\Big] = 2 D_m(u)   
\end{align*}
Let us compute 
\begin{align*}
    \frac{d}{d u} P_m(u) & = D_m(u) + u D'_m (u) - 3 D_m(u)\\
    & = u D'_m(u) - 2 D_m(u)\\
    & = u \lfloor m \rfloor - 2 \left(\lfloor m \rfloor u - \frac{\lfloor m \rfloor(\lfloor m\rfloor + 2)}{2} + m\right)\\
    & = \lfloor m \rfloor(\lfloor m\rfloor + 2 - u) - 2m \\
    & \le \lfloor m \rfloor(\lfloor m\rfloor + \frac32 - m) - 2m\\
    & = \lfloor m \rfloor\left(\frac32 -(m-\lfloor m\rfloor)\right) - 2 \lfloor m \rfloor -2(m - \lfloor m \rfloor)\\
    & = - \frac{\lfloor m \rfloor}{2} - (\lfloor m \rfloor + 2)(m- \lfloor m \rfloor)<0
\end{align*}
for all $m>0$. We notice that $P_m(m+\tfrac{1}{2}) = \frac{\lfloor m \rfloor+1}{2}(m -\lfloor m \rfloor)(\lfloor m \rfloor+1-m)$ is positive for non-integer $m$ and $0$ for integer $m$. 

Now, it follows that 
\begin{itemize}
    \item If $\min\{m+1, \lfloor m\rfloor +\tfrac{3}{2}\}= m+1$, then
    \begin{align*}
        P_m(m+1) \le \frac{3}{16}(1-2m) < 0 \, .
    \end{align*}
    \item If $\min\{m+1, \lfloor m\rfloor +\tfrac{3}{2}\}= \lfloor m\rfloor +\tfrac{3}{2}$, then
    \begin{align*}
        P_m\big( \lfloor m\rfloor +\tfrac{3}{2}\big)
        & = \big(\lfloor m\rfloor +\tfrac{3}{2}\big)\sum_{j= 0}^\infty\Big[\big(\lfloor m\rfloor +1 -j\big)_+ -\big(\lfloor m\rfloor +1- m -j\big)_+\Big] \\
        & \qquad -\frac{3}{2} \sum_{j=0}^\infty \Big[\big(\lfloor m\rfloor +1 -j\big)_+^2 -\big(\lfloor m\rfloor +1- m -j\big)_+^2\Big]\\
        & = \big(\lfloor m\rfloor +\tfrac{3}{2}\big)\Big[\sum_{j= 0}^{\lfloor m\rfloor} \big(\lfloor m\rfloor +1 -j\big) -\big(\lfloor m\rfloor +1- m\big)\Big] \\
        & \qquad -\frac{3}{2} \Big[\sum_{j=0}^{\lfloor m\rfloor} \big(\lfloor m\rfloor +1 -j\big)^2 -\big(\lfloor m\rfloor +1- m\big)^2\Big] \\
        &=\big(\lfloor m\rfloor+\tfrac32\big)
        \Big[\frac{(\lfloor m\rfloor+1)(\lfloor m\rfloor+2)}{2}-\big(\lfloor m\rfloor+1-m\big)\Big]\\
        &\qquad-\frac32
        \Big[\frac{(\lfloor m\rfloor+1)(\lfloor m\rfloor+2)(2\lfloor m\rfloor+3)}{6}-\big(\lfloor m\rfloor+1-m\big)^2\Big]\\
        &=-\big(\lfloor m\rfloor+\tfrac32\big)\big(\lfloor m\rfloor+1-m\big)
        +\frac32\big(\lfloor m\rfloor+1-m\big)^2\\
        &=\big(\lfloor m\rfloor+1-m\big)
        \Big[-\lfloor m\rfloor-\tfrac32+\frac32\big(\lfloor m\rfloor+1-m\big)\Big]\\
        &=-\frac12\big(\lfloor m\rfloor+1-m\big)\big(3m-\lfloor m\rfloor\big)\ <\ 0\, ,
    \end{align*}
    where in the second equality we used that $0<\lfloor m\rfloor+1-m\le1$, so that $(\lfloor m\rfloor+1-m-j)_+$ vanishes for $j\ge1$, and the last inequality holds since $\lfloor m\rfloor+1-m>0$ and $3m-\lfloor m\rfloor\ge2m>0$.
\end{itemize}
It remains to check what happens for $\lfloor m\rfloor +\tfrac{3}{2}\le u\le m+1$. On the interval, it holds
\begin{align*}
P_m(u)
&=\sum_{j=0}^\infty\Big[u\big(u-j-\tfrac12\big)_+-\frac32\big(u-j-\tfrac12\big)_+^2\Big] \\
& \quad -\sum_{j=0}^\infty\Big[u\big(u-m-j-\tfrac12\big)_+-\frac32\big(u-m-j-\tfrac12\big)_+^2\Big]\\
&=\sum_{j=0}^{\lfloor m\rfloor+1}\Big[u\big(u-j-\tfrac12\big)-\frac32\big(u-j-\tfrac12\big)^2\Big]
 -u\big(u-m-\tfrac12\big)+\frac32\big(u-m-\tfrac12\big)^2\\
&=\sum_{j=0}^{\lfloor m\rfloor+1}\Big[-\frac{u^2}{2}+2u\big(j+\tfrac12\big)-\frac32\big(j+\tfrac12\big)^2\Big]
 \\
 & \quad +\big(u-m-\tfrac12\big)\Big[\frac12\big(u-m-\tfrac12\big)-\frac12-m\Big]\\
&=-\frac{\lfloor m\rfloor+2}{2}\,u^2+(\lfloor m\rfloor+2)^2\,u-\frac{(\lfloor m\rfloor+2)\big[4(\lfloor m\rfloor+2)^2-1\big]}{8}\\
&\qquad-\frac12\big(u-m-\tfrac12\big)\big(m+\tfrac32-u\big)-m\big(u-m-\tfrac12\big)\\
&=\frac{\lfloor m\rfloor+2}{2}\big(u-\lfloor m\rfloor-\tfrac32\big)\big(\lfloor m\rfloor+\tfrac52-u\big)
 \\
 & \quad-\frac12\big(u-m-\tfrac12\big)\big(m+\tfrac32-u\big)-m\big(u-m-\tfrac12\big)\\
&\overset{(*)}{\le}\frac{u-m-\tfrac12}{2}\Big[(\lfloor m\rfloor+2)\big(m-\lfloor m\rfloor\big)-\tfrac12-2m\Big]\\
&=-\frac{u-m-\tfrac12}{2}\Big[\lfloor m\rfloor\big(2-m+\lfloor m\rfloor\big)+\tfrac12\Big]\ <\ 0\, ,
\end{align*}
where in (*) we used 
\begin{align*}
0& \le u-\lfloor m\rfloor-\tfrac32 = (u-m-\frac12) - (\lfloor m\rfloor +1 - m) \\
& \le(u-m-\frac12) - (\lfloor m\rfloor +1 - m)(u-m-\frac12)\\
& =\left[1 -  (\lfloor m\rfloor +1 - m)\right](u-m-\frac12)\\
& = (m- \lfloor m\rfloor)(u-m-\frac12)
\end{align*}
together with
\begin{align*}
& 0\le\lfloor m\rfloor+\tfrac52-u\le1 \\
& m+\tfrac32-u\ge m + \frac{3}{2} -(m+1) =\frac12.
\end{align*}
Therefore, for all $m> 1/2$, $P_m$ has a unique zero $u^*_m$ in $[m+\tfrac12, \min\{m+1, \lfloor m\rfloor +\tfrac{3}{2}\}]$, and moreover $P_m\ge0$ on $(0,u^*_m)$ and $P_m<0$ on $(u^*_m,\infty)$. Then 
\begin{align*}
\partial_s \, \cbJ(m,s) & = 4s^4 \int_0^{\infty}  u P_m(u) \varphi(su)  \, du\\
& = \frac{4s^4}{\sqrt{2\pi}} e^{-s^2u^{*^2}_m/2} \int_0^\infty u P_m(u) e^{-s^2(u^2 - u^{*^2}_m)/2} du
\end{align*}
Notice now that $\frac{4s^4}{\sqrt{2\pi}} e^{-s^2u^{*^2}_m/2} >0$ for all $m, s$, and $\int_0^\infty u P_m(u) e^{-s^2(u^2 - u^{*^2}_m)/2} du$ is strictly increasing in $s$, since its derivative in $s$ is $\int_0^\infty u\, s\,(u^{*^2}_m - u^2) P_m(u) e^{-s^2(u^2 - u^{*^2}_m)/2} du$ and the integrand is nonnegative and not identically zero, as $(u^{*^2}_m - u^2)$ and $P_m(u)$ have the same sign. Therefore, for all $m>1/2$, $s \mapsto \cbJ_s(m, s)$ has at most one zero. For all $m>1/2$, a minimum exists because 
\begin{align*}
    \cbJ(m,s) & =  1-4s^3\int_{0}^\infty D_m(u)\varphi(su) du \\
    &\le  1-4s^3\int_{1/2}^\infty D_m(u)\varphi(su) du \\
    & \le 1- 4s^3\int_{m+\tfrac{1}{2}}^\infty m\varphi(su) du \\
    & <1
\end{align*}
while $\lim_{s\to0}\,\cbJ(m,s)=\lim_{s\to\infty}\,\cbJ(m,s)=1$; indeed, $D_m(u)=0$ for $u<\tfrac12$ and $D_m(u)\le m(u+\tfrac12)$, so that $4s^3\int_0^\infty D_m(u)\varphi(su)\,du\le 4ms^3\int_{1/2}^\infty(u+\tfrac12)\varphi(su)\,du$, which tends to $0$ both as $s\to0$ and as $s\to\infty$.

We note now that at the critical point $s^*(m)$ we have
\begin{align*}
    0 & =  \cbJ_s(m,s^*(m)) \\
    & = \frac{4s^*(m)^4}{\sqrt{2\pi}} e^{-s^{*^2}(m)u^{*^2}_m/2} \int_0^\infty u P_m(u) e^{-s^{*^2}(m)(u^2 - u^{*^2}_m)/2} du \\
    &\implies \int_0^\infty u P_m(u) e^{-s^*(m)^2(u^2 - u^{*^2}_m)/2}\, du=0 \, .
\end{align*}
Finally, at $s^*(m)$, it holds
\begin{align*}
    &  \cbJ_{ss}(m,s^*(m)) \\
    & \quad= \frac{4s^*(m)^4}{\sqrt{2\pi}} e^{-s^{*^2}(m)u^{*^2}_m/2}\frac{d}{ds}\left(\int_0^\infty u P_m(u) e^{-s^2(u^2 - u^{*^2}_m)/2} du\right)\Big|_{s=s^*(m)}\\
    & \quad= \frac{4s^*(m)^4}{\sqrt{2\pi}} e^{-s^{*^2}(m)u^{*^2}_m/2}\int_0^\infty u\, s^{*}(m)(u^{*^2}_m - u^2) P_m(u)e^{-s^{*^2}(m)(u^2 - u^{*^2}_m)/2} du \\
    &\quad >0 \, 
\end{align*}
since $(u^{*^2}_m - u^2)$ and $P_m(u)$ have the same sign for all $u\in(0,\infty)$, and their product is not identically zero.
\item[(iii)] The function $z\mapsto\big(z-s\lfloor z/s\rceil\big)^2$ is even and $s$-periodic, and equals $z^2$ on $(-\tfrac s2,\tfrac s2)$; its Fourier expansion is therefore the rescaled expansion of $x^2$ on $(-\pi,\pi)$, namely
\[
\big(z-s\lfloor z/s\rceil\big)^2=\frac{s^2}{12}+\frac{s^2}{\pi^2}\sum_{k=1}^{\infty}\frac{(-1)^k}{k^2}\cos\Big(\frac{2\pi k z}{s}\Big),
\]
with absolute and uniform convergence. Taking expectations with respect to $Z\sim\cN(0,1)$ and using $\E[\cos(tZ)]=e^{-t^2/2}$ with $t=2\pi k/s$ yields
\begin{align*}
\mathcal{J}_\infty(s)& =\E\Big[\big(Z-s\lfloor Z/s\rceil\big)^2\Big]=\frac{s^2}{12}+\frac{s^2}{\pi^2}\sum_{k=1}^{\infty}\frac{(-1)^k}{k^2}\,e^{-2\pi^2k^2/s^2}\\
& = \frac{s^2}{12} + \frac{s^2}{\pi^2} \sum_{k=1}^\infty \frac{(-1)^k}{k^2} e^{-k^2\tau(s)}\, ,
\end{align*}
where $\tau(s) := \frac{2\pi^2}{s^2}$.
Therefore, for all $s>0$, we have that
\begin{align*}
 \mathcal{J}_{\infty}'(s) & = \frac{s}{6} + \frac{1}{\pi^2}\sum_{k=1}^\infty \frac{(-1)^k}{k^2} \frac{d}{ds} \left(s^2e^{-k^2\tau(s)}  \right)\\
 & =  \frac{s}{6} + \frac{1}{\pi^2}\sum_{k=1}^\infty \frac{(-1)^k}{k^2} \left[2s e^{-k^2\tau(s)} - s^2e^{-k^2\tau(s)} k^2\tau'(s)\right] \\
& = \frac{s}{6} + \frac{1}{\pi^2}\sum_{k=1}^\infty \frac{(-1)^k}{k^2} \left[2s e^{-k^2\tau(s)} + s^2e^{-k^2\tau(s)} \frac{4\pi^2k^2}{s^3}\right]\\
 & =  \frac{s}{6} + \frac{1}{\pi^2}\sum_{k=1}^\infty \frac{(-1)^k}{k^2} 2s\left[ 1 +  \frac{2\pi^2k^2}{s^2}\right]e^{-k^2\tau(s)} \\
 & =  \frac{s}{6} + \frac{1}{\pi^2}\sum_{k=1}^\infty (-1)^k 2s\left[ \frac{1}{k^2} +  \tau(s)\right]e^{-k^2\tau(s)} \\
 & <\frac{s}{6}\, ,
\end{align*}
since  $\left[ \frac{1}{k^2} +  \tau(s)\right]e^{-k^2\tau(s)}$ is strictly decreasing to $0$ in $k$ and $-(1+\tau)e^{-\tau}<0$. Recall now that $g(x) = \varphi(x) - x(1-\Phi(x))$ and $g'(x) = -x\varphi(x) - (1- \Phi(x)) + x\varphi(x)  = -(1-\Phi(x))$, hence
\begin{equation}\label{eq:cbjs}
    \begin{split}
        \cbJ_s(m, s)  &= \mathcal{J}_{\infty}'(s) +4\sum_{j=0}^\infty \left[\varphi\left(\left(m+\frac{1}{2} + j\right)s\right) - \right. \\
   &\qquad \left.2s \left(m+\frac{1}{2} + j\right)\left(1 - \Phi\left(\left(m+\frac{1}{2} + j\right)s\right)\right)\right] \, .
    \end{split}
\end{equation}
We note that, whenever $y\ge \sqrt{2}$, we have 
\begin{align}\label{eq:boundddony}
    \varphi(y) - 2y(1-\Phi(y)) < \frac{1 - y^2}{1+y^2} \varphi(y) \le -\frac{1}{3}\varphi(y)  \,
\end{align}
where we have used Lemma \ref{lem:millratio}. At $s = \frac{\sqrt{2}}{m+ \tfrac{1}{2}}$, it holds $\left(m+\tfrac{1}{2} + j\right)\frac{\sqrt{2}}{m+ \tfrac{1}{2}} \ge \sqrt{2}$ for all $j\ge0$. Therefore, combining the fact that all the terms in the sum appearing in \eqref{eq:cbjs} are negative and \eqref{eq:boundddony}, and keeping only the term $j=0$, we obtain
\begin{align*}
    \cbJ_s\left(m, \frac{\sqrt{2}}{m+ \tfrac{1}{2}}\right) < \frac{\sqrt{2}}{6m+ 3} - \frac{4}{3}\varphi(\sqrt{2}) \le \frac{\sqrt{2}}{9} - \frac{4}{3}\varphi(\sqrt{2}) <0
\end{align*}
for all $m\ge1$, since $\frac{\sqrt2}{9}$ and $\frac43\varphi(\sqrt2)=\frac{4}{3}\frac{e^{-1}}{\sqrt{2\pi}}$.
By (ii) of Proposition \ref{prop:scalarlandscape}, the function $\cbJ_s(m, \cdot)$ has a unique zero in $s^*(m)$ with $\cbJ_{ss}(m, s^*(m))>0$. By continuity it holds that
\begin{align*}
& \cbJ_s(m, s) <0 \qquad s <  s^*(m)   \\
& \cbJ_s(m, s) >0 \qquad s >  s^*(m) \, , 
\end{align*}
and, since $\cbJ_s\left(m, \frac{\sqrt{2}}{m+ \tfrac{1}{2}}\right)<0$, we know that $\frac{\sqrt{2}}{m+ \tfrac{1}{2}} < s^*(m)$. Let us call now $y_j = (m+j+\tfrac{1}{2})s^{*}(m)$. Since $y_j =(m+j+\tfrac{1}{2})s^{*}(m) >\frac{\sqrt{2}}{m+ \tfrac{1}{2}} (m+j+\tfrac{1}{2})\ge\sqrt{2}$, for all $j\ge 0$, we have that 
\begin{align*}
    y_j\varphi(y_j) - 2(1-\Phi(y_j)) \overset{(*)}{>} \left(y_j - \frac{2}{y_j}\right)\varphi(y_j) >0 \, , 
\end{align*}
where in (*) we have used again Lemma \ref{lem:millratio}. Since $g'=-(1-\Phi)$, we have $\cbJ_m(m,s)=-4s^2\sum_{j=0}^\infty\big(1-\Phi\big((m+j+\tfrac12)s\big)\big)$, and differentiating in $s$ gives $\cbJ_{ms}(m,s)=4s\sum_{j=0}^\infty\big[x_j\varphi(x_j)-2(1-\Phi(x_j))\big]$ with $x_j=(m+j+\tfrac12)s$. Therefore, 
\begin{align}\label{eq:cbjms}
    \cbJ_{ms}(m, s^*(m)) = 4s^*(m) \sum_{j=0 }^\infty \Big[y_j \varphi(y_j) - 2 (1-\Phi(y_j))\Big] > 0 \, .
\end{align}
 By the implicit function theorem, since $\cbJ_{ss}\bigl(m,s^*(m)\bigr)>0$ for every $m\ge1$, we have that $s^*(\cdot)\in\mathcal{C}^\infty([1,\infty))$, and 
\begin{align}\label{eq:impthm}
    \frac{d s^*}{dm}(m) =-\frac{\cbJ_{ms}\bigl(m,s^*(m)\bigr)}{\cbJ_{ss}\bigl(m,s^*(m)\bigr)}.
\end{align}
Combining \eqref{eq:cbjms}, \eqref{eq:impthm} and $\cbJ_{ss}\bigl(m,s^*(m)\bigr)>0$, we conclude that 
$$
\frac{d s^*}{dm}(m)
=-\frac{\cbJ_{ms}\bigl(m,s^*(m)\bigr)}
{\cbJ_{ss}\bigl(m,s^*(m)\bigr)}<0\, .
$$ 

\end{itemize}
\end{proof}

\subsubsection{Computation of $\kappa(m)$ for $m>0$ integer} \label{sec:compofkappa}
Let $y_j:=j+\tfrac12$. We recall that
\[
\cbJ(m,s)=1-4\sum_{j=0}^{\infty}s\,g(y_j s)+4\sum_{j=0}^{\infty}s\,g\big((m+j+\tfrac12) s\big) \, ,
\]
where $g(x)=\varphi(x)-x\big(1-\Phi(x)\big)$. Since we are interested in integer $m$, we note that the two series telescope, because $m+j+\tfrac12=y_{j+m}$, and the expression reduces to the finite sum 
$$
\cJ(m,s)=1-4s\sum_{j=0}^{m-1}g(y_js) \, .
$$
Moreover, it holds $g'(x)=-(1-\Phi(x))$ and $g''(x)=\varphi(x)$. Then, for all $c>0$, it holds that
\begin{align*}
\big[s\,g(cs)\big]' & =g(cs)+cs\,g'(cs)\\
& =\varphi(cs)-cs\big(1-\Phi(cs)\big)-cs\big(1-\Phi(cs)\big)\\
& =\varphi(cs)-2cs\big(1-\Phi(cs)\big)\, ,
\end{align*}
and
\begin{align*}
   \big[s\,g(cs)\big]'' & = c\, g'(cs) + c\, g'(cs) + c^2 s\, g''(cs) \\
   & = - 2c\big(1 - \Phi(cs)\big) + c^2 s\,\varphi(cs)\\
   & = c\Big[cs\,\varphi(cs)-2\big(1-\Phi(cs)\big)\Big] .
\end{align*}
In particular, a direct computation leads to
\begin{equation}\label{eq:determiningsm}
    \begin{split}
\partial_s\,\cE^{\infty}(m,s)
& =-4\sum_{j=0}^{m-1}\big[s\,g(y_j s)\big]' \\
& =4\sum_{j=0}^{m-1}\Big[2y_js\big(1-\Phi(y_js)\big)-\varphi(y_js)\Big]\, .
    \end{split}
\end{equation}
and
\begin{align*}
\cJ_{ss}(m,s)
& =-4\sum_{j=0}^{m-1}\big[s\,g(y_j s)\big]'' \\
& =-4\sum_{j=0}^{m-1}y_j\Big[y_js\,\varphi(y_js)-2\big(1-\Phi(y_js)\big)\Big]\, ,
\end{align*}
and, consequently,
\begin{align*}
\kappa(m,s) & =s^2\,\cJ_{ss}(m,s)
\\
& =4s\sum_{j=0}^{m-1}\Big[2y_js\big(1-\Phi(y_js)\big)-(y_js)^2\,\varphi(y_js)\Big]\, .
\end{align*}

The following algorithm is used for the numerical evaluation of the scale-invariant curvature $\kappa(M_B,s^*(M_B))$ reported in Figure~\ref{fig:kappaisdecr}.

\begin{algorithm}[ht]
\caption{Computation of $\kappa(M_B,s^*(M_B))$ for $B\in\{2,3,4,5,8,9\}$. The functions \texttt{E\_s} and \texttt{kappa} implement the finite sums of Section~\ref{sec:compofkappa}; \texttt{s\_star} locates the unique zero of $\cE^{\infty}_s(m,\cdot)$ by Brent's method on the bracket $[\sqrt2/(m+\tfrac12),20]$.}
\label{alg:kappa}
\begin{lstlisting}
import numpy as np
from scipy.stats import norm
from scipy.optimize import brentq

phi, Q = norm.pdf, norm.sf   # Q = 1 - Phi

def E_s(m, s):
    t = (np.arange(m) + 0.5) * s
    return -4 * np.sum(phi(t) - 2 * t * Q(t))

def kappa(m, s):
    t = (np.arange(m) + 0.5) * s
    return 4 * s * np.sum(2 * t * Q(t) - t**2 * phi(t))

def s_star(m):
    lo, hi = np.sqrt(2) / (m + 0.5), 20.0
    return brentq(lambda s: E_s(m, s), lo, hi, xtol=1e-12)

bits = [2, 3, 4, 5, 8, 9]
for B in bits:
    m = 2 ** (B - 1) - 1
    s = s_star(m)
    print(B, m, s, kappa(m, s))
\end{lstlisting}
\end{algorithm}

The following figure shows the scale-invariant curvature $\kappa$ of the limiting objective $s\mapsto\mathcal{E}^\infty(M_B,s)$, evaluated at its minimizer $s^*(M_B)$, as a function of the bit-width $B$. In particular, we note that $\kappa$ decreases monotonically with $B$, so the objective is sharply peaked around the optimal scale at low precision and increasingly flat as $B$ grows, which is why the choice of scale matters most at W2 and W3.
\begin{figure}[ht!]
    \centering
    \includegraphics[width=0.5\linewidth]{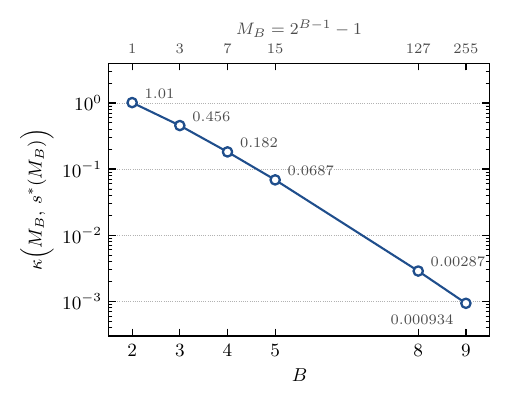}
    \caption{Scale-invariant curvature $\kappa(M_B, s^*(M_B))$ of $\mathcal{E}^\infty(M_B, s^*(M_B))$. $\kappa$ is decreasing as $B$ increases.}
    \label{fig:kappaisdecr}
\end{figure}

\subsubsection{Why relative curvature $\kappa$?}
\label{rem:scale_sense}
The scale-invariant curvature $\kappa$ is a natural quantity
for studying sensitivity to relative scale errors. To illustrate
this, fix the bit-width and let
$\vartheta \sim \mathcal N(0,I_d)$ and
$\vartheta_\sigma=\sigma\vartheta$, with $\sigma>0$. Define
the variance-normalized loss
\[
\mathcal E_\sigma(s)
:=
\frac{1}{d\sigma^2}
\mathbb E\!\left[
\left\|
\vartheta_\sigma
-s\roundB{\vartheta_\sigma/s}
\right\|_2^2
\right].
\]
Since $\vartheta_\sigma/s=\vartheta/(s/\sigma)$, we obtain
\[
\mathcal E_\sigma(s)=\mathcal E_1(s/\sigma),
\qquad
s_\sigma^*=\sigma s^*,
\]
where $s^*$ minimizes $\mathcal E_1$. Consequently, the same
relative scale error $\varepsilon$ produces the same
variance-normalized loss:
\[
\mathcal E_\sigma\!\left(s_\sigma^*(1+\varepsilon)\right)
=
\mathcal E_1\!\left(s^*(1+\varepsilon)\right).
\]
Although the ordinary second derivative depends on $\sigma$,
\[
\mathcal E_\sigma''(s_\sigma^*)
=
\sigma^{-2}\mathcal E_1''(s^*),
\]
the scale-invariant curvature does not:
\[
\kappa
=
(s_\sigma^*)^2\mathcal E_\sigma''(s_\sigma^*)
=
(s^*)^2\mathcal E_1''(s^*).
\]
It therefore governs the local loss increase under relative
scale perturbations:
\[
\mathcal E_\sigma\!\left(s_\sigma^*(1+\varepsilon)\right)
-\mathcal E_\sigma(s_\sigma^*)
=
\frac{\kappa}{2}\varepsilon^2+o(\varepsilon^2).
\]
This invariance motivates using relative perturbations to
assess sensitivity to scale choices, including those produced
by the rules in Table~\ref{tab:scale-selection}.

In practice, the scale-selection rules in
Table~\ref{tab:scale-selection} return different positive
scales for the same quantization setting. Their discrepancies
can be expressed as relative deviations from a common
reference scale:
\[
s_{\mathrm{rule}}
=
s_{\mathrm{ref}}(1+\varepsilon_{\mathrm{rule}}).
\]
Relative perturbations therefore provide a common way to
assess the effect of these differences across weight
magnitudes. When the reference scale minimizes the smooth
normalized loss, $\kappa$ governs the resulting loss increase up
to second order.

\subsection{Technical Lemmas}
The following elementary fact allows us to control the supremum of a quadratic polynomial over an interval by its values at three points. 
\begin{lemma}\label{lem:threepointLagr}
Let $a<b$, $m:=\frac{a+b}{2}$, and let $p$ be a real polynomial of degree at
most $2$. Then
\[
\sup_{s\in[a,b]}|p(s)|\ \le\ 3\max\big\{|p(a)|,\,|p(m)|,\,|p(b)|\big\}.
\]
\end{lemma}

\begin{proof}
Define the Lagrange basis polynomials associated with the nodes $a,m,b$,
\[
L_a(s):=\frac{(s-m)(s-b)}{(a-m)(a-b)},\qquad
L_m(s):=\frac{(s-a)(s-b)}{(m-a)(m-b)},\qquad
L_b(s):=\frac{(s-a)(s-m)}{(b-a)(b-m)} .
\]
Each has degree $2$, equals $1$ at its own node and $0$ at the other two.
Hence $q:=p(a)L_a+p(m)L_m+p(b)L_b$ is a polynomial of degree at most $2$
that agrees with $p$ at $a$, $m$ and $b$; the difference $p-q$ has degree at
most $2$ and three distinct roots, so $p=q$ identically, i.e.
\begin{equation}\label{eq:lagrange}
p(s)=p(a)\,L_a(s)+p(m)\,L_m(s)+p(b)\,L_b(s)\qquad\text{for all }s\in\R .
\end{equation}
We now bound the basis polynomials on $[a,b]$. The affine change of variable
$s=m+\frac{b-a}{2}\,x$ maps $[a,b]$ onto $x\in[-1,1]$ and the nodes $a,m,b$
to $-1,0,1$, and transforms the basis polynomials into
\[
L_a=\frac{x(x-1)}{2},\qquad L_m=1-x^2,\qquad L_b=\frac{x(x+1)}{2}.
\]
On $[-1,1]$ the function $x\mapsto x(x-1)$ attains its minimum $-\frac14$ at
$x=\frac12$ and its maximum $2$ at $x=-1$, so $|L_a|\le1$; clearly
$L_m\in[0,1]$; and $|L_b|\le1$ by the symmetry $x\mapsto-x$. Therefore, by
\eqref{eq:lagrange} and the triangle inequality, for every $s\in[a,b]$,
\begin{align*}
|p(s)|& \le|p(a)|\,|L_a(s)|+|p(m)|\,|L_m(s)|+|p(b)|\,|L_b(s)|\\
&\le|p(a)|+|p(m)|+|p(b)| \\
& \le3\max\big\{|p(a)|,|p(m)|,|p(b)|\big\},    
\end{align*}

and the claim follows by taking the supremum over $s\in[a,b]$.
\end{proof}
We will control operator norms through the following standard $\varepsilon$-net
argument, which reduces the supremum over the unit spheres to a maximum over a
finite set of cardinality exponential in the dimension
(see \cite[Sec.~4.4]{vershynin2018}).
\begin{lemma}\label{lem:net}
Let $p,q\ge1$. There exist $\tfrac14$-nets $\cU\subset\Sph{p-1}$ and $\cV\subset\Sph{q-1}$ (with respect to the Euclidean norm) with $|\cU|\le 9^{p}$ and $|\cV|\le 9^{q}$, and for every $M\in\R^{p\times q}$
\[
\opn{M}\ \le\ 2\max_{(u,v)\in\cU\times\cV}|u^\top M v| .
\]
\end{lemma}

\begin{proof}
The cardinality bound is the standard volumetric estimate $\mathcal N(\Sph{n-1},\varepsilon)\le(1+2/\varepsilon)^n$ with $\varepsilon=\tfrac14$.
Let $u\in\Sph{p-1}$, $v\in\Sph{q-1}$ and pick $u_0\in\cU$, $v_0\in\cV$ with $\|u-u_0\|_2\le\frac14$, $\|v-v_0\|_2\le\frac14$. Since
$u^\top Mv-u_0^\top Mv_0=(u-u_0)^\top Mv+u_0^\top M(v-v_0)$,
\[
|u^\top M v|\ \le\ |u_0^\top M v_0|+\big(\|u-u_0\|_2+\|v-v_0\|_2\big)\opn{M}
\ \le\ \max_{(a,b)\in\cU\times\cV}|a^\top M b|+\tfrac12\opn{M}.
\]
Taking the supremum over $u,v$ and rearranging (all quantities are finite) gives the claim.
\end{proof}
Our analysis of the loss landscape of $\cE^{\infty}$ repeatedly requires sharp two-sided bounds on the Gaussian tail $1-\Phi$, which are most conveniently expressed through the inverse Mills ratio $\lambda:=\varphi/(1-\Phi)$. The lemma below collects the elementary properties of $\lambda$ that we shall need in the following lemma.
\begin{lemma}\label{lem:millratio}
    Let $\lambda(t) := \varphi(t)/\big(1-\Phi(t)\big)$ for $t\ge0$. Then
    \begin{enumerate}[label=(\roman*)]
        \item $\lambda\in\mathcal{C}^\infty([0,\infty))$ and $\lambda(0) = 2\varphi(0) = \sqrt{2/\pi}$;
        \item $\lambda'(t) = \lambda(t)\big(\lambda(t) - t\big)$ for all $t\ge0$;
        \item $t < \lambda(t) < t + \dfrac{1}{t}$ for all $t>0$; equivalently, $\dfrac{t}{1+t^2}\,\varphi(t) < 1-\Phi(t) < \dfrac{\varphi(t)}{t}$.
    \end{enumerate}
\end{lemma}
\begin{proof}
    (i) is clear since $1-\Phi>0$ is smooth. For (ii), using $\varphi'(t) = -t\varphi(t)$ and $(1-\Phi)' = -\varphi$,
    \begin{align*}
        \lambda'(t) = \frac{-t\varphi(t)\big(1-\Phi(t)\big) + \varphi(t)^2}{\big(1-\Phi(t)\big)^2} = -t\,\lambda(t) + \lambda(t)^2 \, .
    \end{align*}
    For (iii), the upper bound on the tail follows from $1-\Phi(t) = \int_t^\infty\varphi(u)\,du < \int_t^\infty \frac{u}{t}\varphi(u)\,du = \varphi(t)/t$. For the lower bound let $h(t) := 1-\Phi(t) - \frac{t}{1+t^2}\varphi(t)$; then $h(t)\to0$ as $t\to\infty$ and a direct computation gives $h'(t) = -\frac{2\varphi(t)}{(1+t^2)^2} < 0$, so $h>0$ on $(0,\infty)$. Rearranging the two tail bounds gives $t<\lambda(t)<t+1/t$.
\end{proof}

\section{Experiments}\label{app:experiments}
In this section, we recall some quantization algorithm and preprocessing that are used for the experiments. Moreover, we report additional experiments that assess what our theory captures.

\subsection{Experimental Details}
\label{app:experimental-details}

This appendix specifies the setup of all experiments in
Section~\ref{sec:experiments} and Appendix~\ref{app:experiments}. The released code
(see the Reproducibility Statement) contains a frozen configuration file for every run,
and its \texttt{REPRODUCE.md} gives the commands for the five experiment families
below. Table~\ref{tab:exp-map} lists where the results of each family are reported.
As in Section~\ref{sec:experiments}, W$B$ denotes weight-only quantization at $B$ bits.
\begin{table}[h]
\centering
\caption{Experiment families and their reported results.}
\label{tab:exp-map}
\small
\begin{tabular}{@{}p{0.12\linewidth}p{0.30\linewidth}p{0.49\linewidth}@{}}
\toprule
Section & Content & Reported in \\
\midrule

\ref{subsec:globalscale}
& Global scale sensitivity and ablations
& Figures~\ref{fig:spread},
  \ref{fig:analyticregret_HIP},
  \ref{fig:spread-hip},
  \ref{fig:spread-extensions},
  \ref{fig:spread-without-minmax},
  \ref{fig:global-scale-perturbations},
  and~\ref{fig:calibration-size} \\
\addlinespace

\ref{app:zero-shot}
& Scale selection and zero-shot accuracy
& Tables~\ref{tab:scale-choice-performance-raw},
  \ref{tab:scale-choice-performance_HIP},
  \ref{tab:zero-shot-w3},
  \ref{tab:zero-shot-w2},
  and~\ref{tab:zero-shot-w2-engines-hip} \\
\addlinespace

\ref{app:HIP}
& Preprocessing and effective rank
& Figure~\ref{fig:w2-preprocessing-depth}
  and Table~\ref{tab:w2-preprocessing-depth} \\
\addlinespace

\ref{app:quanterror}
& Gaussian landscapes and finite-width convergence
& Figures~\ref{fig:gaussianlandscapes}
  and~\ref{fig:uniformerror} \\
\addlinespace

\ref{app:localreconstr}
& Local reconstruction sensitivity and empirical landscapes
& Figures~\ref{fig:local-gptq-sensitivity},
  \ref{fig:localreconstrHIP},
  \ref{fig:localreconstrHIPvdnoHIP},
  and~\ref{fig:landscape} \\

\bottomrule
\end{tabular}
\end{table}

\subsubsection{Models}
\label{app:models}

We quantize five publicly available decoder-only models (Table~\ref{tab:models}). Each is
loaded in fp16 at a pinned Hugging Face revision with a context length of 2{,}048 tokens.
OPT-125M is converted locally from the pinned PyTorch checkpoint to safetensors. Every
linear layer inside the transformer blocks is quantized: the attention projections
\texttt{q}, \texttt{k}, \texttt{v}, \texttt{o}, and the MLP projections (\texttt{gate},
\texttt{up} and \texttt{down} for Llama and Qwen; \texttt{fc1} and \texttt{fc2} for OPT). The quantization of each of this module exactly matches the quantization of $\vartheta^{(\ell)}$ in the sense of
Section~\ref{sec:notandpre}, calibrated on its own input activations.
The token embeddings and the language-model head stay in fp16. The short names in
Table~\ref{tab:models} are those used in the figures.

\begin{table}[h]
\centering
\caption{Models, short names used in the figures, and pinned checkpoint revisions.}
\label{tab:models}
\small
\resizebox{\linewidth}{!}{%
\begin{tabular}{@{}llll@{}}
\toprule
Model & Short name & Hugging Face repository & Revision \\
\midrule
OPT-125M & OPT-125M & \texttt{facebook/opt-125m} & \texttt{27dcfa74d334} \\
Llama-3.2-1B-Instruct & Llama-1B & \texttt{meta-llama/Llama-3.2-1B-Instruct} & \texttt{9213176726f5} \\
Llama-3.2-3B-Instruct & Llama-3B & \texttt{meta-llama/Llama-3.2-3B-Instruct} & \texttt{0cb88a4f764b} \\
Llama-3.1-8B-Instruct & Llama-8B & \texttt{meta-llama/Llama-3.1-8B-Instruct} & \texttt{0e9e39f249a1} \\
Qwen3-8B & Qwen3-8B & \texttt{Qwen/Qwen3-8B} & \texttt{b968826d9c46} \\
\bottomrule
\end{tabular}%
}
\end{table}

\subsubsection{Quantization setup}
\label{app:quantization}

\paragraph{Grid.}
All experiments are weight-only; activations remain in fp16. A W$B$ weight is quantized
with the round-to-nearest map $\Pi_{s,\MB}$ of \eqref{def:roundB}, i.e.\ onto the
symmetric uniform grid $s\{-\MB,\dots,\MB\}$ with $\MB = 2^{B-1}-1$ levels per side and
$2^B-1$ levels in total, with clipping to $\pm \MB s$. At $B=2$ this is the ternary grid
$\{-s,0,s\}$. We use $B\in\{2,3,4,6,8\}$. One fp16 scale is stored per quantization
group, and there are no zero-points. Per-channel W3 on Llama-3.1-8B therefore costs
$3.003$ bits per weight.

\paragraph{Granularity.}
\emph{Per-channel} quantization uses one scale per output row, which gives the
row-wise problems studied in Section~\ref{others}. \emph{Grouped} quantization uses one
scale per contiguous group of $g\in\{512,256,128,64\}$ input columns. With activation
ordering, the groups are formed after the column permutation.

\paragraph{GPTQ.}
Layers are quantized sequentially, block by block. Within a block the order is
$\{\texttt{q},\texttt{k},\texttt{v}\}\to \texttt{o}\to\{\texttt{up},\texttt{gate}\}\to\texttt{down}$.
Each block is calibrated on the outputs of the already-quantized preceding blocks
(``carry'' propagation), so the activations $X^{\ell-1}$ in \eqref{eq:PTQprob} come
from the partially quantized network. The Hessian proxy is
\[
  H=\frac{2}{N_c}\sum_{n=1}^{N_c} X_n X_n^\top ,
\]
where $X_n\in\R^{d\times 2048}$ holds the layer inputs for calibration window $n$. Up
to the constant factor, which cancels in the normalized loss of Section~\ref{others} and
in every scale rule below, this is the Gram matrix $XX^\top$ of the theory. $H$ is
accumulated in fp32 over $N_c=128$ windows with a forward batch of 4 windows. Before
inverting $H$ we add $\lambda\Id_d$ with
$\lambda = 0.01\cdot\operatorname{mean}(\operatorname{diag}H)$. Columns are processed in
order of decreasing $H_{ii}$ (activation ordering) in lazy batches of 128 columns, using
the inverse-Cholesky update of the reference implementation~\citep{frantar2022gptq}.
Columns with $H_{ii}=0$ are set to zero. For grouped quantization the scale of a group
is selected on the error-compensated weights at the point where the algorithm enters
the group.

\paragraph{Other engines.}

\emph{ResComp}~\citep{li2026rethinking} runs GPTAQ-style~\citep{li2025gptaq} paired
compensation plus a residual correction toward the full-precision weights, with
$\alpha_{\text{GPTAQ}}=\alpha_{\text{ResComp}}=0.25$. It uses the stability mode of the
reference implementation: \texttt{org} at W2 and \texttt{allw} otherwise.
\emph{QRoNoS}~\citep{zhang2026QRoNoS} re-solves the full layer at the first column and
uses ``reset'' propagation: each block is calibrated on full-precision inputs. Both
engines use the damping, activation ordering and block size above; for grouped rows the
compensation block is $\min(128,g)$. Because these engines use paired inputs, they
additionally admit the rule RTNH-cross, which searches $s$ over the grid of
Appendix~\ref{app:scale-rules} to minimize
$\|W X_{\mathrm{fp}} - \Pi_{s,\MB}(W)\, X_q\|_F^2$, where $X_{\mathrm{fp}}$ and $X_q$
are the full-precision and quantized-prefix inputs. We run it for $B\le 3$ only.

\paragraph{Hadamard incoherence processing (HIP).}
HIP (Appendix~\ref{app:HIP}) is applied independently to the input dimension of every
matrix, $W\mapsto WR$ with $R=\operatorname{diag}(\xi)\,Q$, where $Q$ is a normalized
Hadamard matrix ($Q^\top Q=\Id_d$) from the QuIP\# implementation~\citep{tseng2024quip}
at a pinned commit, and $\xi\in\{-1,1\}^d$ are Rademacher signs drawn from a seed and the
matrix name. The Hessian is transformed as $H\mapsto R^\top H R$, which leaves its
spectrum, and hence $\reff(H)$, unchanged (Appendix~\ref{app:HIP}). At inference the
layer input is rotated online by $R^\top$, so the network function is unchanged before
quantization. Unless stated otherwise the rotation seed is 0.

\subsubsection{Scale-selection rules}
\label{app:scale-rules}

Let $w$ denote the (error-compensated) weights of one row or group, and let
$s^e = \max_q |w_q|/\MB$ be its min-max scale (Section~\ref{sec:notandpre}). Every rule
outputs one scale $s$ per row or group (Table~\ref{tab:scale-selection}). Min-max and
Analytic require no search; Shrink-2.4 and WMSE search without activation data;
Proxy-static and RTNH use the Hessian.

\begin{table}[h]
\centering
\caption{Scale-selection rules. $e(s)=\Pi_{s,\MB}(w)-w$ is the rounding error at
scale $s$, and $H_{gg}$ is the undamped Hessian block of the row or group. The search
window for $\eta$ depends on the experiment (Appendix~\ref{app:scale-rules}).}
\label{tab:scale-selection}
\small
\begin{tabular}{@{}lll@{}}
\toprule
Rule & Scale & Objective \\
\midrule
Min-max & $s^e$ & none \\
Shrink-2.4 & $p\,s^e$, $p\in\{1.00,0.99,\dots,0.21\}$ & $\sum_q |e_q(p\, s^e)|^{2.4}$ \\
WMSE & $s^e e^{\eta}$, grid search & $\sum_q e_q(s)^2$ \\
Proxy-static (ours) & $s^e e^{\eta}$, grid search & $\sum_q e_q(s)^2/U_{qq}^2$ (GPTQ step loss) \\
RTNH & $s^e e^{\eta}$, grid search & $e(s)^\top H_{gg}\, e(s)$ \\
\midrule
Analytic & $s^*(\MB)\,\hat\sigma$ & none; $\hat\sigma$ is the RMS of $w$ \\
\bottomrule
\end{tabular}
\end{table}

\paragraph{Min-max and Shrink-2.4.}
Min-max is the symmetric min-max choice $s^e$ of Section~\ref{sec:notandpre}.
Shrink-2.4 is the clipping search of the GPTQ reference quantizer, applied to our
symmetric $(2^B-1)$-level grid: it scores the 80 shrink fractions $p_i = 1-i/100$,
$i=0,\dots,79$, of the min-max range by the $\ell^{2.4}$ rounding error.

\paragraph{Analytic scale.}
$s^*(\MB)$ is the unique minimizer of the limiting objective $\cJ(\MB,\cdot)$
(Proposition~\ref{prop:scalarlandscape}), obtained as the zero of
$\partial_s\cJ(\MB,\cdot)$ in the closed form \eqref{eq:determiningsm} by bisection.
This gives $s^*(\MB) = 1.2240,\ 0.6508,\ 0.3534,\ 0.1055,\ 0.0309$ for $B=2,3,4,6,8$
($\MB=1,3,7,31,127$). $\hat\sigma$ is computed on the current compensated weights and
excludes input columns with $H_{ii}=0$.
\paragraph{Proxy-static.}
\label{app:proxy-static}
Let $U$ be the upper Cholesky factor of the inverse
damped Hessian in activation order, and let
$e_q(s)=\Pi_{s,M_B}(w_q)-w_q$.
For each row or group, Proxy-static selects
\[
    s_{\mathrm{PS}}
    \in\operatorname*{arg\,min}_{s\in\mathcal S}
    \sum_q \frac{e_q(s)^2}{U_{qq}^2},
\]
where $\mathcal S$ is the candidate grid described below. Candidate scores use a fixed snapshot of the weights at group entry, including compensation from previous groups. The objective uses GPTQ's per-column loss weights
without simulating candidate-dependent error feedback.
After scale selection, quantization proceeds with the
usual GPTQ updates.
Algorithm~\ref{alg:proxy-static} summarizes the procedure.
We motivate proxy-static by analysing the GPTQ. 
Our selector is motivated by the loss incurred at each
GPTQ quantization step.
Let $A$ denote the positive-definite, damped Gram matrix
in activation order, and consider a quadratic loss 
$\mathcal L(z)=z^\top A z$.
At step $q$, let $F_q=\{q,\ldots,d\}$ denote the remaining
free coordinates, and let $w^{(q)}$ be their current
compensated weights. Fixing coordinate $q$ to its quantized
value introduces the residual
\[
    \epsilon_q(s)
    =w_q^{(q)}-Q_s(w_q^{(q)}).
\]
The optimal compensating update solves
\[
    \min_{\Delta\in\mathbb R^{|F_q|}}
    \Delta^\top A_{F_qF_q}\Delta
    \quad\text{subject to}\quad
    \Delta_q=-\epsilon_q(s).
\]
where $A_{F_qF_q}$ is the principal
submatrix of $A$ indexed by $F_q$. Writing $e_q$ for the corresponding coordinate vector,
the solution and minimum are
\[
    \Delta^\star
    =-\frac{\epsilon_q(s)}
    {[A_{F_qF_q}^{-1}]_{qq}}
    A_{F_qF_q}^{-1}e_q,
    \qquad
    \Delta\mathcal L_q(s)
    =\frac{\epsilon_q(s)^2}
    {[A_{F_qF_q}^{-1}]_{qq}}.
\]
Here the increment is measured from the conditional
quadratic minimizer given the previously fixed coordinates.

If $A^{-1}=U^\top U$ with $U$ upper triangular, then
\[
    [A_{F_qF_q}^{-1}]_{qq}=U_{qq}^2.
\]
Thus GPTQ's step loss is $\epsilon_q(s)^2/U_{qq}^2$.
Evaluating these losses for every candidate scale would
require recomputing the candidate-dependent compensation
trajectory.

Proxy-static instead freezes the weights at group entry.
For a row or group $g$ with snapshot $v$, it selects
\[
    s_{\mathrm{PS}}
    \in\operatorname*{arg\,min}_{s\in\mathcal S}
    \sum_{q\in g}
    \frac{(v_q-Q_s(v_q))^2}{U_{qq}^2}.
\]
This retains GPTQ's per-column loss weights while
approximating the sequential residuals by static rounding
residuals. The factor $U$ is already available from GPTQ,
so candidate scoring requires no additional factorization
or candidate-specific compensation pass.
After scale selection, the usual GPTQ updates are applied.
Algorithm~\ref{alg:proxy-static} summarizes the procedure.
\begin{algorithm}[t]
\caption{GPTQ with Proxy-static scale selection}
\label{alg:proxy-static}
\begin{algorithmic}[1]
\Require Weights $W\in\mathbb{R}^{m\times d}$,
         input Gram $H$, bit-width $B\geq 2$,
         group width $G$, candidate count $C$,
         interval $[\eta_{\min},\eta_{\max}]$,
         damping fraction $\lambda$
\Ensure Quantized weights $\widehat W$ and group scales
        $\{s^\star_{rg}\}$

\State $K\gets 2^{B-1}-1$
\State Define $Q_s(x)=s\,\operatorname{clip}
       (\operatorname{round}(x/s),-K,K)$
\State $\mathcal D\gets\{j:H_{jj}=0\}$
\State $W_{:,\mathcal D}\gets 0$;
       $H_{jj}\gets 1$ for $j\in\mathcal D$
\State Let $P$ sort $\operatorname{diag}(H)$ in descending order
\State $\widetilde W\gets WP^\top$;
       $\widetilde H\gets PHP^\top$
\State $\widetilde H\gets\widetilde H+
       \lambda\,\operatorname{mean}
       (\operatorname{diag}(\widetilde H))I$
\State Compute upper triangular $U$ satisfying
       $\widetilde H^{-1}=U^\top U$
\State $(\eta_k)_{k=1}^{C}\gets
       \operatorname{linspace}(\eta_{\min},\eta_{\max},C)$
\State Replace the $\eta_k$ closest to zero by $0$

\For{each contiguous group $g$ in activation order}
    \State $V\gets\operatorname{copy}(\widetilde W_{:,g})$
           \Comment{freeze compensated group-entry weights}
    \For{each output row $r$}
        \State $a\gets\max_{j\in g}|V_{rj}|$
        \If{$a=0$}
            \State $s^\star_{rg}\gets 1$
                   \Comment{any positive scale quantizes zero exactly}
        \Else
            \State $s^{(0)}_{rg}\gets a/K$
            \For{$k=1,\ldots,C$}
                \State $s_{rgk}\gets s^{(0)}_{rg}\exp(\eta_k)$
                \State $\displaystyle
                L_{rgk}\gets\sum_{j\in g}
                \frac{(V_{rj}-Q_{s_{rgk}}(V_{rj}))^2}
                     {U_{jj}^2}$
            \EndFor
            \State Select $k^\star$ minimizing $L_{rgk}$,
                   breaking numerical ties toward smallest $|\eta_k|$
            \State $s^\star_{rg}\gets s_{rgk^\star}$
        \EndIf
    \EndFor

    \For{columns $j\in g$ in activation order}
        \State $(q_j)_r\gets
               Q_{s^\star_{rg}}(\widetilde W_{rj})$
               for every row $r$
        \State $e_j\gets(\widetilde W_{:,j}-q_j)/U_{jj}$
        \State $\widetilde W_{:,j:d}\gets
               \widetilde W_{:,j:d}-e_jU_{j,j:d}$
        \State $\widehat W^{\mathrm{ord}}_{:,j}\gets q_j$
    \EndFor
\EndFor
\State $\widehat W\gets\widehat W^{\mathrm{ord}}P$
\State \Return $\widehat W$, $\{s^\star_{rg}\}$,
       and the group assignment induced by $P$
\end{algorithmic}
\end{algorithm}

\subsubsection{Compute resources and software}
\label{app:compute}

Every run uses a single NVIDIA H100 (94\,GB) or H200 (141\,GB) GPU.

\subsection{Global scale-sensitivity}
\label{subsec:globalscale}

In this section we add further experiments to understand how different scale-selection rules impact the performance of PTQ of LLMS. This section presents additional experiments examining how scale-selection rules affect the performance of quantized LLMs.

\subsubsection{Sensitivity to scale selection with HIP preprocessing}

\begin{figure}[h!]
    \centering
    \includegraphics[width=0.8\textwidth]{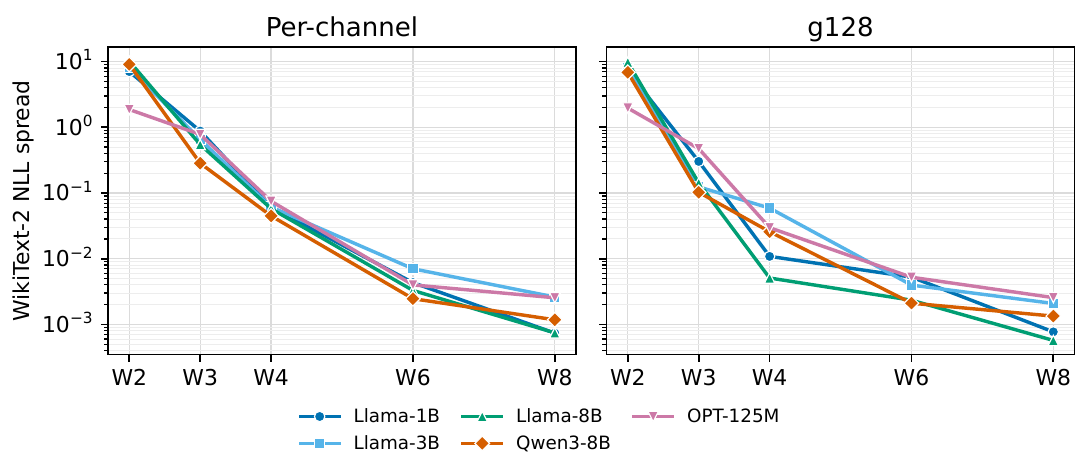}
    \caption{WikiText-2 NLL spread across scale-selection rules for GPTQ with HIP preprocessing, plotted against bit-width for different models. Left: per-channel quantization. Right: group size 128. The spread is the difference between the maximum and minimum NLL across Min-max, Shrink-2.4, WMSE, Proxy-static, and RTNH. The vertical axis uses a logarithmic scale.}
    \label{fig:spread-hip}
\end{figure}

HIP preprocessing appears to produce a steeper decrease in NLL spread with bit-width, qualitatively similar to the effect of reducing the group size to 128 in GPTQ without HIP. However, as discussed in Remark~\ref{rem:ruleisnotcurvature}, this observation does not directly imply lower curvature of the GPTQ loss landscape. The spread also depends on how far apart the scales selected by the different rules are. HIP may increase agreement among these choices: suppressing extreme weight coordinates can reduce the influence of outliers on min-max scaling, while a more Gaussian weight distribution may improve agreement between WMSE-based scales and the analytic prediction. This provides a possible additional explanation for the reduced variation in end-to-end performance.

\subsubsection{Sensitivity to scale selection with GPTQ extensions}
To assess whether the trend in Figure~\ref{fig:spread} extends beyond GPTQ, we repeat the scale-selection comparison with ResComp~\citep{li2026rethinking} and QRoNoS~\citep{zhang2026QRoNoS}. We evaluate both methods at W2, W4, and W8 on Llama-3.2-1B and Llama-3.1-8B, with and without HIP preprocessing.

\begin{figure}[h!]
    \centering
    \includegraphics[width=0.8\textwidth]{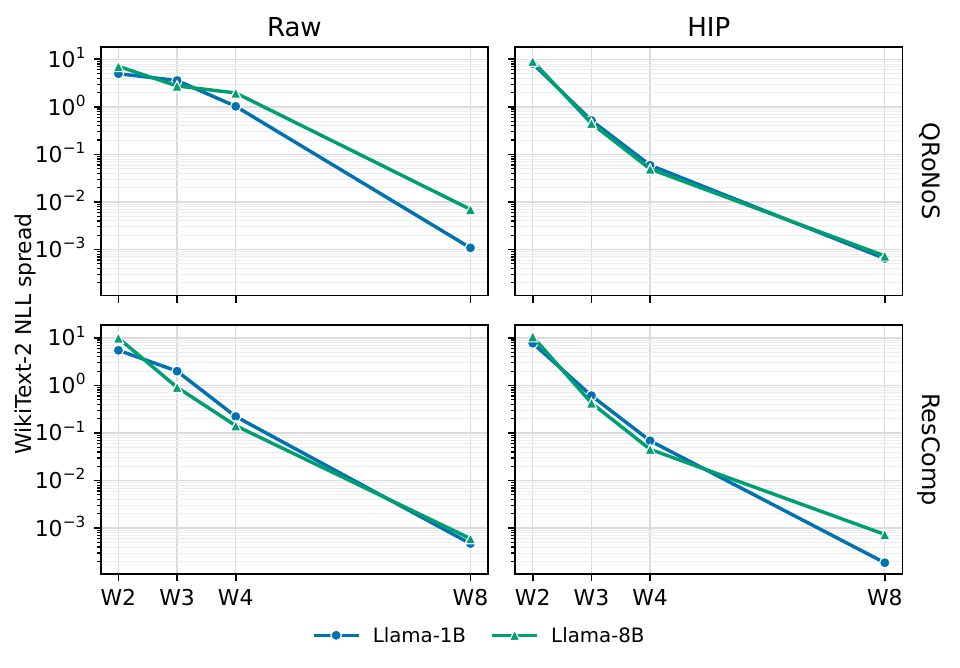}
    \caption{WikiText-2 NLL spread across scale-selection rules for ResComp and QRoNoS, with and without HIP preprocessing, as a function of bit-width. The spread is the difference between the maximum and minimum NLL across the tested scale-selection rules.}
    \label{fig:spread-extensions}
\end{figure}

Both methods exhibit a decrease in NLL spread with increasing bit-width, qualitatively matching the trend observed for GPTQ. HIP preprocessing is associated with an even steeper decrease. These results suggest that the reduced sensitivity to scale selection at higher precision extends beyond GPTQ to related PTQ methods in the tested settings.

\subsubsection{Scale sensitivity without min-max}

To understand whether the bit-width dependence of NLL spread is mainly due to the poor performance of min-max at low precision, we repeat the comparison using only the search based rules Shrink-2.4, WMSE, Proxy-static, and RTNH.

\begin{figure}[h!]
    \centering
    \includegraphics[width=0.8\linewidth]{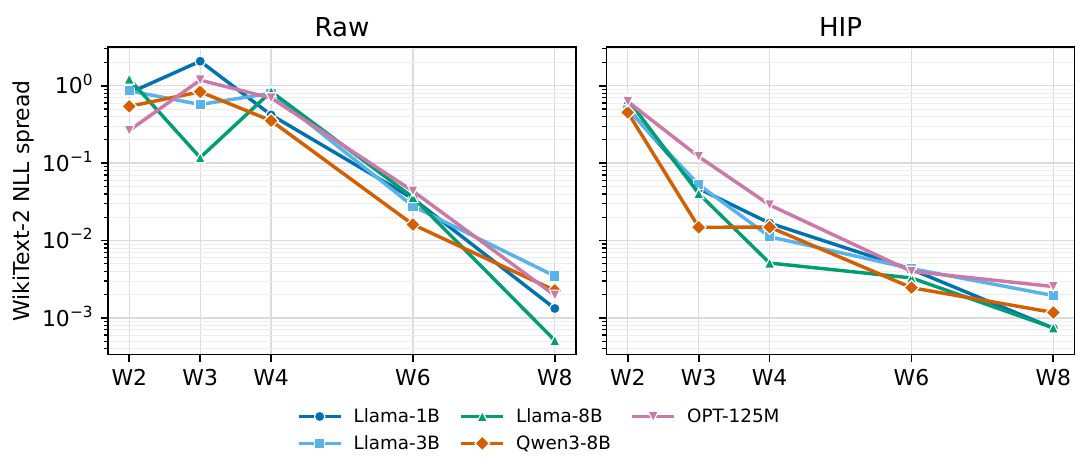}
    \caption{WikiText-2 NLL spread across Shrink-2.4, WMSE, Proxy-static, and RTNH for per-channel GPTQ without HIP (left) and with HIP (right). The spread is the maximum minus minimum NLL across these four rules; min-max is excluded. The vertical axis uses a logarithmic scale.}
    \label{fig:spread-without-minmax}
\end{figure}

Figure~\ref{fig:spread-without-minmax} shows that the broad reduction in NLL spread with increasing precision persists after excluding min-max. Without HIP, the dependence is nonmonotonic between W2 and W4, followed by a pronounced decrease at higher precisions. With HIP, the decrease is more consistent across the tested bit-widths, with a particularly large reduction between W2 and W3. Thus, the greater practical importance of scale selection at low precision persists among the search-based rules and can not be explained by the bad performance of min-max. Performing a scale search alone does not ensure comparable end-to-end performance: the criterion used to select the scale remains important.

\subsubsection{End-to-end sensitivity to controlled scale perturbations}
\label{app:controlledsens}
As discussed in Remark \ref{rem:ruleisnotcurvature} the scale-sensitivity plots based on different scale rules do not account for the possibility that discrepancy of the scale rules themselves change with the precision. We therefore add a controlled global sensitivity analysis at a fixed relative perturbation amplitude. We perturb the scales returned by the WMSE search and measure the resulting change in end-to-end performance in NLL. For each bit-width $B$ and channel $c$, we set
\[
s_{B,c}^{(r,\pm)}
=
s_{B,c}^{0}(1\pm\delta\epsilon_{r,c}),
\qquad
\delta=0.05,
\]
where $\epsilon_{r,c}\in\{-1,+1\}$ are independent random signs. We use five random directions, matched across bit-widths within each model, and evaluate both signs of each direction. Each perturbed scale configuration is applied in a fresh sequential GPTQ run from the original floating-point weights, recomputing downstream activations, Gram matrices, and error compensation.

Let $\Delta_{B,r,\pm}$ denote the change in NLL relative to the corresponding unperturbed baseline. We summarize the magnitude of the response by
\[
C_B(\delta)
=
\frac{1}{2R}
\sum_{r=1}^{R}
\left(
|\Delta_{B,r,+}|+|\Delta_{B,r,-}|
\right),
\qquad R=5.
\]
This statistic captures both improvements and degradations without cancellation between opposite perturbation arms. We interpret $C_B(\delta)$ as finite-perturbation scale sensitivity,
not directly as curvature. Away from a stationary point its leading contribution
is the directional NLL gradient; near a stationary point, second-order
effects can dominate. The absolute values retain improvements as well
as degradations, so this statistic measures response magnitude rather
than the signed penalty from perturbing the scales.  We evaluate per-channel GPTQ without HIP at W2, W3, W4, W6, and W8 on five models, using one calibration draw and evaluating on WikiText-2 and C4.

Figure~\ref{fig:global-scale-perturbations} shows substantially larger responses at W2--W4 than at W6--W8. Across the Llama and Qwen models and the two evaluation corpora, the reported W4-to-W6 reduction is approximately $8$--$96$-fold. OPT-125M exhibits a weaker reduction and remains more sensitive at high precision. The dependence on bit-width is not monotonic: for example, Llama-3.1-8B is more sensitive at W4 than at W3 on WikiText-2.

These results further support greater practical sensitivity to relative scale perturbations at low precision in the tested setting. They complement the local reconstruction-loss analysis, while showing that its regular bit-width dependence need not neccessarily translate into a monotonic end-to-end response. The interpretation also depends on baseline quality: some W2 models are severely degraded, and several W4 baselines perform worse than their W3 counterparts. 

\begin{figure}[t]
    \centering
    \includegraphics[width=0.8\linewidth]{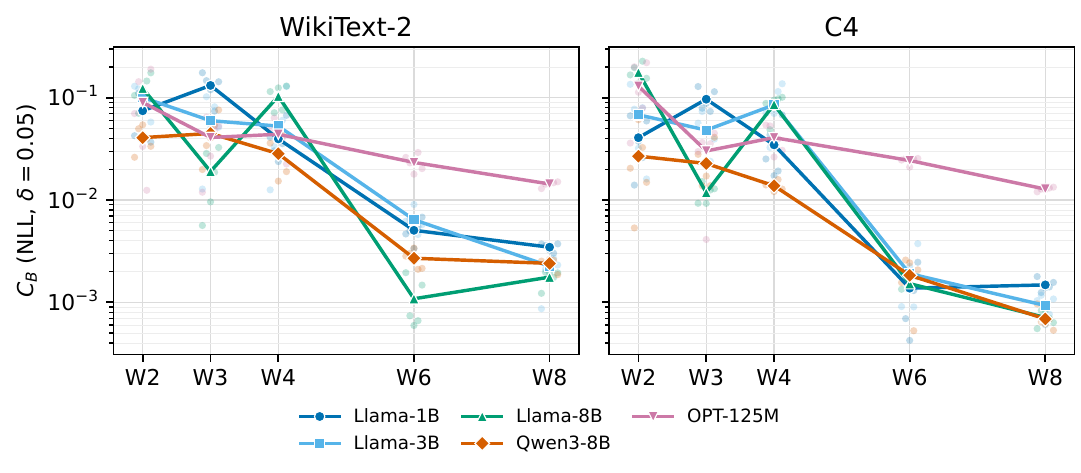}
    \caption{End-to-end sensitivity to controlled scale perturbations. Each channel's WMSE-selected scale is increased or decreased by $5\%$, followed by a complete sequential GPTQ run. Each point reports the mean absolute NLL change relative to the unperturbed WMSE baseline, averaged over ten perturbed runs corresponding to five random sign patterns and their opposites. Experiments use per-channel GPTQ and one calibration draw; random directions are matched across bit-widths within each model.}
    \label{fig:global-scale-perturbations}
\end{figure}

\subsubsection{Calibration-Set Size}

Our main experiments use 128 calibration sequences of 2048
tokens each. To assess whether our observations depend on this
choice, we vary the number of calibration sequences from 16 to
256 while keeping the sequence length fixed. We compare
Min-max, Analytic, and Proxy-static on Llama-1B and OPT-125M
at W2 and W3, with and without HIP.

Figure~\ref{fig:calibration-size} shows that calibration-set
size affects NLL, sometimes nonmonotonically, but substantial
differences between scale-selection rules persist across
the tested sizes. In particular, increasing the calibration
set does not consistently close the performance gaps between
rules. Thus, the practical importance of scale selection at
low precision is not specific to our default calibration-set
size.

\begin{figure}[htbp]
    \centering
    \includegraphics[width=\linewidth]
    {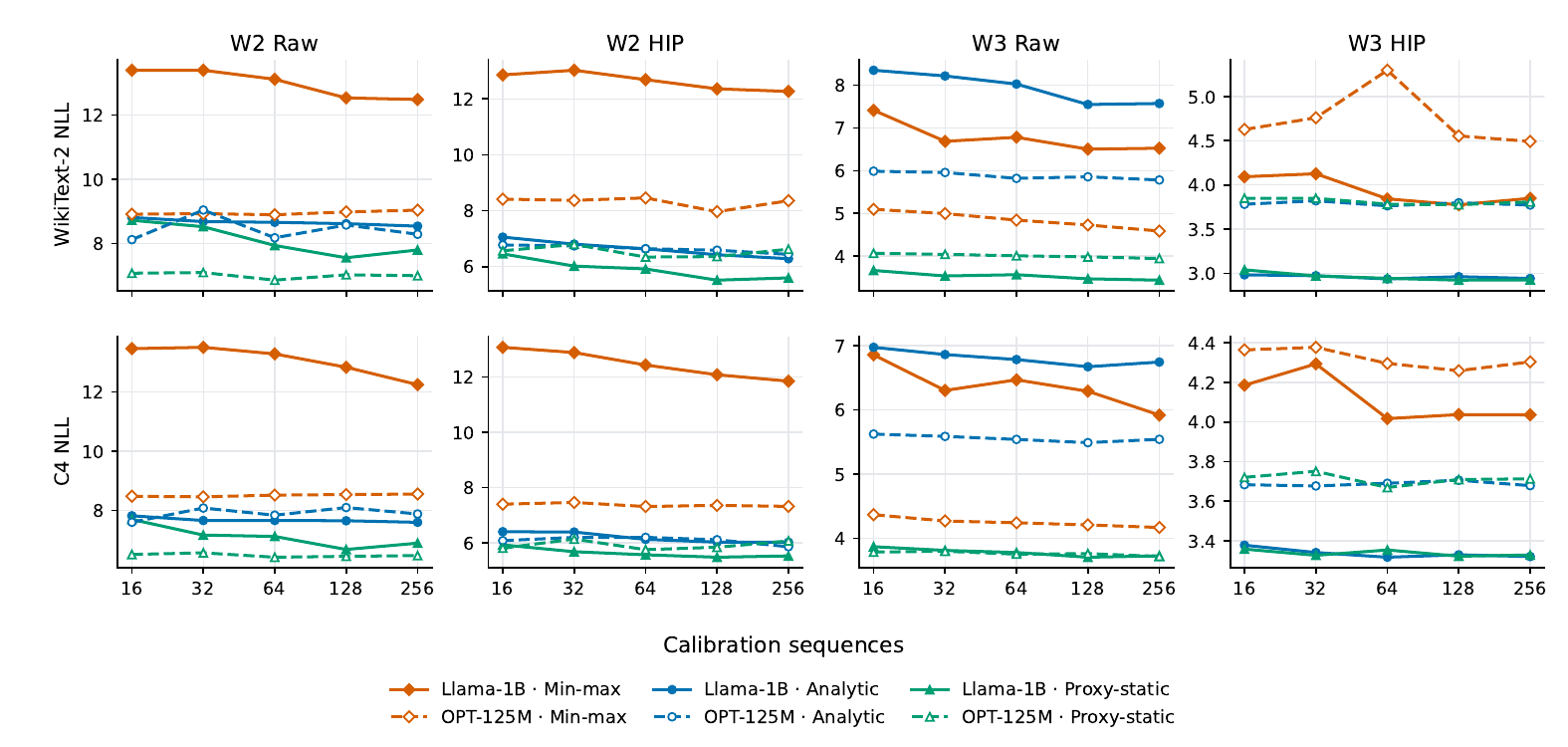}
    \caption{
    Effect of calibration-set size on WikiText-2 NLL (top)
    and C4 NLL (bottom). Each plot compares Min-max, Analytic,
    and Proxy-static for Llama-1B and OPT-125M at W2 and W3,
    with and without HIP. Each calibration sequence contains
    2048 tokens. 
    }
    \label{fig:calibration-size}
\end{figure}

\subsection{Zero-shot evaluation}
\label{app:zero-shot}

We complement perplexity evaluation with six zero-shot tasks:
PIQA, ARC-Easy, ARC-Challenge, HellaSwag, WinoGrande, and BoolQ.
All quantized results use per-channel weight-only quantization,
calibration seed 0, and 128 C4 sequences of 2048 tokens each.
The symmetric grid has $K_B=2^{B-1}-1$, so W2 is ternary.
\emph{Raw} denotes no preprocessing; HIP uses rotation seed 0.
We report \texttt{acc\_norm} for PIQA, ARC-Easy, ARC-Challenge,
and HellaSwag, and \texttt{acc} for WinoGrande and BoolQ.
Accuracies are percentages, while meanis computed as their unweighted average on the six tasks.
WikiText-2 and C4 columns report perplexity. Lower perplexity and
higher accuracy are better. Bold entries identify column-wise best
reported values within each matched model, precision, preprocessing,
engine, and campaign panel, including ties at the displayed precision.
The FP16 reference is not bolded.
These are single-calibration-seed measurements; differences of a few
tenths of an accuracy point do not establish resolved rankings.

\paragraph{W3 scale selection.}
Table~\ref{tab:zero-shot-w3} compares four scale-selection rules for GPTQ.
With HIP, the search-free analytic rule achieves mean accuracies of
72.08\%, 63.75\%, and 72.36\% on Llama-3.1-8B, Llama-3.2-3B,
and Qwen3-8B, respectively, comparable to the searched rules and
substantially above min-max. Without HIP, the analytic rule fails
severely on all three models. Proxy-static achieves the lowest
perplexity on both corpora and the highest six-task mean in each
raw W3 comparison. Thus, the relative performance of the selectors
depends strongly on preprocessing.

\paragraph{W2 scale selection.}
Table~\ref{tab:zero-shot-w2} shows substantial degradation in several
W2 configurations, especially without HIP and with min-max.
The proxy-static results use the search interval
$\eta\in[-1,1]$. Proxy-static achieves lower perplexity on both
corpora than Shrink-2.4 in all six model--preprocessing comparisons,
but higher six-task mean accuracy in only three. For example,
on Llama-3.1-8B with HIP, proxy-static achieves WikiText-2/C4
perplexity of 31.38/43.81, compared with 39.16/53.99 for Shrink-2.4,
while their mean accuracies are 48.50\% and 48.69\%, respectively.
Scale selection therefore affects both metrics, but their rankings
need not coincide.

\paragraph{Transfer across quantization engines.}
Table~\ref{tab:zero-shot-w2-engines-hip} compares Shrink-2.4
and proxy-static for GPTQ, ResComp, and QRoNoS at W2 with
HIP preprocessing. This focused campaign uses 128 log-spaced
proxy candidates over $\eta\in[-1.0,0.1]$. Its search settings differ from the main GPTQ campaign, so
overlapping configurations are reported separately.
Proxy-static lowers perplexity on both corpora in all six
model-engine comparisons and achieves higher mean accuracy
in five, although some accuracy differences are small.
We omit the corresponding results without HIP because all
tested configurations exhibit severe degradation.
These findings support proxy-static as a practical scale
selector beyond GPTQ, while showing that perplexity
improvements do not always translate into accuracy gains.

\begin{table*}[h!]
\centering
\caption{W3 per-channel GPTQ with a symmetric seven-level grid $\{ks:k=-3,\ldots,3\}$. The proxy-static search uses $\eta\in[-1,1]$. Panels (a) and (b) show raw and HIP results. FP16 references are excluded from bolding, which identifies the best quantized rule within each model and panel. Lower perplexity and higher accuracy are better.}
\label{tab:zero-shot-w3}
\begingroup
\setlength{\tabcolsep}{3pt}
\renewcommand{\arraystretch}{1.08}
\resizebox{\textwidth}{!}{%
\begin{tabular}{llrrrrrrrrr}
\toprule
Model & Rule & Wiki2 PPL & C4 PPL & PIQA & ARC-E & ARC-C & HellaSwag & WinoGrande & BoolQ & Mean \\
\midrule
\multicolumn{11}{l}{\textbf{(a) Without preprocessing}} \\
Llama-3.1-8B & FP16 reference & $7.213$ & $11.388$ & $80.96$ & $79.59$ & $55.12$ & $79.26$ & $74.03$ & $84.10$ & $75.51$ \\
\addlinespace[2pt]
 & Shrink-2.4 & $14.05$ & $17.18$ & $70.46$ & $55.30$ & $36.86$ & $70.76$ & $66.61$ & $79.48$ & $63.25$ \\
 & Proxy-static & $\mathbf{11.76}$ & $\mathbf{16.33}$ & $\mathbf{72.91}$ & $\mathbf{63.51}$ & $\mathbf{41.21}$ & $\mathbf{72.09}$ & $\mathbf{68.51}$ & $\mathbf{79.85}$ & $\mathbf{66.34}$ \\
 & Analytic & $9.3\times10^{5}$ & $6.43\times10^{5}$ & $51.36$ & $25.42$ & $26.79$ & $26.73$ & $50.75$ & $62.17$ & $40.54$ \\
 & Min-max & $41.47$ & $45.51$ & $62.35$ & $40.78$ & $28.24$ & $47.20$ & $51.54$ & $52.42$ & $47.09$ \\
\addlinespace
Llama-3.2-3B & FP16 reference & $11.048$ & $16.486$ & $75.52$ & $67.85$ & $46.16$ & $70.44$ & $67.32$ & $78.62$ & $67.65$ \\
\addlinespace[2pt]
 & Shrink-2.4 & $42.48$ & $47.93$ & $69.91$ & $53.75$ & $34.22$ & $59.46$ & $55.96$ & $66.85$ & $56.69$ \\
 & Proxy-static & $\mathbf{21.16}$ & $\mathbf{27.55}$ & $\mathbf{70.29}$ & $\mathbf{58.46}$ & $\mathbf{35.92}$ & $\mathbf{60.93}$ & $\mathbf{61.96}$ & $\mathbf{70.95}$ & $\mathbf{59.75}$ \\
 & Analytic & $2310.17$ & $1027.36$ & $51.52$ & $26.68$ & $25.68$ & $26.68$ & $52.96$ & $39.94$ & $37.24$ \\
 & Min-max & $120.61$ & $73.61$ & $56.20$ & $33.16$ & $26.11$ & $41.56$ & $52.72$ & $52.84$ & $43.77$ \\
\addlinespace
Qwen3-8B & FP16 reference & $9.715$ & $15.363$ & $77.75$ & $80.93$ & $56.48$ & $74.92$ & $67.64$ & $86.61$ & $74.05$ \\
\addlinespace[2pt]
 & Shrink-2.4 & $37.95$ & $33.39$ & $66.49$ & $43.39$ & $31.57$ & $50.36$ & $55.56$ & $75.23$ & $53.77$ \\
 & Proxy-static & $\mathbf{18.51}$ & $\mathbf{21.39}$ & $\mathbf{73.23}$ & $\mathbf{60.14}$ & $\mathbf{40.61}$ & $\mathbf{62.68}$ & $\mathbf{64.40}$ & $\mathbf{81.83}$ & $\mathbf{63.82}$ \\
 & Analytic & $3.13\times10^{10}$ & $5.99\times10^{10}$ & $52.12$ & $25.72$ & $27.22$ & $26.55$ & $49.17$ & $37.83$ & $36.43$ \\
 & Min-max & $22.39$ & $26.90$ & $63.22$ & $41.46$ & $26.54$ & $53.57$ & $52.64$ & $58.65$ & $49.35$ \\
\midrule
\multicolumn{11}{l}{\textbf{(b) HIP preprocessing}} \\
Llama-3.1-8B & FP16 reference & $7.213$ & $11.388$ & $80.96$ & $79.59$ & $55.12$ & $79.26$ & $74.03$ & $84.10$ & $75.51$ \\
\addlinespace[2pt]
 & Shrink-2.4 & $\mathbf{9.09}$ & $\mathbf{14.45}$ & $78.45$ & $74.41$ & $\mathbf{50.77}$ & $\mathbf{74.00}$ & $71.67$ & $81.74$ & $71.84$ \\
 & Proxy-static & $9.23$ & $14.62$ & $78.13$ & $75.67$ & $48.81$ & $73.87$ & $\mathbf{73.32}$ & $81.65$ & $71.91$ \\
 & Analytic & $9.21$ & $14.60$ & $\mathbf{78.56}$ & $\mathbf{76.14}$ & $50.60$ & $73.65$ & $71.11$ & $\mathbf{82.45}$ & $\mathbf{72.08}$ \\
 & Min-max & $15.40$ & $22.77$ & $70.46$ & $54.38$ & $33.19$ & $62.02$ & $59.51$ & $62.11$ & $56.94$ \\
\addlinespace
Llama-3.2-3B & FP16 reference & $11.048$ & $16.486$ & $75.52$ & $67.85$ & $46.16$ & $70.44$ & $67.32$ & $78.62$ & $67.65$ \\
\addlinespace[2pt]
 & Shrink-2.4 & $14.96$ & $20.21$ & $71.76$ & $60.40$ & $38.31$ & $64.33$ & $64.40$ & $75.90$ & $62.52$ \\
 & Proxy-static & $15.01$ & $20.56$ & $\mathbf{72.52}$ & $62.63$ & $38.05$ & $63.76$ & $64.09$ & $75.66$ & $62.79$ \\
 & Analytic & $\mathbf{14.59}$ & $\mathbf{20.06}$ & $72.09$ & $\mathbf{63.68}$ & $\mathbf{39.59}$ & $\mathbf{64.64}$ & $\mathbf{65.04}$ & $\mathbf{77.46}$ & $\mathbf{63.75}$ \\
 & Min-max & $30.21$ & $33.48$ & $59.19$ & $37.25$ & $27.13$ & $54.46$ & $55.49$ & $55.66$ & $48.20$ \\
\addlinespace
Qwen3-8B & FP16 reference & $9.715$ & $15.363$ & $77.75$ & $80.93$ & $56.48$ & $74.92$ & $67.64$ & $86.61$ & $74.05$ \\
\addlinespace[2pt]
 & Shrink-2.4 & $10.76$ & $16.78$ & $75.68$ & $73.82$ & $48.55$ & $70.65$ & $68.11$ & $85.26$ & $70.35$ \\
 & Proxy-static & $\mathbf{10.68}$ & $\mathbf{16.70}$ & $\mathbf{76.88}$ & $73.82$ & $49.49$ & $\mathbf{70.91}$ & $69.30$ & $85.50$ & $70.98$ \\
  & Analytic & $10.87$ & $16.93$ & $76.71$ & $\mathbf{77.82}$ & $\mathbf{52.47}$ & $70.90$ & $\mathbf{70.32}$ & $\mathbf{85.93}$ & $\mathbf{72.36}$ \\
 & Min-max & $15.20$ & $20.99$ & $69.80$ & $48.32$ & $32.76$ & $61.49$ & $56.59$ & $68.04$ & $56.17$ \\
\bottomrule
\end{tabular}}
\endgroup
\end{table*}

\begin{table*}[h!]
\centering
\caption{W2--ternary per-channel GPTQ with grid $\{-s,0,s\}$. Panels (a) and (b) show raw and HIP results. FP16 references are excluded from bolding, which identifies the best quantized rule within each model and panel. Lower perplexity and higher accuracy are better.}
\label{tab:zero-shot-w2}
\begingroup
\setlength{\tabcolsep}{3pt}
\renewcommand{\arraystretch}{1.08}
\resizebox{\textwidth}{!}{%
\begin{tabular}{llrrrrrrrrr}
\toprule
Model & Rule & Wiki2 PPL & C4 PPL & PIQA & ARC-E & ARC-C & HellaSwag & WinoGrande & BoolQ & Mean \\
\midrule
\multicolumn{11}{l}{\textbf{(a) Without preprocessing}} \\
Llama-3.1-8B & FP16 reference & $7.213$ & $11.388$ & $80.96$ & $79.59$ & $55.12$ & $79.26$ & $74.03$ & $84.10$ & $75.51$ \\
\addlinespace[2pt]
 & Shrink-2.4 & $158.59$ & $110.10$ & $53.37$ & $28.58$ & $24.57$ & $\mathbf{34.08}$ & $50.12$ & $\mathbf{62.54}$ & $\mathbf{42.21}$ \\
 & Proxy-static & $\mathbf{157.51}$ & $\mathbf{105.48}$ & $\mathbf{53.50}$ & $\mathbf{29.00}$ & $23.70$ & $31.30$ & $50.70$ & $55.30$ & $40.60$ \\
 & Analytic & $2.18\times10^{5}$ & $2.37\times10^{5}$ & $51.14$ & $24.12$ & $\mathbf{27.30}$ & $26.09$ & $51.14$ & $56.61$ & $39.40$ \\
 & Min-max & $7.87\times10^{5}$ & $6.73\times10^{5}$ & $50.76$ & $24.62$ & $26.19$ & $26.66$ & $\mathbf{51.22}$ & $48.13$ & $37.93$ \\
\addlinespace
Llama-3.2-3B & FP16 reference & $11.048$ & $16.486$ & $75.52$ & $67.85$ & $46.16$ & $70.44$ & $67.32$ & $78.62$ & $67.65$ \\
\addlinespace[2pt]
 & Shrink-2.4 & $981.90$ & $639.95$ & $52.23$ & $23.86$ & $26.28$ & $26.93$ & $50.20$ & $53.79$ & $38.88$ \\
 & Proxy-static & $\mathbf{442.39}$ & $\mathbf{315.76}$ & $\mathbf{52.30}$ & $\mathbf{25.80}$ & $25.80$ & $\mathbf{27.70}$ & $50.50$ & $\mathbf{54.40}$ & $\mathbf{39.40}$ \\
 & Analytic & $9291.88$ & $3934.12$ & $51.25$ & $25.13$ & $25.68$ & $26.47$ & $50.91$ & $38.04$ & $36.25$ \\
 & Min-max & $2.59\times10^{5}$ & $2.47\times10^{5}$ & $51.41$ & $24.16$ & $\mathbf{27.56}$ & $26.40$ & $\mathbf{52.25}$ & $47.46$ & $38.21$ \\
\addlinespace
Qwen3-8B & FP16 reference & $9.715$ & $15.363$ & $77.75$ & $80.93$ & $56.48$ & $74.92$ & $67.64$ & $86.61$ & $74.05$ \\
\addlinespace[2pt]
 & Shrink-2.4 & $135.89$ & $85.64$ & $57.02$ & $\mathbf{32.03}$ & $24.23$ & $33.43$ & $\mathbf{52.25}$ & $\mathbf{62.11}$ & $\mathbf{43.51}$ \\
 & Proxy-static & $\mathbf{71.14}$ & $\mathbf{54.13}$ & $\mathbf{57.70}$ & $31.10$ & $24.70$ & $\mathbf{36.40}$ & $50.60$ & $56.60$ & $42.80$ \\
 & Analytic & $2.16\times10^{11}$ & $3.58\times10^{11}$ & $52.12$ & $26.30$ & $\mathbf{26.79}$ & $26.27$ & $48.46$ & $37.83$ & $36.30$ \\
 & Min-max & $3.93\times10^{5}$ & $6.45\times10^{4}$ & $50.27$ & $24.58$ & $26.28$ & $26.12$ & $49.17$ & $48.93$ & $37.56$ \\
\addlinespace
\midrule
\multicolumn{11}{l}{\textbf{(b) HIP preprocessing}} \\
Llama-3.1-8B & FP16 reference & $7.213$ & $11.388$ & $80.96$ & $79.59$ & $55.12$ & $79.26$ & $74.03$ & $84.10$ & $75.51$ \\
\addlinespace[2pt]
 & Shrink-2.4 & $39.16$ & $53.99$ & $\mathbf{62.68}$ & $39.86$ & $24.40$ & $\mathbf{42.90}$ & $54.93$ & $\mathbf{67.37}$ & $\mathbf{48.69}$ \\
 & Proxy-static & $\mathbf{31.38}$ & $\mathbf{43.81}$ & $62.00$ & $\mathbf{39.90}$ & $25.10$ & $42.00$ & $\mathbf{56.50}$ & $65.20$ & $48.50$ \\
 & Analytic & $119.45$ & $75.27$ & $58.43$ & $36.78$ & $23.63$ & $38.12$ & $52.57$ & $58.53$ & $44.68$ \\
 & Min-max & $6.82\times10^{5}$ & $6.6\times10^{5}$ & $52.45$ & $25.25$ & $\mathbf{26.02}$ & $26.75$ & $48.15$ & $52.11$ & $38.45$ \\
\addlinespace
Llama-3.2-3B & FP16 reference & $11.048$ & $16.486$ & $75.52$ & $67.85$ & $46.16$ & $70.44$ & $67.32$ & $78.62$ & $67.65$ \\
\addlinespace[2pt]
 & Shrink-2.4 & $123.56$ & $123.84$ & $56.80$ & $\mathbf{34.93}$ & $23.98$ & $\mathbf{34.70}$ & $51.93$ & $60.58$ & $43.82$ \\
 & Proxy-static & $\mathbf{94.12}$ & $\mathbf{95.11}$ & $\mathbf{57.30}$ & $34.30$ & $24.10$ & $34.50$ & $52.20$ & $\mathbf{62.30}$ & $\mathbf{44.10}$ \\
 & Analytic & $179.82$ & $196.67$ & $57.24$ & $31.40$ & $23.29$ & $31.42$ & $\mathbf{52.33}$ & $56.73$ & $42.07$ \\
 & Min-max & $3.43\times10^{5}$ & $2.65\times10^{5}$ & $52.23$ & $26.05$ & $\mathbf{25.34}$ & $26.24$ & $50.12$ & $46.39$ & $37.73$ \\
\addlinespace
Qwen3-8B & FP16 reference & $9.715$ & $15.363$ & $77.75$ & $80.93$ & $56.48$ & $74.92$ & $67.64$ & $86.61$ & $74.05$ \\
\addlinespace[2pt]
 & Shrink-2.4 & $32.46$ & $36.32$ & $68.39$ & $52.82$ & $33.11$ & $48.17$ & $59.83$ & $68.99$ & $55.22$ \\
 & Proxy-static & $\mathbf{22.65}$ & $\mathbf{29.75}$ & $\mathbf{68.80}$ & $\mathbf{57.60}$ & $\mathbf{34.60}$ & $\mathbf{51.20}$ & $\mathbf{62.90}$ & $\mathbf{74.40}$ & $\mathbf{58.20}$ \\
 & Analytic & $46.98$ & $45.13$ & $65.07$ & $50.34$ & $29.44$ & $45.04$ & $57.93$ & $67.13$ & $52.49$ \\
 & Min-max & $5.06\times10^{4}$ & $1.91\times10^{4}$ & $49.46$ & $24.24$ & $26.96$ & $25.55$ & $49.88$ & $43.21$ & $36.55$ \\
\addlinespace
\bottomrule
\end{tabular}}
\endgroup
\end{table*}

\begin{table*}[p]
\centering
\caption{W2--ternary per-channel cross-engine comparison with HIP and grid $\{-s,0,s\}$. Bold values compare selectors within each model and engine, excluding FP16 references. These measurements are separate from the main GPTQ campaign. Lower perplexity and higher accuracy are better.}
\label{tab:zero-shot-w2-engines-hip}
\begingroup
\setlength{\tabcolsep}{3pt}
\renewcommand{\arraystretch}{1.08}
\resizebox{\textwidth}{!}{%
\begin{tabular}{llrrrrrrrrr}
\toprule
Engine & Rule & Wiki2 PPL & C4 PPL & PIQA & ARC-E & ARC-C & HellaSwag & WinoGrande & BoolQ & Mean \\
\midrule
\multicolumn{11}{l}{\textbf{(a) Llama-3.1-8B}} \\
\multicolumn{2}{l}{FP16 reference} & $7.213$ & $11.388$ & $80.96$ & $79.59$ & $55.12$ & $79.26$ & $74.03$ & $84.10$ & $75.51$ \\
\addlinespace
GPTQ & Shrink-2.4 & $60.17$ & $60.98$ & $60.83$ & $37.58$ & $23.63$ & $41.15$ & $52.17$ & $61.25$ & $46.10$ \\
 & Proxy-static & $\mathbf{33.71}$ & $\mathbf{45.22}$ & $\mathbf{63.00}$ & $\mathbf{40.45}$ & $\mathbf{25.51}$ & $\mathbf{43.82}$ & $\mathbf{55.33}$ & $\mathbf{66.76}$ & $\mathbf{49.14}$ \\
\addlinespace
ResComp & Shrink-2.4 & $38.59$ & $43.39$ & $\mathbf{62.79}$ & $36.11$ & $23.72$ & $42.37$ & $55.72$ & $\mathbf{65.90}$ & $47.77$ \\
 & Proxy-static & $\mathbf{29.69}$ & $\mathbf{35.45}$ & $61.59$ & $\mathbf{36.28}$ & $\mathbf{23.81}$ & $\mathbf{44.93}$ & $\mathbf{55.96}$ & $64.25$ & $\mathbf{47.80}$ \\
\addlinespace
QRoNoS & Shrink-2.4 & $61.07$ & $59.03$ & $\mathbf{66.54}$ & $\mathbf{45.71}$ & $\mathbf{27.22}$ & $45.12$ & $56.67$ & $65.38$ & $\mathbf{51.11}$ \\
 & Proxy-static & $\mathbf{38.17}$ & $\mathbf{43.70}$ & $56.58$ & $32.32$ & $22.35$ & $\mathbf{46.87}$ & $\mathbf{57.38}$ & $\mathbf{67.49}$ & $47.17$ \\
\addlinespace
\midrule
\multicolumn{11}{l}{\textbf{(b) Qwen3-8B}} \\
\multicolumn{2}{l}{FP16 reference} & $9.715$ & $15.363$ & $77.75$ & $80.93$ & $56.48$ & $74.92$ & $67.64$ & $86.61$ & $74.05$ \\
\addlinespace
GPTQ & Shrink-2.4 & $34.16$ & $38.39$ & $67.63$ & $54.50$ & $31.74$ & $48.65$ & $\mathbf{59.43}$ & $\mathbf{70.40}$ & $55.39$ \\
 & Proxy-static & $\mathbf{24.41}$ & $\mathbf{32.99}$ & $\mathbf{68.72}$ & $\mathbf{55.47}$ & $\mathbf{32.68}$ & $\mathbf{51.59}$ & $58.80$ & $67.00$ & $\mathbf{55.71}$ \\
\addlinespace
ResComp & Shrink-2.4 & $105.93$ & $101.35$ & $56.91$ & $32.62$ & $22.70$ & $34.46$ & $\mathbf{53.43}$ & $\mathbf{62.91}$ & $43.84$ \\
 & Proxy-static & $\mathbf{90.31}$ & $\mathbf{88.28}$ & $\mathbf{59.09}$ & $\mathbf{32.66}$ & $\mathbf{23.63}$ & $\mathbf{36.52}$ & $52.33$ & $60.98$ & $\mathbf{44.20}$ \\
\addlinespace
QRoNoS & Shrink-2.4 & $76.01$ & $57.94$ & $67.03$ & $48.32$ & $28.24$ & $44.60$ & $58.09$ & $63.36$ & $51.61$ \\
 & Proxy-static & $\mathbf{34.25}$ & $\mathbf{32.72}$ & $\mathbf{68.93}$ & $\mathbf{52.44}$ & $\mathbf{30.03}$ & $\mathbf{50.89}$ & $\mathbf{61.17}$ & $\mathbf{69.17}$ & $\mathbf{55.44}$ \\
\addlinespace
\bottomrule
\end{tabular}}
\endgroup
\end{table*}

\subsection{Preprocessing and effective rank}

\label{app:HIP}

\paragraph{HIP} Consider a fixed nonzero weight row $w\in\mathbb{R}^d$ and a
uniformly random orthogonal matrix $R$. A random rotation
distributes $wR$ uniformly on the sphere of radius $\|w\|_2$.
Equivalently,
\[
    wR \overset{d}{=}
    \|w\|_2 \frac{g}{\|g\|_2},
    \qquad g\sim\mathcal{N}(0,\operatorname{Id}_d).
\]
Since $\|g\|_2/\sqrt{d}$ concentrates around one, individual
coordinates are approximately Gaussian in high dimensions,
with standard deviation $\|w\|_2/\sqrt{d}$, the row's RMS.
Applying the corresponding inverse transformation to the
activations preserves the original linear map exactly:
\[
    \widetilde W=WR,\qquad
    \widetilde X=R^\top X,\qquad
    \widetilde W\widetilde X=WX.
\]
Although the full-precision output is unchanged, coordinatewise quantization in the rotated basis leads to different errors.

For fixed activations $X$, the orthogonal rotation gives $\widetilde H=R^\top H R$. This changes the coordinate system but preserves the Hessian's eigenvalues. Since $H=XX^\top$ is positive semidefinite, its trace is the sum of its eigenvalues and its operator norm is the largest eigenvalue. Both are therefore unchanged, so
\[
r_{\mathrm{eff}}(\widetilde H)
=\frac{\operatorname{tr}(\widetilde H)}{\|\widetilde H\|_{\mathrm{op}}}
=\frac{\operatorname{tr}(H)}{\|H\|_{\mathrm{op}}}
=r_{\mathrm{eff}}(H).
\]
Thus, rotation alone cannot improve the effective rank; changes in upstream quantization may still affect it by changing $X$.
Rotation therefore acts only on the weight distribution, not on the conditioning of the objective.
QuIP~\citep{chee2023quip} uses exactly such structured random orthogonal
preprocessing to reduce weight outliers and obtain incoherence
guarantees, substantially improving quantization performance.

Our HIP transform uses efficient randomized Hadamard mixing
based on QuIP\#~ \cite{tseng2024quip} rather than a full uniformly random rotation. Here, each transformed coordinate is a normalized random signed sum of the original weights, motivating a Gaussian approximation through a central-limit argument when no individual
contribution dominates.

\paragraph{Whitening.}
Whitening follows the same template as HIP, a change of basis $T$ that preserves the linear map, but drops the orthogonality constraint and thereby changes the geometry of the reconstruction objective. Let $T$ be an invertible matrix and set
\[
    \widetilde W=WT,\qquad
    \widetilde X=T^{-1}X,\qquad
    \widetilde H=T^{-1}HT^{-\top}.
\]
Unlike orthogonal rotations, such transformations alter the Hessian spectrum and hence its effective rank. For positive definite $H$, choosing $T$ with $TT^\top=H$, e.g.\ $T=H^{1/2}$, whitens the activations: $\widetilde H=I$, $r_{\mathrm{eff}}(\widetilde H)=d$, and the reconstruction objective reduces to the plain weight error $\|\widetilde W-\widetilde Q\|_F^2$, for which coordinatewise rounding is optimal. The cost is that the whitening moves the anisotropy from the Hessian into the weights: the rows of $\widetilde W=WH^{1/2}$ have covariance proportional to $H$ rather than being isotropic, and the quantization grid, mapped back to the original basis, is no longer a cube lattice but its image under $H^{-1/2}$. Coordinatewise quantization in the whitened basis therefore optimizes an isotropic objective over a different set of representable weights.
\paragraph{Diagonal rescaling.}
Between rotation and full whitening sits the per-channel rescaling used by SmoothQuant~\citep{xiao2023smoothquant} and AWQ~\citep{lin2024awq}. Let $T=D$ be diagonal with positive entries, so that
\[
    \widetilde W=WD,\qquad
    \widetilde X=D^{-1}X,\qquad
    \widetilde H=D^{-1}HD^{-1}.
\]
The natural choice $D=\operatorname{diag}(H)^{1/2}$ divides each input channel by its root-mean-square activation, turning $\widetilde H$ into the correlation matrix of the input channels. The resulting matrix has $\operatorname{tr}(\widetilde H)=d$, while the off-diagonal entries are the correlation coefficients $H_{ij}/\sqrt{H_{ii}H_{jj}}$. Consequently
\[
r_{\mathrm{eff}}(\widetilde H)=\frac{d}{\|\widetilde H\|_{\mathrm{op}}},
\]
which equals $d$ exactly when the channels are uncorrelated, in which case diagonal rescaling coincides with whitening.

\paragraph{Matched target-local comparison.}
To distinguish scale agreement from reconstruction quality,
we examine ten difficult W2 targets from Llama-3.1-8B,
selecting one matrix-row pair from each of ten fixed
depth bins.  In each of the bins, we select the matrix with the largest ratio of its summed row-wise best-searched GPTQ reconstruction error to signal energy. Within that matrix, we select the output row with the largest normalized best-searched GPTQ loss.

For each target, the raw, HIP, diagonal-scaling+HIP,
and whitening+HIP conditions share the same pre-target
weights and activation Hessian. Preprocessing is applied
only to the selected target, followed by local GPTQ
evaluation across candidate scales.
Figure \ref{fig:w2-preprocessing-depth} and Table \ref{tab:w2-preprocessing-depth} show that
whitening brings the best-searched scale closer to
the Gaussian analytic prediction in these rows, but
generally increases the best-searched reconstruction
loss relative to HIP alone. Improved scale agreement
therefore need not translate into improved reconstruction
quality.

\begin{figure*}[t]
    \centering
    \includegraphics[width=0.8\textwidth]
        {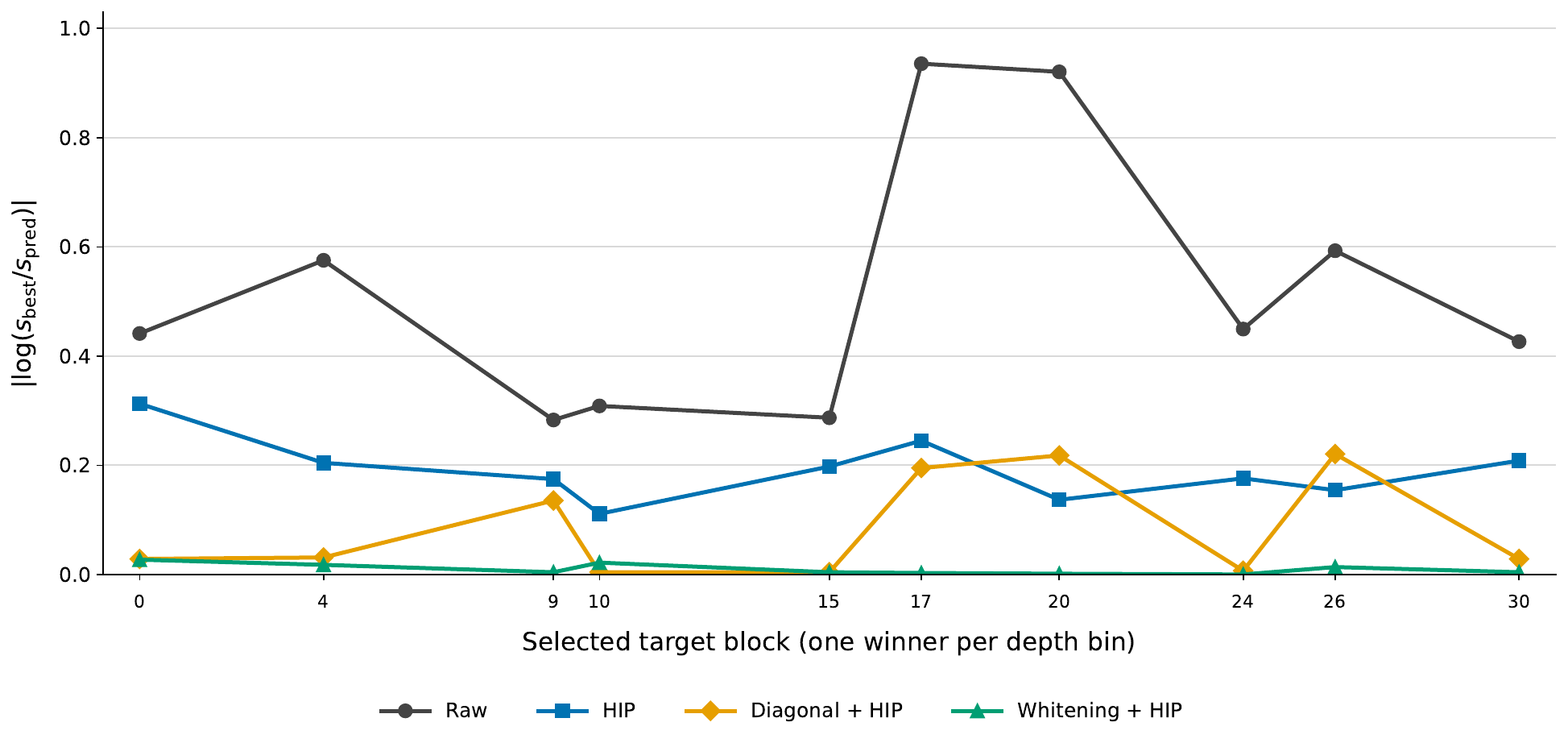}\\[-0.5em]
    \includegraphics[width=0.80\textwidth]
        {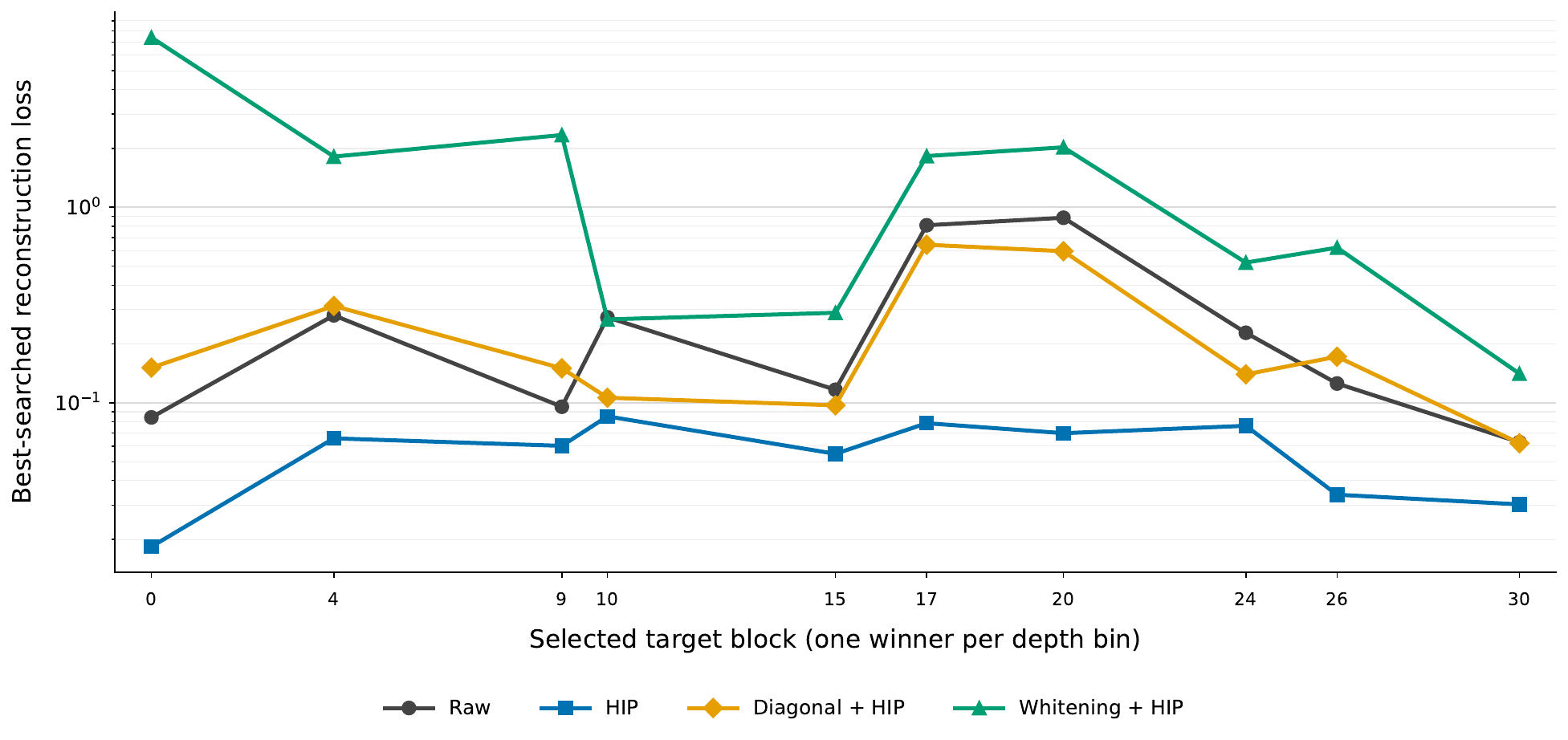}
    \caption{
    Target-local preprocessing on ten difficult W2 GPTQ
    rows selected across the 32 blocks of Llama-3.1-8B.
    Top: absolute log ratio between the best-searched
    scale and the Gaussian analytic scale.
    Bottom: best-searched local reconstruction loss
    on a logarithmic axis.
    Each point represents one matrix--row pair selected
    from a prespecified depth bin; horizontal positions
    indicate the actual block indices.
    All four conditions share the same unrotated,
    quantized upstream stream and pre-target data.
    Transforms are applied only to the selected target,
    and losses use a common normalization based on
    the raw row and Gram.
    Minima are the best evaluated values, not certified
    global optima.
    }
    \label{fig:w2-preprocessing-depth}
\end{figure*}

\begin{table*}[h!]
    \centering
    \small
    \begin{tabular}{lcccc}
        \toprule
        Preprocessing
        & Median scale gap
        & Gap lower
        & Median best-loss ratio
        & Loss lower \\
        & $\left|\log(s_{\mathrm{search}}/s_{\mathrm{an}})\right|$
        & than HIP
        & relative to HIP
        & than HIP \\
        \midrule
        None                & 0.445   & 0/10  & 3.46  & 0/10 \\
        HIP                 & 0.187   & ---   & 1.00  & ---  \\
        Diagonal + HIP      & 0.0298  & 8/10  & 3.63  & 0/10 \\
        Whitening + HIP     & 0.00406 & 10/10 & 20.8  & 0/10 \\
        \bottomrule
    \end{tabular}
    \caption{Target-local W2 GPTQ preprocessing on ten selected difficult rows
    of Llama-3.1-8B-Instruct.
    Counts compare each target against HIP alone. Lower is better in both numeric columns.}
    \label{tab:w2-preprocessing-depth}
\end{table*}

\subsection{Quantization error landscapes}
\label{app:quanterror}
\subsubsection{Gaussian landscapes and convergence}
\label{app:gaussianlandscapes}

We illustrate the finite-width behavior of the normalized
RTN loss using W3 symmetric quantization with levels
$\{-3,\ldots,3\}$. For $z\sim\mathcal N(0,I_d)$, we consider
the identity Gram and random orthogonal projectors of
rank $k$, yielding
\[
    L_d(s)=\frac{\|z-Q_s(z)\|_2^2}{d},
    \qquad
    L_{d,k}(s)=\frac{\|U^\top(z-Q_s(z))\|_2^2}{k},
\]
respectively, where $U^\top U=I_k$ and $U$ is independent
of $z$. To understand if the effective rank condition (H2) of Theorem \ref{thm:uniform} is necessary we choose  
$k \in \{d, \lceil\sqrt d\rceil, \lceil\log d\rceil \}$, and $constant = 8$.

Figure~\ref{fig:gaussianlandscapes} compares
individual landscapes at increasing widths with the
Gaussian reference.  While $k = d, \lceil\sqrt d $ converge at decreasing rates towards the gaussian landscape, $k=\lceil\log d\rceil  $  and $k =8 $ keep their unregular form even at high $d. $.

This matches the theoretical prediction, as $k = d, \lceil\sqrt d $ satisfy the Theorem \ref{thm:uniform} asymptotic rank
condition, whereas the logarithmic and fixed-rank
constructions probe regimes outside that sufficient
condition. These curves illustrate individual
realizations; the uniform error plot in Figure \ref{fig:uniformerror} summarizes
variability across 20 repetitions and validates Theorem \ref{thm:uniform}.

\begin{figure}[h!]
    \centering
    \includegraphics[width=0.8\textwidth]{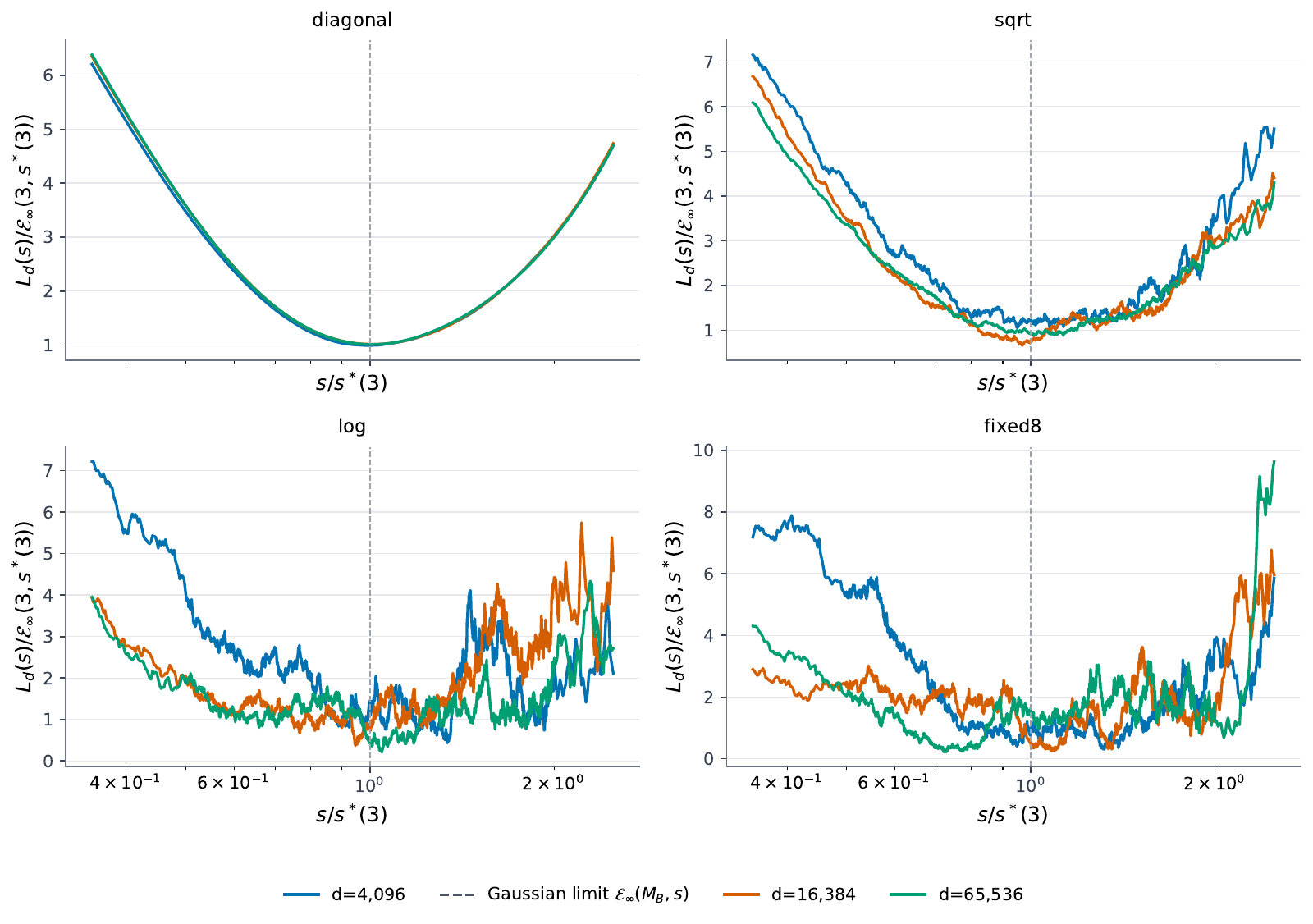}
    \caption{
Finite-width W3 RTN landscapes for four effective-rank
constructions. Each panel overlays
$d\in\{4096,16384,65536\}$ and the Gaussian reference
(dashed). Curves show one realization per configuration
(repetition 0). 
The horizontal coordinate is $s/s^\star(3)$, and all
losses are divided by
$\mathcal E_\infty(3,s^\star(3))$.
Only $0.35\leq s/s^\star(3)\leq2.5$ is displayed.
}
    \label{fig:gaussianlandscapes}
\end{figure}

\begin{figure}[h!]
    \centering
    \includegraphics[width=0.6\textwidth]{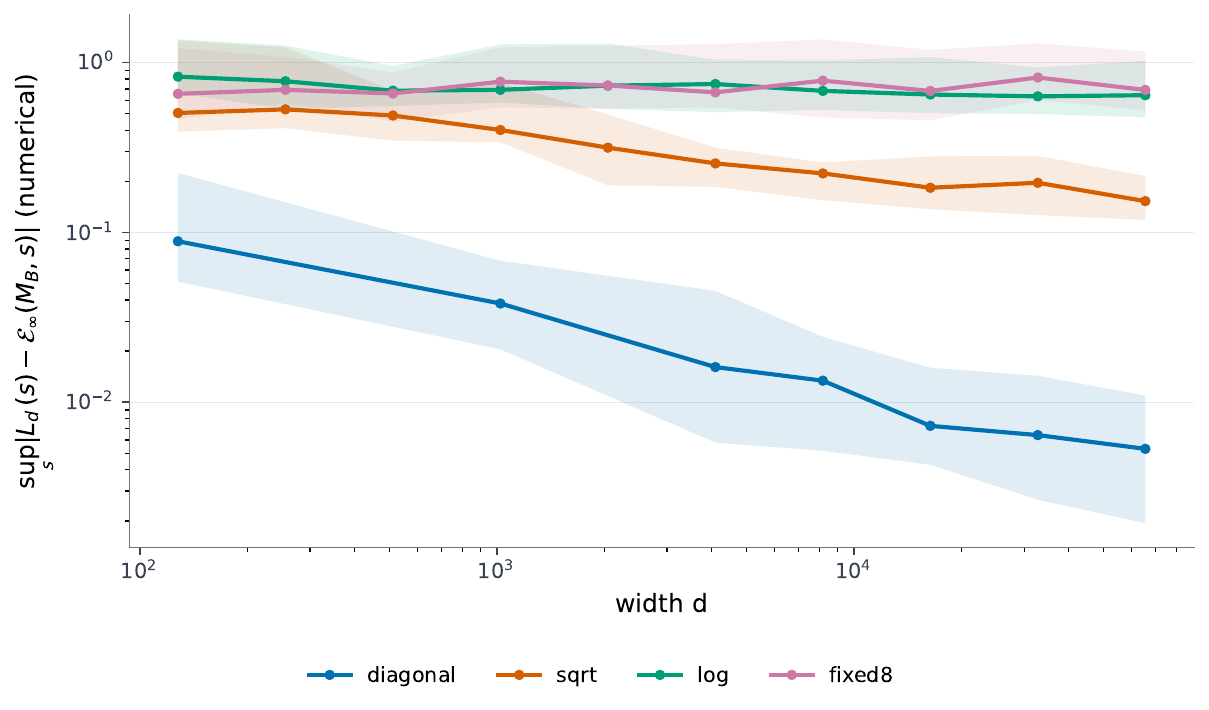}
    \caption{\textbf{Uniform-error estimates for W3 Gaussian quantization landscapes.} Median maximum error over the evaluated scale grid $\mathcal S$, $\max_{s\in\mathcal S}|L_d(s)-\mathcal E_\infty(M_3,s)|$,
as a function of width $d$. Shaded regions show the
10th--90th percentiles across 20 repetitions.
Curves correspond to effective ranks
$k=d$, $\lceil\sqrt d\rceil$, $\lceil\log d\rceil$,
and $8$. Errors decrease in the identity and
square-root-rank settings, while the logarithmic-
and fixed-rank settings show no comparable decrease
over the tested widths..}
    \label{fig:uniformerror}
\end{figure}

\subsection{Local Reconstruction Sensitivity with HIP}
\label{app:localreconstr}
We repeat the local reconstruction sensitivity analysis
of Section~\ref{sec:localreconstr} with HIP preprocessing.
Figure~\ref{fig:localreconstrHIP}, the HIP counterpart of
Figure~\ref{fig:local-gptq-sensitivity}, shows a similar
approximately exponential decrease in median sensitivity
with increasing bit-width.

To compare sensitivity magnitudes directly,
Figure~\ref{fig:localreconstrHIPvdnoHIP} overlays the results
with and without HIP at $\delta=5\%$. For each model and
bit-width, we also report the median paired ratio
$C_{\mathrm{sym}}^{\mathrm{HIP}}/
C_{\mathrm{sym}}^{\mathrm{noHIP}}$, matching calibration
seed, matrix, and row ID. Only pairs with positive, finite
sensitivity in both conditions are included. Each condition is evaluated around its own best-searched
scale.

\paragraph{An illustrative reconstruction landscape.}
Figure~\ref{fig:landscape} compares the Gaussian reference
with empirical RTN and GPTQ landscapes for one fixed row
of Llama-3.1-8B at W2, W3, and W8, with and without HIP.
The RTN landscapes resemble the Gaussian reference in
overall shape, although their minimizing scales can differ.
The basin insets illustrate how HIP brings the analytic
scale closer to the empirical minimizer in this example.

GPTQ compensation also changes the shape of the local
basin. However, this single-row comparison does not
establish a general flattening effect. Moreover, the
landscapes are displayed after division by the
bit-dependent Gaussian minimum, whereas our sensitivity
statistic excludes this additional normalization.
Cross-precision sensitivity comparisons therefore rely
on $C_{\mathrm{sym}}$ under the common loss normalization,
rather than on the visual sharpness of these curves.

\begin{figure}[h!]
    \centering
    \includegraphics[width=0.9\textwidth]
    {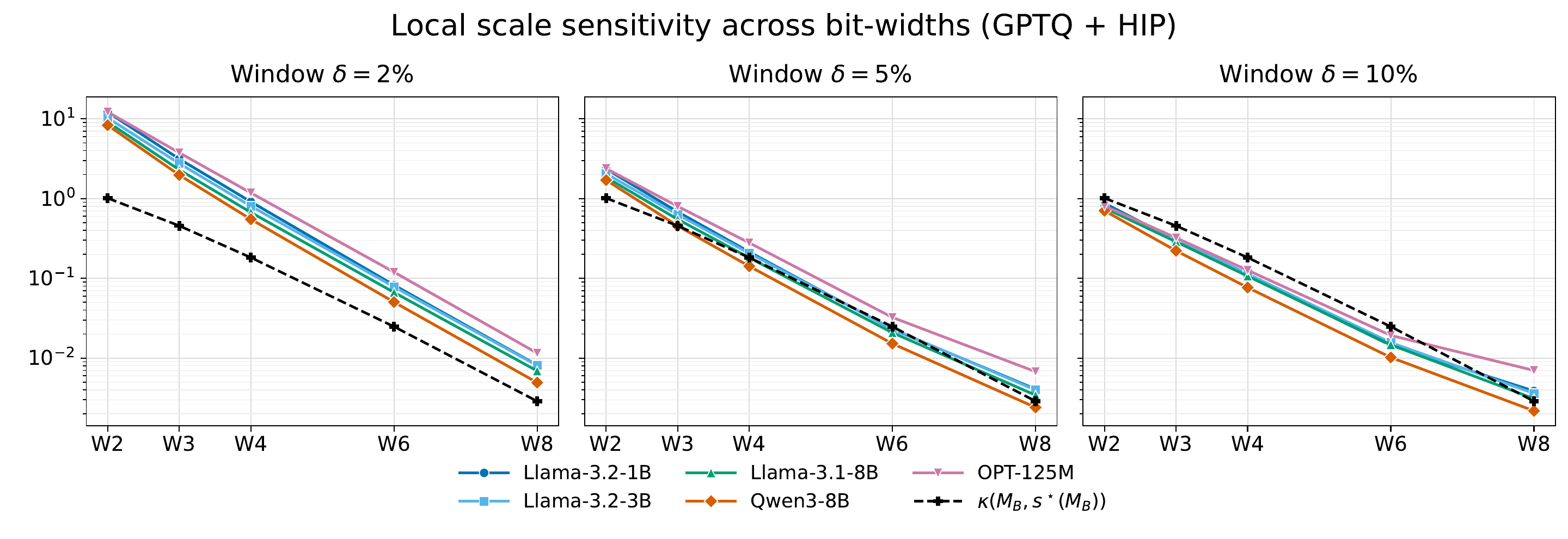}
    \caption{
    Local GPTQ scale sensitivity after HIP preprocessing.
    Colored curves show the median symmetric finite-window
    sensitivity $C_{\mathrm{sym}}$ around each row's
    best-searched scale at the tested precisions
    W2, W3, W4, W6, and W8.
    Panels use relative scale perturbations
    $\delta\in\{2\%,5\%,10\%\}$.
    The black dashed curve shows the Gaussian
    scale-invariant differential curvature $\kappa$,
    which is identical across panels.
    Medians pool the selected rows across all studied
    matrices and three calibration seeds for each model
    and bit-width.
    }
    \label{fig:localreconstrHIP}
\end{figure}

\begin{figure}[h!]
    \centering
    \includegraphics[width=0.9\textwidth]
    {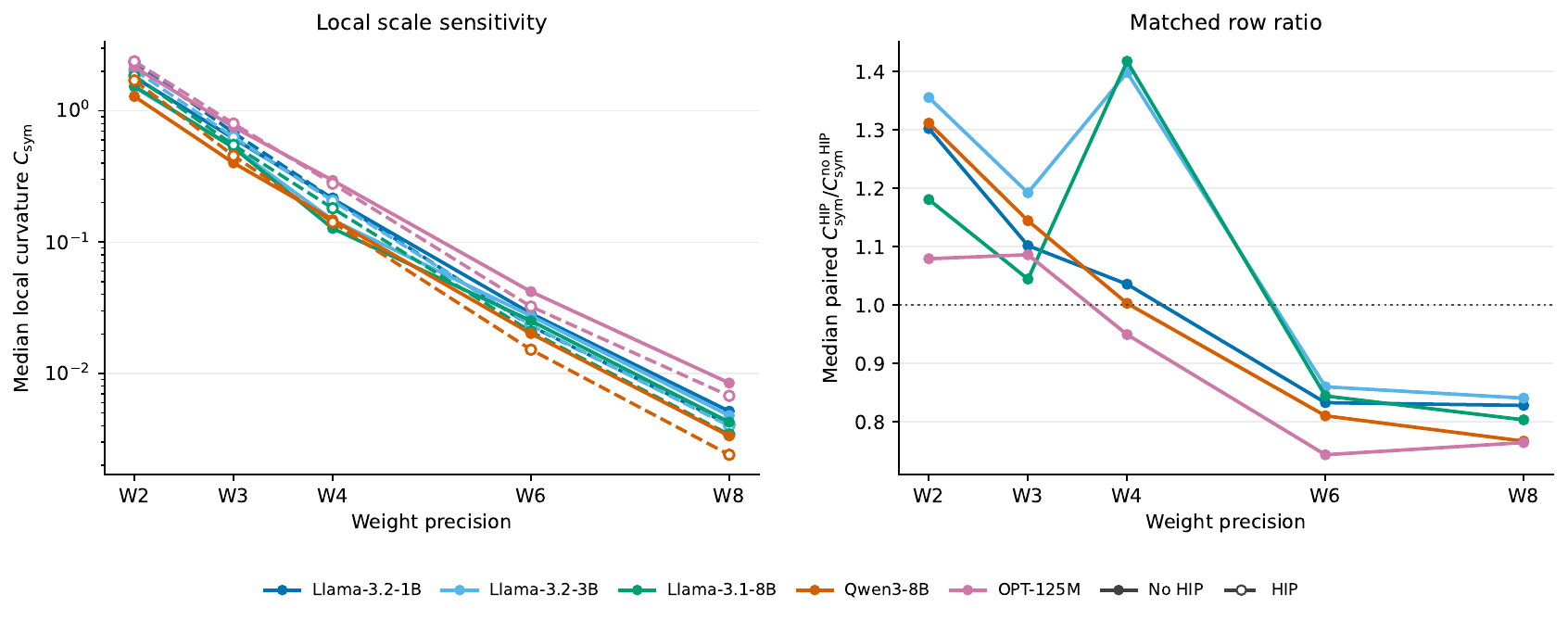}
    \caption{
    Effect of HIP on local GPTQ scale sensitivity at
    $\delta=5\%$.
    Left: median symmetric finite-window sensitivity
    $C_{\mathrm{sym}}$ by model and bit-width, with filled
    markers for no HIP and open markers for HIP.
    Right: median of the paired HIP-to-no-HIP sensitivity
    ratios, matching model, bit-width, calibration seed,
    matrix, and row ID. Only pairs with positive, finite
    sensitivity in both conditions are included.
    This statistic is the median of individual ratios,
    rather than the ratio of the pooled medians.
    }
    \label{fig:localreconstrHIPvdnoHIP}
\end{figure}

\begin{figure}[h!]
    \centering
    \includegraphics[width=\linewidth]
        {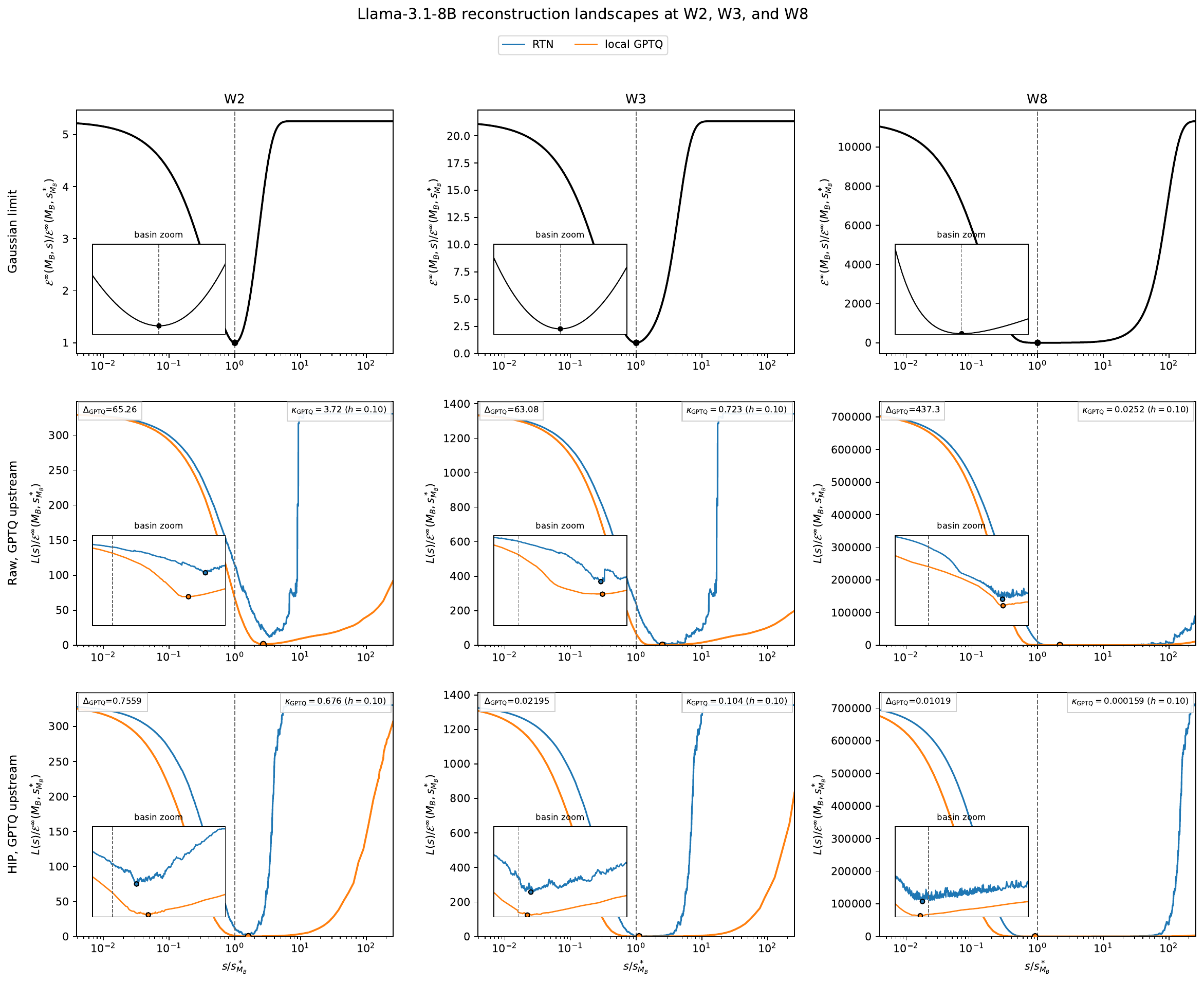}
    \caption{Scale-dependent reconstruction landscapes for
Llama-3.1-8B, block 0, \texttt{k\_proj}, output row 150.
Columns show W2, W3, and W8; rows show the Gaussian
reference and empirical RTN and GPTQ losses without
and with HIP. Scales and losses are normalized by the
Gaussian optimal scale and minimum loss, respectively.
Dashed vertical lines mark the analytic scale; markers
identify best-searched scales. Insets show the local
basins. Empirical losses use inputs from the upstream
GPTQ-quantized network at the corresponding precision.
This figure illustrates one fixed row rather than an
aggregate over the representative cohort.
}
    \label{fig:landscape}
\end{figure}

\end{document}